\documentclass[twoside]{article}

\usepackage[accepted]{aistats2023}
\usepackage[round]{natbib}

\usepackage[colorlinks=true]{hyperref}       % hyperlinks
\hypersetup{linkcolor=blue, citecolor=blue, urlcolor=blue}
\usepackage{url}            % simple URL typesetting
\usepackage{makecell}
\usepackage{tablefootnote}
\usepackage{multirow}
\usepackage{hhline}
\usepackage{amssymb}
\usepackage{array}
\usepackage{pifont, dsfont}
\usepackage{enumerate}
\usepackage{graphicx,epsfig,amsfonts,epstopdf,tabularx}

\usepackage{latexsym}
\usepackage{xcolor}
\usepackage{algorithm,algpseudocode}
\usepackage{algorithmicx}
\usepackage{stmaryrd}
\usepackage{bm}
\usepackage{multirow}
\usepackage{graphicx}
\usepackage{arydshln}
\usepackage{mathtools}
\usepackage{tabulary}
\usepackage{booktabs}

\usepackage{subfig}
\usepackage{xspace}
\usepackage{comment}
\usepackage{mathrsfs}
\usepackage{wrapfig}

\newtheorem{remark}{\textbf{Remark}}

\newtheorem{lemma}{\textbf{Lemma}}
\newtheorem{theorem}{\textbf{Theorem}}
\newtheorem{corollary}{\textbf{Corollary}}

\newtheorem{myDef}{Definition}

\newenvironment{proof}[1] {Proof {#1}:}{\hfill$\square$}

\newcommand{\UCBODC}{{\texttt{UCB-ODC}}\xspace}
\newcommand{\AAEODC}{{\texttt{AAE-ODC}}\xspace}
\newcommand{\ODC}{{\texttt{ODC}}\xspace}
\newcommand{\BT}{{\text{buffer threshold}}\xspace} % originally "buffer size requirement"
\newcommand{\BTs}{{\text{buffer thresholds}}\xspace} % originally "buffer size requirements"

\begin{document}


\runningauthor{Yu-Zhen Janice Chen, Lin Yang, Xuchuang Wang, Xutong Liu, Mohammad Hajiesmaili, John C. S. Lui, Don Towsley}

\twocolumn[

\aistatstitle{On-Demand Communication for Asynchronous Multi-Agent Bandits}

\aistatsauthor{ Yu-Zhen Janice Chen \And Lin Yang \And  Xuchuang Wang \and Xutong Liu}
\aistatsaddress{ Univ. of Massachusetts Amherst \And  Nanjing University \And The Chinese Univ. of Hong Kong} 
\aistatsauthor{Mohammad Hajiesmaili \And John C. S. Lui \And Don Towsley }
\aistatsaddress{Univ. of Massachusetts Amherst \And The Chinese Univ. of Hong Kong \And Univ. of Massachusetts Amherst } 
]

\begin{abstract}
This paper studies a cooperative multi-agent multi-armed stochastic bandit problem where agents operate \textit{asynchronously} -- agent pull times and rates are unknown, irregular, and heterogeneous -- and face the same instance of a $K$-armed bandit problem.
Agents can share reward information to speed up the learning process at additional communication costs.
We propose \texttt{ODC}, an on-demand communication protocol that tailors the communication of each pair of agents based on their empirical pull times. 
\texttt{ODC} is efficient when the pull times of agents are highly heterogeneous, and its communication complexity depends on the empirical pull times of agents.
\texttt{ODC} is a generic protocol that can be integrated into most cooperative bandit algorithms without degrading their performance. We then incorporate \ODC into the natural extensions of \texttt{UCB} and \texttt{AAE} algorithms and propose two communication-efficient cooperative algorithms. 
Our analysis shows that both algorithms are near-optimal in regret. 
\end{abstract}

%!TEX root = ODC-CameraReady.tex

\section{INTRODUCTION}

Asynchronous multi-agent multi-armed bandit (MAMAB) settings arise naturally in several applications. For instance, in online advertising with multiple heterogeneous servers, server processing capabilities and speeds are often different. Furthermore, the times that servers receive recommendation requests are often unknown and irregular. Another example is clinical trials with multiple labs in collaboration, where trial times depend on client visit times, 
which vary from lab to lab. In other large-scale distributed learning scenarios, such as IoT devices cooperating to learn an underlying environment, agents can be asynchronous in nature due to task arrangements or hardware limits. 

This paper studies a MAMAB setting where agents with \textit{unknown asynchronous} decision times cooperate to improve their learning performance.
Concretely, we consider a system where a group of $M$ agents, $\mathcal{A} = \{1, ..., M\}$, cooperate to solve the same instance of a $K$-armed bandit problem. 
An agent repeatedly chooses an arm from the arm set to pull and receives a stochastic reward from it. Agents have different numbers of decision rounds (pull times) at arbitrary unknown times.
Each agent aims to minimize its individual \textit{regret} -- the cumulative difference between the reward received by the agent and the expected reward of the best arm in the arm set. 
Agents cooperate by sharing reward information with each other, and their goal is to together minimize the \textit{group regret} -- the total amount of individual regret among the $M$ agents.
Cooperation among agents, however, comes with an additional \textit{communication} cost, which can be expensive for some applications when agents are geographically dispersed or have limited power/bandwidth resources for communication.

Prior studies~\citep{yang2021cooperative, yang2022distributed} have shown that it is possible to achieve near-optimal group regret by immediately broadcasting rewards.
In an asynchronous setting where agents have different pull speeds, immediate broadcasts can incur unnecessary communication costs. With immediate broadcast communication, an agent can receive multiple reward-sharing messages from another agent between its two decision rounds; these messages could have been accumulated (buffered) by that agent and sent all at once, incurring lower communication overhead. Hence, for a group of asynchronous agents, tailoring the message exchange protocol between each pair of agents can yield better  
communication efficiency. 

This paper aims to reduce communication costs over that of the immediate broadcast communication protocol (\texttt{IBC}) while achieving the same order of regret. 
The lack of synchronization between agents, however, poses a challenge on determining the timing of communication. Specifically, agents are uncertain of other agents’ learning progress at any time due to the arbitrary asynchronicity of agent pull times and hence need to trade-off communication costs to learn this information for better cooperation.
One might apply the idea of coordinated cooperative learning, e.g., the leader-follower framework, which has proven to be efficient in prior studies~\citep{kolla2018collaborative, dubey2020cooperative, wang2020optimal} of the synchronous MAMAB problem. However, unknown and irregular agent pull speeds hinder the application of coordinated cooperative learning. This can lead to a scenario where agents chosen to be in charge of exploration, leaders, are slow (have small pull rates), and agents chosen to perform exploitation are fast (have large pull rates), which can incur high regret. 
Another alternative is customizing spontaneous communication between agents, where each agent deliberately chooses its communication frequency to other agents according to their pull rates. However, efficient implementation of customized spontaneous communication is not possible since agents do not have prior knowledge of the pull times of others.

\paragraph{Contributions.} 
This paper develops \textit{On-Demand} Communication (\ODC), an efficient protocol for the asynchronous cooperative MAMAB model, where unique technical challenges are introduced by the unknown, irregular, and different decision times of agents. By the design of \ODC, we address the challenge of reducing the number of communications among asynchronous agents. Specifically, \ODC reduces the number of communications by tailoring the times communications occur between each pair of agents based on their empirical pull times.
More importantly, \ODC is generic and can be used with most cooperative bandit algorithms. 
We propose two decentralized MAMAB algorithms, \UCBODC and \AAEODC, which combine \ODC with natural extensions of \texttt{UCB} and \texttt{AAE} algorithms respectively.
Our analysis shows that both  \UCBODC and \AAEODC achieve near-optimal group regret upper bounds of $O(\sum_{i:\Delta_i > 0} \log(N)/\Delta_i)$, where $N\equiv \sum_{j \in \mathcal{A}} N_j$ is the total number of decision rounds of all agents, $N_j$ is the total number of decision rounds of agent $j$, and $\Delta_i$ is the  suboptimality gap of arm $i$.

Under \ODC, communication complexity, i.e., the total number of messages sent among agents, depends on the specific decision times of agents. We show that the communication complexity of \ODC is 
$O(\sum_{j, j^\prime \in \mathcal{A}} \min\{N_j, N_{j^\prime}\})$, which depends on the agents with the fewest decision rounds. This communication complexity is much smaller than that of the immediate broadcast communication protocol (\texttt{IBC}), $O(MN)$, when agent pull times are highly heterogeneous. Moreover, following prior ideas on the synchronous MAMAB setting, one has the option to tune message transmission rates under \ODC by allowing messages to vary in size 
to further reduce the communication complexity. 
For example, if the number of observations in a message is doubled after each communication, 
the communication complexity of \ODC becomes $O(\sum_{j, j^\prime \in \mathcal{A}}\min\{\log N_j, \log N_{j^\prime}\})$. In this way, our asynchronous policy can recover the state-of-the-art logarithmic communication complexity when applied to the synchronous MAMAB setting.

Our experimental results verify our theoretical observations and demonstrate that \ODC is especially advantageous when agent pull speeds are highly diversified, and when there exist many slow agents.

\paragraph{Prior Work.} 
We review the most relevant work here and refer to Appendix~\ref{sec:related} for extended literature review.
The most relevant work considers asynchronous bandit agents cooperating in a fully decentralized manner~\citep{yang2021cooperative, yang2022distributed,sankararaman2019social, feraud2019decentralized}. 
The model in~\cite{yang2021cooperative,yang2022distributed} assumes each agent periodically make decisions at different \textit{known} frequencies. Our paper assumes that pulling times are unknown and irregular. 
\cite{sankararaman2019social} study a gossip protocol, i.e.,  an agent can only communicate with one other agent at each time.
Last, \cite{feraud2019decentralized} studies the scenario where the goal is to identify the best arm instead of minimizing regret.
More broadly there is extensive prior work on MAMAB with synchronous agents either in a fully decentralized setting, e.g.,~\citep{szorenyi2013gossip, chawla2020gossiping, landgren2016distributed, buccapatnam2015information, martinez2019decentralized, madhushani2021one, cesa2016delay}, or using coordinated cooperative approach~\citep{shi2021federated, wang2019distributed, wang2020optimal, bar2019individual, chakraborty2017coordinated, dubey2020cooperative, kolla2018collaborative}. 
In the synchronous MAMAB setting, the batch approach (a.k.a., doubling epoch, phase, buffer)~\citep{perchet2016batched, gao2019batched} has been used to achieve logarithmic communication complexity, e.g, by~\cite{agarwal2021multi, shi2021heterogeneous, boursier2019sic}.
There are also works on asynchronous multi-agent learning in related fields such as federated linear bandit~\citep{li2022asynchronous, he2022simple} and online convex optimization with full information or semi-bandit feedback~\citep{cesa2020cooperative, jiang2021asynchronous, joulani2019think, bedi2019asynchronous, della2021efficient}. 

%!TEX root = ODC-CameraReady.tex

\section{ASYNCHRONOUS MULTI-AGENT BANDITS}\label{sec:model}

We study an asynchronous version of the cooperative multi-agent multi-armed bandit (MAMAB) problem with a set $\mathcal{A} = \{1, ..., M\}$ of $M$ independent agents and a set $\mathcal{K} = \{1, ..., K\}$ of $K$ arms. Each arm $i \in \mathcal{K}$ is associated with a mutually independent sequence of i.i.d. rewards, taken to be Bernoulli with mean $0 \leq \mu(i) \leq 1$. 
Let $i^* = \arg\max\nolimits_{i\in \mathcal{K}}\mu(i)$ denote the optimal arm. 
Define the suboptimality gap of arm $i$ as $\Delta_i \equiv \mu(i^*) - \mu(i)$ and let $\Delta \equiv \min_{i\in \mathcal{K}\setminus\{i^*\}}\Delta_i$ denote the smallest suboptimality gap in the arm set.

Agents operate \textit{asynchronously.} 
Let $N_j$ be the total number of decisions made 
by agent $j$; agent $j \in \mathcal{A}$ pulls arms at time slots $t^{j}_1, t^{j}_2, ...,$ and $t^{j}_{N_j}$, where both $N_j$ and the time slots are not known by any agent including agent $j$. 
We make no assumptions about when agents pull arms and the total number of pulls they make. One agent may pull many arms within an arbitrary interval, while another agent might not pull any arm. 
Furthermore, agents are allowed to join, leave, and re-join the system at arbitrary times. Let $T \equiv \max\nolimits_{j\in\mathcal{A}}t^j_{N_j}$ denote the learning horizon of the entire group of agents and $N\equiv \sum_{j \in \mathcal{A}} N_j$ denote the total number of decisions among all agents over the time horizon.

We consider the problem where there are no \textit{collisions}; i.e., agent always receives a Bernoulli reward with mean $\mu(i)$ from arm $i\in  \mathcal{K}$, irrespective of the actions of other agents.
Each agent $j \in \mathcal{A}$ pulls one arm at time $t \in \{t^{j}_1, t^{j}_2, ..., t^{j}_{N_j}\}$ with the goal of minimizing its cumulative regret. The expected cumulative regret of a single agent $j$ is defined as 
	$$\mathbb{E}[R^{j}_{N_j}] = \mu(i^*) N_j - \mathbb{E}\big[\sum\nolimits_{t \in \{t^{j}_1, t^{j}_2, ..., t^{j}_{N_j}\}} x_t(I^{j}_t)\big],$$
where $I^{j}_t \in \mathcal{K}$ is the arm pulled by agent $j$ at time $t$, reward $x_t(I^{j}_t)$ is taken from Bernoulli distribution with value 0 or 1, and the expectation is taken over the randomness of agent's decisions and arm rewards.
We denote the number of times agent $j$ pulls arm $i$ by time $t$ as $n^t_j(i)$, and the number of decisions agent $j$ makes by time $t$ as $n^t_j$.
We assume that every agent can reliably communicate with every other agent to share their observations. 
Let $\hat{n}^t_j(i)$ denote the empirical number of observations of arm $i$ that agent $j$ has at time $t$, either by pulling the arm, or obtained from other agents, and let $\hat{n}_j^t$ denote the total empirical number of observations agent $j$ has at time $t$.
The objective of the cooperative MAMAB problem is to minimize expected \textit{group regret}, defined as
% \begin{align*}\label{eq:group-regret}
	$\mathbb{E}[R] = \sum\nolimits_{j \in \mathcal{A}} \mathbb{E}[R^{j}_{N_j}]$, %,
% \end{align*}
 while maintaining low communication overhead.
Let $C$ denote the total number of messages sent by agents in horizon $T$, as in~\cite{wang2020optimal, yang2021cooperative, yang2022distributed}.
We precisely define the information included in a message in Definition~\ref{def:msg} in \S\ref{sec:odc}.

%!TEX root = ODC-CameraReady.tex

\begin{algorithm*}[!t]
    \footnotesize
	\caption{\ODC for Agent $j$ }
	\label{alg:ODC}
	\begin{algorithmic}[1]
		\State \textbf{Initialization:} exchange demands $E^{j\rightarrow j^\prime} \gets \texttt{True}, \forall j^\prime \in \mathcal{A}$, buffers $b_n^{j\rightarrow j^\prime}(i)\gets 0$, $b_\mu^{j\rightarrow j^\prime}(i)\gets 0, \forall j^\prime \in \mathcal{A}, \forall i\in \mathcal{K}$, number of communications $c^{j \rightarrow j^\prime} \gets 1, \forall j^\prime \in \mathcal{A}$, \BTs $f(c^{j \rightarrow j^\prime}) \gets f(1), \forall j^\prime\in \mathcal{A}$
% 		, algorithm parameters
        \For{$t=1...T$}
            \If{$t$ is a decision time slot of agent $j$, i.e., $t \in \{t^j_1,...,t^j_{N_j}\}$} \Comment{\texttt{decision making}}
        		\State Run an underlying bandit algorithm: pull arm $I_t^{j}$ according to e.g., \texttt{UCB}, or \texttt{AAE}, and receive instantaneous reward $x_t(I^j_t)$ 
        		\State Update the parameters of the underlying bandit algorithm, e.g., empirical mean rewards, number of observations
        		\For{each agent $j^\prime\in \mathcal{A}$}
        		\State Update the buffer for  agent $j^\prime$: $b_n^{j\rightarrow j^\prime}(I^j_t) \gets b_n^{j\rightarrow j^\prime}(I^j_t) +1$, $b_\mu^{j\rightarrow j^\prime}(I^j_t) \gets b_\mu^{j\rightarrow j^\prime}(I^j_t) + x_t(I^j_t)$
        		\If{$E^{j\rightarrow{j^\prime}}$ is $\texttt{True}$ and $\sum_{i\in \mathcal{K}}b_n^{j\rightarrow j^\prime}(i) \geq f(c^{j \rightarrow j^\prime})$} \Comment{\texttt{information sharing}}
        		%\State Share instantaneous reward $x_t(I^j_t)$ with agent $j^\prime$
        		\State Share the buffered information with $j^\prime$, i.e., send a message as defined in Definition~\ref{def:msg}, Set $c^{j\rightarrow j^\prime} \gets c^{j\rightarrow j^\prime} + 1$
        		%\State Share the buffered information with $j'$, i.e., arm indexes, and no. of observations since last sharing round, empirical reward means
        		\State Set exchange demand  $E^{j\rightarrow{j^\prime}} \gets \texttt{False}$ and renew the buffer for agent $j^\prime$ 
        		\State Update \BT $f(c^{j\rightarrow j^\prime})$, e.g., double it $f(c^{j\rightarrow j^\prime}) \gets 2f(c^{j\rightarrow j^\prime}-1)$ or keep it the same 
        % 		\Else \Comment{\texttt{update the buffer information}}
        % 		\State Update the buffered information for $I^j_t$ and agent $j'$
        		\EndIf
        		\EndFor
    		\EndIf
    		%\For{each agent $j^\prime\in \mathcal{A}$} 
        		\For{each new message received from any agent $j^\prime \in \mathcal{A}$}  \Comment{\texttt{message processing}}
        		\State Update the parameters of the underlying bandit algorithm, e.g., empirical mean rewards, number of observations
        		\If{agent $j$ has buffered $f(c^{j \rightarrow j^\prime})$ observations for $j^\prime$, i.e., $\sum_{i\in \mathcal{K}}b_n^{j\rightarrow j^\prime}(i) \geq f(c^{j \rightarrow j^\prime})$}
        		\State Share information by sending a message as defined in Definition~\ref{def:msg} to $j^\prime$, Set $c^{j\rightarrow j^\prime} \gets c^{j\rightarrow j^\prime} + 1$, renew buffer for $j^\prime$
        		\State Update \BT $f(c^{j\rightarrow j^\prime})$, e.g., double it $f(c^{j\rightarrow j^\prime}) \gets 2f(c^{j\rightarrow j^\prime}-1)$ or keep it the same
        		\Else
        		\State Set exchange demand $E^{j\rightarrow{j^\prime}} \gets \texttt{True}$
        		\EndIf
        		\EndFor
        % 		\Par{} 
        % 		\EndPar
    		%\EndFor
        \EndFor
	\end{algorithmic}
\end{algorithm*}

\section{ALGORITHM DESIGN}

  In \S\ref{sec:odc}, we first elaborate the design and provide intuition behind the On-Demand Communication (\ODC) protocol. Then, we incorporate \ODC into bandit algorithms and propose two communication-efficient cooperative bandit algorithms: \UCBODC (\S\ref{sec:ucb-odc}) and \AAEODC (\S\ref{sec:aae-odc}).

\begin{figure*}[!t]
    \vspace{-2.5mm}
	\centering
 	\subfloat[\ODC with static \BTs, i.e., \BTs are updated according to $f(1) \gets 1$, $f(c) \gets f(c-1)$ ]{\includegraphics[width=0.475\textwidth]{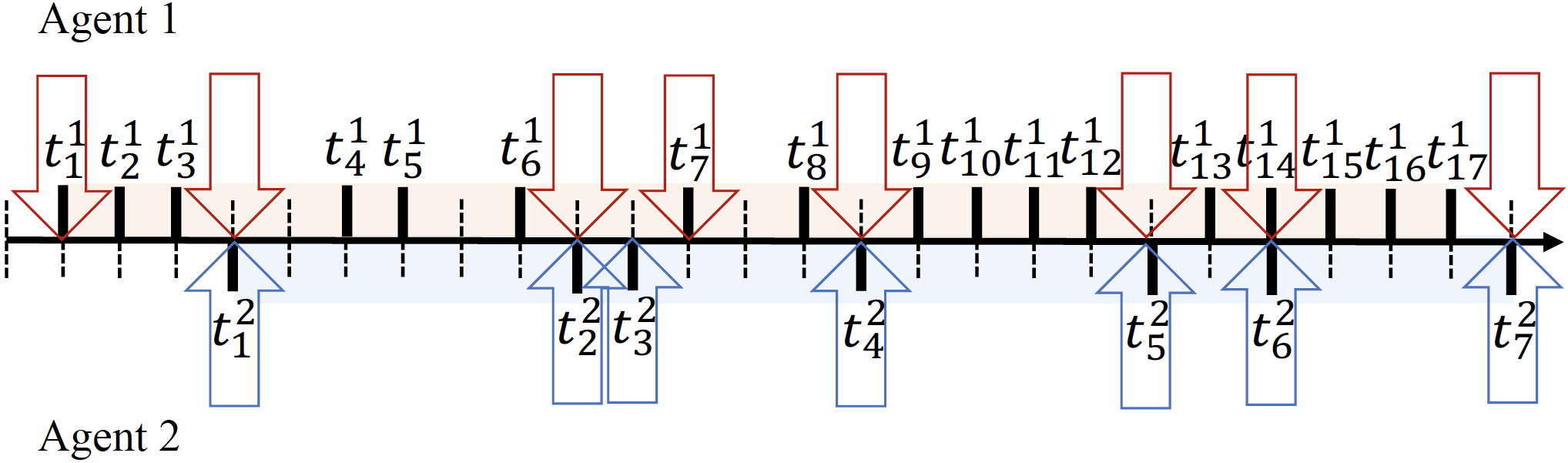}
	\label{fig:timeline-static}}
	\hfill
	\subfloat[\ODC with doubling \BTs, i.e., thresholds are updated according to $f(1) \gets 1$, $f(c) \gets 2f(c-1)$]{\includegraphics[width=0.475\textwidth]{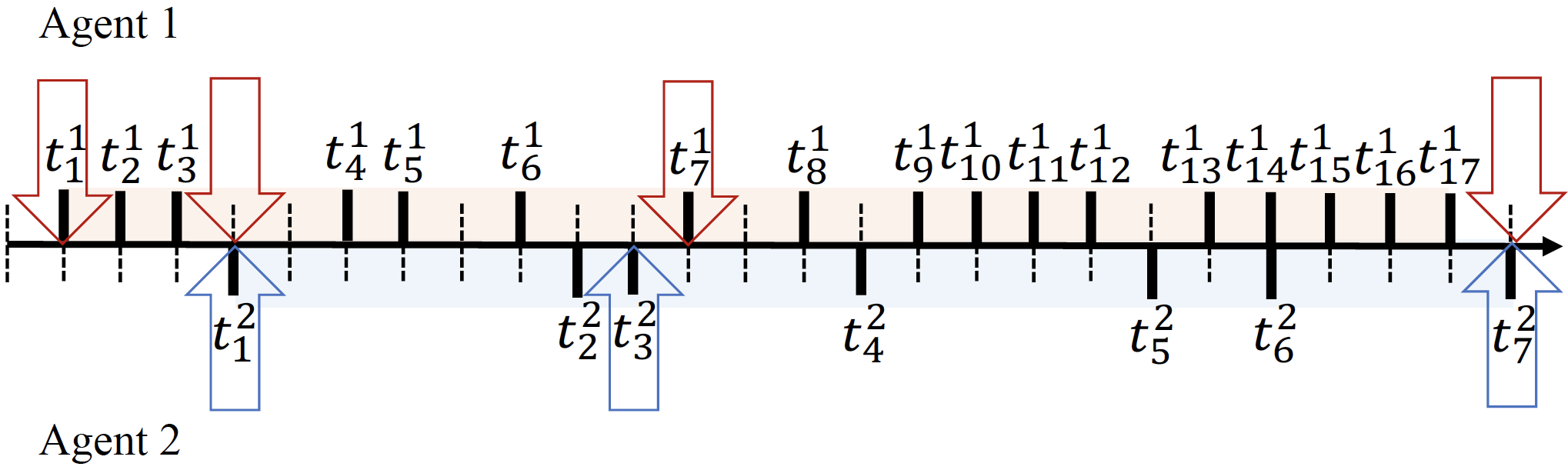}
	\label{fig:timeline-doubling}}
	\vspace{-2mm}
	\caption{Examples of two agents with arbitrary decision times}
	\label{fig:timeline}
\end{figure*}

\subsection{\ODC: On-Demand Communication Protocol} \label{sec:odc}
We present the On-Demand Communication (\ODC) protocol (\ODC) summarized in Algorithm~\ref{alg:ODC}. The core idea of \ODC is to leverage the fact that agents pull arms at different rates to reduce communication complexity while achieving the same order of regret achieved by algorithms that immediately share rewards. 
Consider a scenario with a \textit{fast} and a \textit{slow} agent. By \textit{fast}, we mean the agent pulls many arms while a \textit{slow} agent pulls very few arms during the same time horizon. If agents immediately share their observations, the fast agent incurs a large communication overhead by sending multiple messages between two consecutive decision rounds of the slow agent. In fact, the fast agent can reduce communication overhead while achieving the same regret if it aggregates the instantaneous rewards during the slow agent's non-decision period and sends the information all at once prior to the slow agent's next decision round. Hence, different agent pull rates motivate a new communication protocol that reduces communication complexity by scheduling communication times for each pair of agents according to their pull rates. 

Given the above motivation, one idea is to allow each agent to receive other observations at a rate proportional to its pull rate. In our asynchronous MAMAB model, it is challenging to tailor communication timings because agent pull times are irregular and unknown. A straightforward way to achieve this is to allow agents to request observations from other agents prior to pulling arms. However, requests introduce extra communication overhead, i.e., fast agents may make too many requests to slow agents before they obtain new reward information to share.

The idea implemented in \ODC is to treat each observation sharing message as an exchange demand.
Specifically, we let each agent $j$ maintain a set of binary valued exchange demand variables $(E^{j\rightarrow 1}, E^{j\rightarrow 2},..., E^{j\rightarrow M})$.
Once agent $j$ receives a message from agent $j^\prime$, it sets exchange demand $E^{j\rightarrow{j^\prime}}$ to $\texttt{True}$. 
Then, when agent $j$ acquires new information 
it responds back 
to agent $j^\prime$ and resets $E^{j\rightarrow {j^\prime}}$ to $\texttt{False}$. 
If agent $j$ acquires new information while $E^{j\rightarrow {j^\prime}}$ is $\texttt{False}$, agent $j$  \textit{buffers} it for agent $j^\prime$ while waiting for the exchange demand to be set to $\texttt{True}$.
Specifically, the buffer maintained for agent $j^\prime$ records the following information for each arm $i$: (1) the number of observations of $i$ that agent $j$ acquires by pulling it since the last time agent $j$ sent a message to agent $j^\prime$, denoted as $b_n^{j\rightarrow j^\prime}(i)$, (2) the cumulative reward over the observations of arm $i$ that agent $j$ acquires from pulling it since the last round agent $j$ sends a message to $j^\prime$, denoted as $b_\mu^{j\rightarrow j^\prime}(i)$. 
After agent $j$ sends the buffered information in a message to agent $j^\prime$, it renews the buffer by resetting %every value, 
$b_n^{j\rightarrow j^\prime}(i), b_\mu^{j\rightarrow j^\prime}(i), \forall i\in \mathcal{K}$, to zero.
\begin{myDef}\label{def:msg}
A message sent from agent $j$ to agent $j^\prime$ is a set of $K$ tuples: \{$(b_n^{j\rightarrow j^\prime}(i), b_\mu^{j\rightarrow j^\prime}(i))$, $\forall i\in\mathcal{K}$\}.\footnote{If an agent buffers $n$ observations, the number of observations and cumulative reward of each arm take value in $\{0,...,n\}$ and require $\log(n+1)$ bits.}
\end{myDef}

\ODC also implements thresholds on the buffer sizes\footnote{One may apply the batch/epoch setting techniques in synchronous MAMAB literature to the buffer threshold setting here.}, i.e., a buffer must contain at least as many observations as the \BT, that together with exchange demands determine whether an agent should send a message to another agent. Specifically, each agent $j$ maintains a set of positive integer valued variables $(c^{j\rightarrow 1},..., c^{j\rightarrow M})$ denoting the number of communications from agent $j$ to other agents. 
The \BT, $f(c^{j\rightarrow j^\prime})$, is a positive and monotonically increasing function of the number of communications $c^{j\rightarrow j^\prime}$. After agent $j$ sends a message to agent $j^\prime$ when the communication counter is $c^{j\rightarrow j^\prime}$, the communication counter is incremented by one and the \BT for the next communication is $f(c^{j\rightarrow j^\prime} + 1)$.
Possible candidates for the buffer threshold function include $f(c) = a, c=1,2,\ldots$, where $a$ is a positive integer, and $f(c) =  a^{c-1}, c=1,2,\ldots$, where $a>1$ is a positive integer.  The first example produces a constant \BT and the second allows \BT to increase exponentially  each time a message is sent. Under \ODC, an agent $j$ sends a message to agent $j^\prime$ if $E^{j\rightarrow j^\prime}$ is \texttt{True} and agent $j$ has buffered at least $f(c^{j\rightarrow j^\prime})$ observations for agent $j^\prime$.

In Figure~\ref{fig:timeline}, we provide simple examples of two agents with arbitrary decision time. To illustrate how \ODC (Algorithm~\ref{alg:ODC}) works, we describe the communication schedule for the example in Figure~\ref{fig:timeline-doubling}. %where there is a \textit{fast} agent and a \textit{slow} agent. 
Agent 1 first sends a message to agent 2 at time $t^1_1$, sets exchange demand $E^{1\rightarrow2}$ to \texttt{False}, sets $c^{1\rightarrow2}$ to $2$, and updates $f(c^{1\rightarrow2})$ to $2$. At $t^1_2$ and $t^1_3$, agent 1 pulls arms and buffers the obtained observations because $E^{1\rightarrow2}$ is \texttt{False}. At $t^2_1$, 
agent 1 receives a message from agent 2, replies with a message containing the observations obtained at $t^1_2, t^1_3$, renews the buffer, sets $c^{1\rightarrow2}$ to $3$, and updates $f(c^{1\rightarrow2})$ to $4$. At $t^1_4, t^1_5, t^1_6$, agent 1 pulls arms and buffers the obtained observations. At $t^2_3$, agent 1 receives a message from agent 2; instead of replying with a message, agent 1 sets the exchange demand $E^{1\rightarrow 2}$ to \texttt{True} at this time because it only buffered three observations while the \BT $f(c^{1\rightarrow2})$ is $4$. At $t^1_7$, agent 1 obtains an additional observation and satisfies the \BT; hence, it sends a message to agent 2, renews the buffer, sets $E^{1\rightarrow 2}$ to \texttt{False}, sets $c^{1\rightarrow2}$ to $4$, and updates $f(c^{1\rightarrow2})$ to $8$.

Last, we note that \ODC can handle agent arrivals and departures. Agent notifies the others when it departs or goes offline. When agent $j$ receives a departure notice from agent $j^\prime$, it sets exchange demand $E^{j\rightarrow{j^\prime}}$ to $\texttt{False}$, so it will buffer information for $j^\prime$.  
When an agent (re)joins the system, it notifies all other agents. 
%it send its set of arm indices to all other agents and requests the sets of arm indices from them.
When agent $j$ receives a join notice from agent $j^\prime$, it %replies with its set of arm indices and
(re)initializes exchange demand $E^{j\rightarrow{j^\prime}}$ to \texttt{True} so that %it can start cooperation.
the observations buffered during agent $j^\prime$'s leaving can be sent to it to (re)start cooperation.

\subsection{\UCBODC: Cooperative UCB with \ODC}
\label{sec:ucb-odc}
In this section, we present \UCBODC, a fully decentralized cooperative MAMAB algorithm that samples according to a natural extension of the Upper Confidence Bound (UCB) algorithm and uses \texttt{ODC} for communications. %with \textit{buffered} information sharing for communication. 
Under \UCBODC, agent $j$ computes an empirical mean reward, $\hat{\mu}(i, \hat{n}^t_j(i))$, over $\hat{n}^t_j(i)$ observations of agent $j$ for each arm $i \in \mathcal{K}$.
Note that the value of $\hat{n}^t_j(i)$ not only consists of instantaneous rewards agent $j$ received from pulling arm $i$, but it also includes information agent $j$ received from other agents. Under \UCBODC, agent $j$ also maintains a confidence interval for arm $i$ centered on its empirical mean value, $\hat{\mu}(i, \hat{n}^t_j(i))$, with width defined as 
\begin{equation}\label{eq:confidence-interval}
    %\texttt{CI}_j(i) \equiv \sqrt{\frac{\alpha \log \delta^{-1}}{2 \hat{n}^t_j(i)}},
    \texttt{CI}^t_j(i) \equiv \sqrt{\alpha \log(1/{\delta^t_j})/(2 \hat{n}^t_j(i))},
\end{equation}
where $\alpha$ and $\delta^t_j$ are algorithm parameters. 
With probability at least $1-(\delta^t_j)^\alpha$, the true reward mean, $\mu(i)$, lies in its confidence interval, i.e., $\mu(i) \in [\hat{\mu}(i, \hat{n}^t_j(i)) - \texttt{CI}^t_j(i), \hat{\mu}(i, \hat{n}^t_j(i)) + \texttt{CI}^t_j(i)].$
Further discussion and analysis of the confidence interval can be found in~\cite{bubeck2012regret}. 

Under \UCBODC, agent $j$ selects the arm with the largest upper confidence bound at each decision round, i.e., 
%\begin{align}\label{eq:ucb}
   $$ I^j_t \equiv \arg\max_{i\in\mathcal{K}} \hat{\mu}(i, \hat{n}^t_j(i)) + \texttt{CI}^t_j(i), \,\, t \in \{t_1^j,...t_{N_j}^j\}.$$
%\end{align}
Upon receiving an instantaneous reward for the selected arm $I^j_t$, \UCBODC updates the reward mean estimate and confidence interval of $I^j_t$. 
In the meantime, agent $j$ follows \ODC checking exchange demands, \BTs, and accordingly buffering the information or sending messages.
The pseudocode of \UCBODC is in Appendix~\ref{sec:pseudo-UCB}.

\subsection{\AAEODC: Cooperative AAE with \ODC}
\label{sec:aae-odc}
We propose \AAEODC, which combines \texttt{ODC} with a natural extension of the Active Arm Elimination (AAE) algorithm~\citep{even2006action}. 
Agents executing \AAEODC together maintain a dynamic \textit{candidate set} to keep track of arms likely to be the optimal arm, where the candidate set is updated using the confidence intervals as defined in (\ref{eq:confidence-interval}).
Specifically, the candidate set initially contains all arms. When agents observe or receive rewards, they recompute the confidence intervals of arms; if an arm's confidence interval completely falls below that of any other arm, it is removed from the candidate set as it is unlikely to be the optimal arm. Formally, arm $i$ is removed from the candidate set at time $t$ if for any agent $j$:
$$\exists i^\prime \in \mathcal{K} \text{ s.t. } \hat{\mu}^t_j(i) + \texttt{CI}^t_j(i) < \hat{\mu}^t_j(i^\prime) -\texttt{CI}^t_j(i^\prime).$$
Once an arm is eliminated by an agent, the agent broadcasts the index of the eliminated arm, so that other agents can keep updating the candidate set.
At each decision round, agent $j$ pulls from the candidate set the arm that agent $j$ has fewest observations. 
Once the candidate set size reduces to one, agents have completed the exploration task having identified optimal arm with high probability, and they do not need information from one another anymore.  
Hence, agents under \AAEODC stop their communication once the candidate set size shrinks to one.
The cooperation policy of \AAEODC follows the \ODC protocol and is summarized in the
pseudocode of \AAEODC in Appendix ~\ref{sec:pseudo-AAE}.

%!TEX root = ODC-CameraReady.tex

\section{ANALYSIS OF REGRET AND COMMUNICATION COMPLEXITY}\label{sec:analysis}

When agents are asynchronous and pull arms in arbitrary time slots, the performance of a cooperative bandit algorithm depends on how much agents can cooperate with each other.
A unique technical challenge in the regret analysis of \texttt{UCB-ODC} and \texttt{AAE-ODC} is bounding the additional number of times agents pull suboptimal arms due to delayed observation sharing while waiting for exchange demands between asynchronous agents or waiting for \BTs to be satisfied. 
To facilitate the regret analysis, we let $\tau_i$ denote the time slot such that 
%$$\mathscr{N}(i) \geq \sum\nolimits_{j \in \mathcal{A}} n_{j}^{\tau_i-1}(i),$$
$$\mathscr{N}(i) + M \geq \sum\nolimits_{j \in \mathcal{A}} n_{j}^{\tau_i}(i) > \mathscr{N}(i) \geq \sum\nolimits_{j \in \mathcal{A}} n_{j}^{\tau_i-1}(i),$$
for each suboptimal arm $i \in \mathcal{K}\setminus \{i^*\}$, where for \UCBODC, $\mathscr{N}(i) = (2\alpha \log N)/\Delta_i^2$, and for \AAEODC, $\mathscr{N}(i) = (16\alpha \log N)/\Delta_i^2$.

\subsection{Regret Results}

\begin{theorem}[Expected Group Regret under \ODC]\label{thm:group-regrets}
With algorithm parameter $\alpha \geq 3$ and \BTs being updated according to a positive and monotonically increasing function $f$, we have:\\
\emph{(a)} with $\delta^t_j = 1/N$, the expected group regret of \emph{\UCBODC} satisfies
\begin{equation}\label{eq:UCBODC-regret}
    \mathbb{E}[R]\leq 3KM +\sum_{i\in\mathcal{K}:\Delta_i>0}\Big(\frac{2\alpha \log N}{\Delta_i} +\sum_{j\in\mathcal{A}} F^j_i \Delta_i\Big),
\end{equation}
where $F^j_i$ is a non-negative variable defined as\footnote{We drop $t$ from notations $c_t^{j\rightarrow j^\prime}$ and $E_t^{j \rightarrow j^\prime}$ in algorithm presentations for brevity. The precise notations are used in analysis.}
\begin{equation}\label{eq:fji}
    F^j_i = \min\Big\{\Big(\sum_{j^\prime \in \mathcal{A}\setminus\{j\}}f(c_{\tau_i}^{j^\prime \rightarrow j})\Big), \frac{2\alpha \log N}{\Delta_i^2}\Big\};
\end{equation}
\emph{(b)} with $\delta^t_j =  1/N^2$, the expected group regret of \emph{\AAEODC} satisfies
\begin{equation} \label{eq:AAEODC-regret}
	\mathbb{E}[R] 
	\leq 3KM+\sum\limits_{i\in \mathcal{K}:\Delta_i>0} \Big(\frac{16 \alpha \log N}{\Delta_i} + \sum\limits_{j\in\mathcal{A}} G^j_i \Delta_i\Big),
\end{equation}
where 
$G^j_i$ is a non-negative variable defined as
\begin{equation}\label{eq:gij}
    G^j_i = \min\Big\{\Big(\sum_{j^\prime \in \mathcal{A}\setminus\{j\}}f(c_{\tau_i}^{j^\prime \rightarrow j})\Big), \frac{16\alpha \log N}{\Delta^2_i}\Big\}.
\end{equation}
\end{theorem}

The proofs of Theorem~\ref{thm:group-regrets}(a) and~\ref{thm:group-regrets}(b) deal with each suboptimal arm $i$, 
and upper bounds the extra number of times each agent pulls arm $i$ after time $\tau_i$.
A formal proof is given in Appendix~\ref{sec:group-regrets-proof}. 
In the following, we highlight important properties of the regret results under \ODC. 

\begin{remark} (Regret characterization by total number of decision rounds.)
We observe that the expected regret in Theorem~\ref{thm:group-regrets} is characterized by $N$, the total number of decision rounds among all agents in the learning horizon.
This is in contrast to the synchronous agent setting where regret is usually presented as a function of the learning horizon $T$ and the total number of agents $M$.
When agent pull rates significantly differ from each other, Theorem~\ref{thm:group-regrets} provides a much tighter regret bound than those derived for synchronous settings, as $N$ can be much smaller than $M\times T$.
\end{remark}

\begin{remark} (Regret optimality.)\label{rmk:regret-optimal}
When \BTs are set to a small constant $a$, i.e., $f(c^{j\rightarrow j^\prime})=a, c^{j\rightarrow j^\prime}=1,2,\ldots, \forall j, j^\prime$, then $F^j_i, \forall i,j$ (resp. $G^j_i, \forall i,j$) is bounded by constant $Ma$, and \emph{\UCBODC} (resp. \emph{\AAEODC}) achieves a provable optimal regret upper bound.
To show this, we derive the following lower bound on group regret by adopting the proof techniques for the asymptotic lower bound for single-agent bandits, e.g.,~\cite[Theorem~2.2]{bubeck2012regret}:
\begin{equation}\label{eq:lower-bound}
    %\liminf_{T \rightarrow \infty} \frac{\mathbb{E}[R]}{\log T} \geq \sum_{i\in \mathcal{K}:\mathcal{A}^*_{-i}\neq \emptyset} \frac{\tilde{\Delta}_i}{\text{KL}(\mu(i), \mu(i)+\tilde{\Delta}_i)} = \Omega \left( \sum_{i\in \mathcal{K}:\mathcal{A}^*_{-i}\neq \emptyset} \frac{1}{\tilde{\Delta}_i} \right).
    \mathbb{E}[R] = \Omega \Big( \sum\nolimits_{i\in \mathcal{K}:\Delta_i > 0} \frac{\log N}{\Delta_i} \Big).
    %\mathbb{E}[R] = \Omega \Bigg( \sum\nolimits_{i\in \mathcal{K}:\Delta_i > 0} \log(N)/\Delta_i \Bigg).
\end{equation}
The proof of (\ref{eq:lower-bound}) is given in Appendix~\ref{sec:lower-bound}.
Then, since $F^j_i$ (resp. $G^j_i$) is a constant $\forall i,j$, one observes that the regret of \emph{\UCBODC} in~\eqref{eq:UCBODC-regret} (resp. of \emph{\AAEODC} in \eqref{eq:AAEODC-regret}) is near-optimal compared to the lower bound (\ref{eq:lower-bound}), up to two constant terms, i.e., $3KM$ and $\sum_{i,j} F^j_i$ (resp. $\sum_{i,j} G^j_i$). 
\end{remark}

\begin{remark}
\label{rem:unsent}(Impact of \BTs.)  
The setting of \BTs influences the trade-off between communication complexity and group regret. Remark~\ref{rmk:regret-optimal} shows that \emph{\UCBODC} and \emph{\AAEODC} have near-optimal regrets if \BTs are set to be small (compared to $\log(N)/\Delta^2$). If \BTs are simply set to be always large, one can reduce the communication complexity while incurring higher regret. Depending on specific scenarios, e.g., as in Remark~\ref{rmk:regret-synchronous}, \BTs can be wisely set to achieve low communication while not degrading the regret much.

\end{remark}

\begin{remark}(Performance in synchronous setting.)\label{rmk:regret-synchronous}
When applied to a MAMAB setting with synchronous agents where every agent makes a decision at every time slot, our asynchronous algorithm \emph{\AAEODC} recovers a near-optimal regret $O(\sum_{i\in\mathcal{K}:\Delta_i>0}\log(N)/\Delta_i)$ with logarithmic communication complexity by setting a doubling \BT whose size is proportional to the number of arms remaining in the candidate set, $\mathcal{C}$,  i.e., $f(c)=|\mathcal{C}| 2^{c-1}$, $c=1,2,\ldots$.
We show this in Appendix~\ref{sec:recover-regret} and discuss the recovery of logarithmic communication complexity in Remark~\ref{rmk:comm-recover}.
\end{remark}

%!TEX root = ODC-CameraReady.tex

\begin{table*}
    \caption{Summary of Results (all regret bounds are problem-dependent and we omit the \(1/\Delta\) factor)}
    \centering
    \begin{tabular}{l|llll}
          \toprule
           & Pull Times
           & Buffer Thres.
           & Group Regret
           & Communication \#\\ \midrule
          \texttt{UCB-ODC}
           & Async., Sync.
           & Constant
           & \(O(K\log N)\)
           & \(O(\sum_{j, j^\prime\in\mathcal{A}} \min\{N_j, N_{j^\prime}\})\) \\
          \texttt{AAE-ODC}
           & Async., Sync.
           & Constant
           & \(O(K\log N)\)
           & \(O(\sum_{j, j^\prime\in\mathcal{A}}K\min\{\log N, N_j, N_{j^\prime}\}/\Delta^2)\) \\
           \texttt{AAE-ODC}
           & Sync.
           & Doubling
           & \(O(K\log N )\)
           & \(O(\sum_{j, j^\prime\in\mathcal{A}}\log[K\min\{\log N, N_j, N_{j^\prime}\}/\Delta^2])\) \\
           \bottomrule
    \end{tabular}
\end{table*}

\subsection{Communication Complexity}

\begin{remark}[Communication Complexity under \texttt{IBC}]\label{prop:IBC-comm}
The communication complexity of MAMAB algorithms using immediate broadcasting communication (\emph{\texttt{IBC}}) is 
    $$C = \sum\nolimits_{j \in \mathcal{A}} \sum\nolimits_{j^\prime \in \mathcal{A}\setminus\{j\}} N_j.$$
\end{remark}

\begin{theorem}[Communication Complexities under \ODC]\label{thm:comm}
When \BTs are updated according to a positive and monotonically increasing function $f$, the communication complexities of \emph{\UCBODC} and \emph{\AAEODC} satisfy:\\
\begin{equation}\label{eq:comm}
    C \leq \sum\nolimits_{j \in \mathcal{A}} \sum\nolimits_{j^\prime \in \mathcal{A}\setminus\{j\}} \min\{C_j, C_{j^\prime}\}+1,
\end{equation}
where $C_j$ is the largest integer in set $\{1,...,N_j\}$ such that\\
\emph{(a)} for \emph{\UCBODC}
\begin{equation}\label{eq:ucb-C_j}
\Big(\sum_{c=1}^{C_j} f(c)\Big) \leq N_j;
\end{equation}
\emph{(b)} for \emph{\AAEODC}
\begin{equation}\label{eq:aae-C_j}
\Big(\sum_{c=1}^{C_j} f(c)\Big) \leq \min \Big\{ 2K+\sum_{i\in\mathcal{K}}\frac{16\alpha\log N}{\max\{\Delta_i^2, \Delta^2\}}, N_j\Big\}.
\end{equation}

\end{theorem}
Proofs of Theorem~\ref{thm:comm}(a) and~\ref{thm:comm}(b) are in Appendix~\ref{sec:comm-proof}. 

\begin{corollary} When \BTs $f(c) = a, c=1,2,\ldots$, $a$ is a positive integer,  we have: \\
\emph{(a)} the communication complexity of \emph{\UCBODC} is
$$O\Big( \sum\limits_{j, j^\prime \in \mathcal{A}} \min\Big\{\big\lfloor\frac{N_j}{a}\big\rfloor, \big\lfloor\frac{N_{j^\prime}}{a}\big\rfloor\Big\}\Big);$$
\emph{(b)} the communication complexity of \emph{\AAEODC} is
$$O\Big(\sum\limits_{j, j^\prime \in \mathcal{A}} \min\Big\{ 
\Big\lfloor\frac{K\log N}{a\Delta^2}\Big\rfloor
, \big\lfloor\frac{N_j}{a}\big\rfloor, \big\lfloor\frac{N_{j^\prime}}{a}\big\rfloor\Big\}\Big).$$
\end{corollary}

\begin{corollary}\label{coro:doubling-buffer}
When \BTs $f(c) =  a^{c-1}, c=1,2,\ldots$, $a>1$ is a positive integer,  we have:\\ 
\emph{(a)} the communication complexity of \emph{\UCBODC} is
$$O\Big(\sum\limits_{j, j^\prime \in \mathcal{A}} \min\big\{\lfloor\log_a(N_j)\rfloor, \lfloor\log_a(N_{j^\prime})\rfloor\big\}\Big);$$
\emph{(b)} the communication complexity of \emph{\AAEODC} is
{\small
$$O\Big(\!\!\sum\limits_{j, j^\prime \in \mathcal{A}}\! \min\Big\{ 
\Big\lfloor\log_a\!\Big(\!\frac{K\log N}{a\Delta^2}\!\Big)\Big\rfloor
, \lfloor\log_a(N_j)\rfloor, \lfloor\log_a(N_{j^\prime})\rfloor\Big\}\Big).$$
}
\end{corollary}

In what follows, we highlight the significance of the communication complexity results of \ODC.

\begin{remark} (Communication complexity characterization by number of decision rounds.)\label{rmk:comm-empirical}
A major contribution of Theorem \ref{thm:comm} is the generalization of the communication complexity analysis to the asynchronous-agent setting where the upper bound depends on the number of decision rounds of the agents instead of the length of time horizon. More specifically, the communication complexity of \emph{\ODC} implicitly depends on the total number of decisions by all agents, $N$. %, since $\mathscr{T}^{j\rightarrow{j^\prime}}$ is a subset of decision time slots of agent $j$.
In general when agent pull rates differ significantly from each other, $N$ is much smaller than $M\times T$, and Theorem \ref{thm:comm} provides a much tighter upper bound than previous results relying on $T$. 
\end{remark}

\begin{remark} (Performance with asynchronous agents.)\label{rmk:comm-asynchronous}
\emph{\ODC} is able to deal with heterogeneous pull rates, since communication can be tailored for on-demand transmissions. 
Especially, under \ODC, the number of communications between each pair of agents depends on the slower agent; while, under \texttt{IBC}, the number of communications between each pair of agents is dominated by the faster agent.
For example, consider a two-agent system where agent $j^{\text{fast}}$ is a fast agent that pulls arms much more often than a slow agent $j^{\text{slow}}$, i.e., $N_{j^{\text{fast}}} \gg N_{j^{\text{slow}}}$. 
By Theorem~\ref{thm:comm}, the number of messages sent by \emph{\ODC} is at most $2 N_{j^{\text{slow}}}+2$,
while, by Remark~\ref{prop:IBC-comm}, the number of messages sent under immediate communication can be as large as~$N_{j^{\text{fast}}} + N_{j^{\text{slow}}}$ .
\end{remark}

\begin{remark}(Recovery of logarithmic communication complexity in synchronous setting.)\label{rmk:comm-recover}
When our asynchronous \emph{\ODC} protocol is applied to a MAMAB setting with synchronous agents, i.e., $N_j = T, \forall j\in\mathcal{A}$, Corollary~\ref{coro:doubling-buffer} implies that we can recover a $O(M^2\log T)$ communication complexity by doubling \BT after each message transmission. 
\end{remark}

\begin{remark}(Double logarithmic communication complexity of \emph{\AAEODC})
As shown in Theorem~\ref{thm:comm}(b), the communication complexity of \emph{\AAEODC} depends logarithmically on the total number of decision rounds among all agents, $N$, and therefore depends logarithmically on the learning horizon, $T$, when the suboptimality gaps are large, e.g., $\Delta_i \gg 1/\sqrt{N_j}, \forall j\in\mathcal{A}, i \in\mathcal{K}\setminus\{i^*\}$. 
This is because \emph{\AAEODC} stops communication once the exploration task completes, i.e., once the \textit{candidate set} size becomes one. Corollary~\ref{coro:doubling-buffer}(b) further shows that double logarithmic communication complexity can be achieved if the \BTs of \emph{\AAEODC} are set to be doubling.

\end{remark}

Note that \ODC also works for the scenario that communication has a deterministic delay and that each agent only has access to a \textit{local} subset of the $K$ arms, as in~\cite{yang2022distributed, chawla2020gossiping, yang2021cooperative}. We provide regret and communication complexity analysis of both algorithms for such scenario in Appendix~\ref{sec:delay} and~\ref{sec:heterogeneous}.

%!TEX root = ODC-CameraReady.tex

\begin{figure*}[t]
	\centering
	\hfill
 	\subfloat[UCB Communication]{\includegraphics[width=0.24\textwidth]{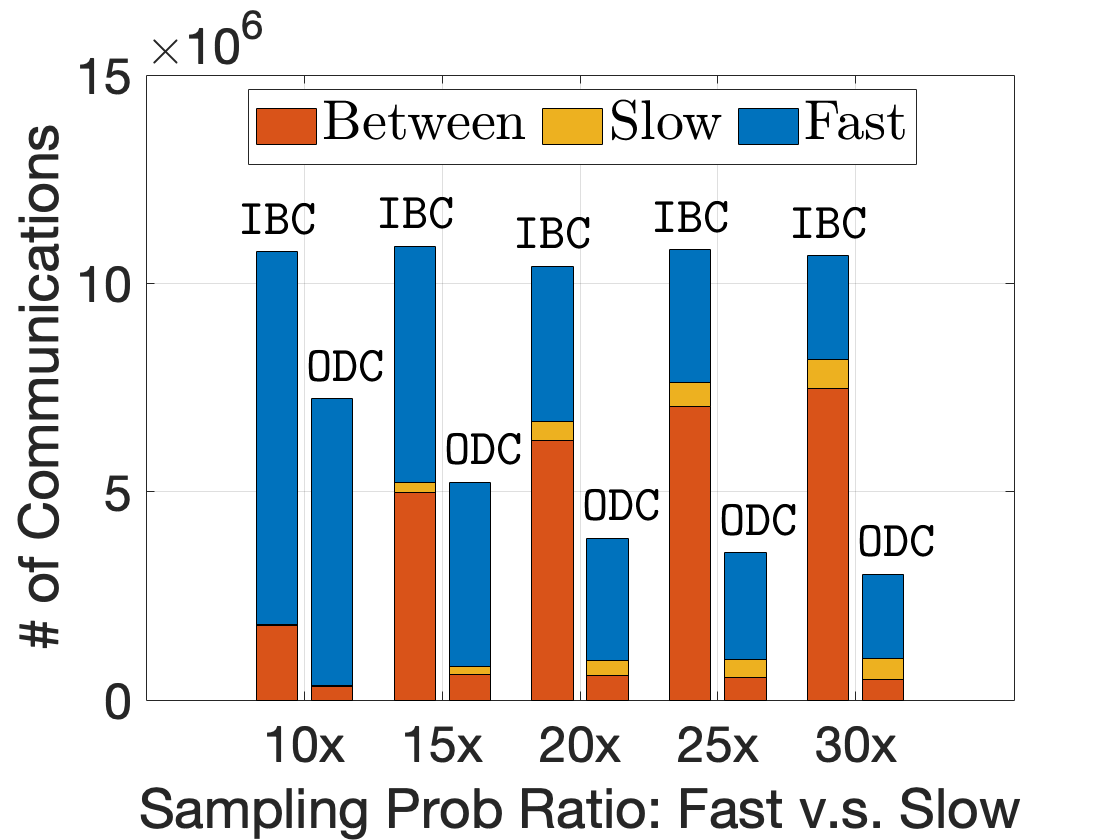}
 	\label{fig:exp6-UCB-comm}}
 	\hfill
	\subfloat[AAE Communication]{\includegraphics[width=0.24\textwidth]{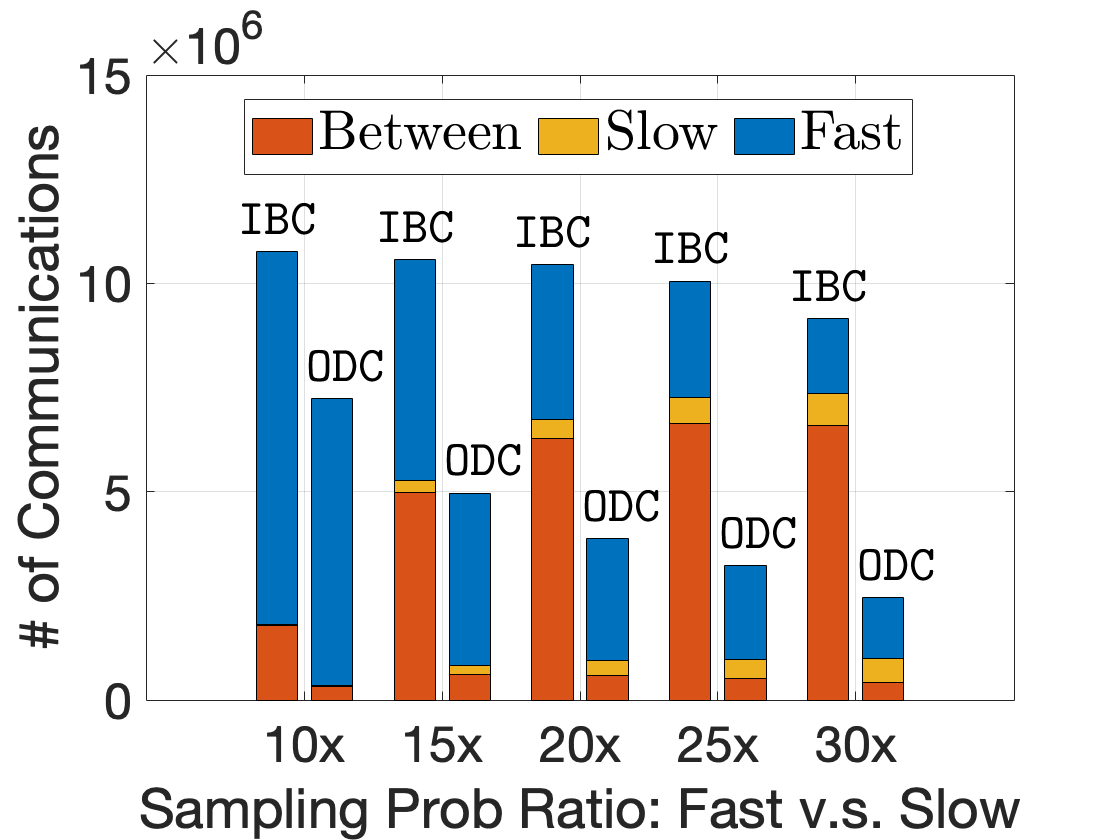}
 	\label{fig:exp6-AAE-comm}}
 	\hfill
	\subfloat[UCB and AAE Regret]{\includegraphics[width=0.24\textwidth]{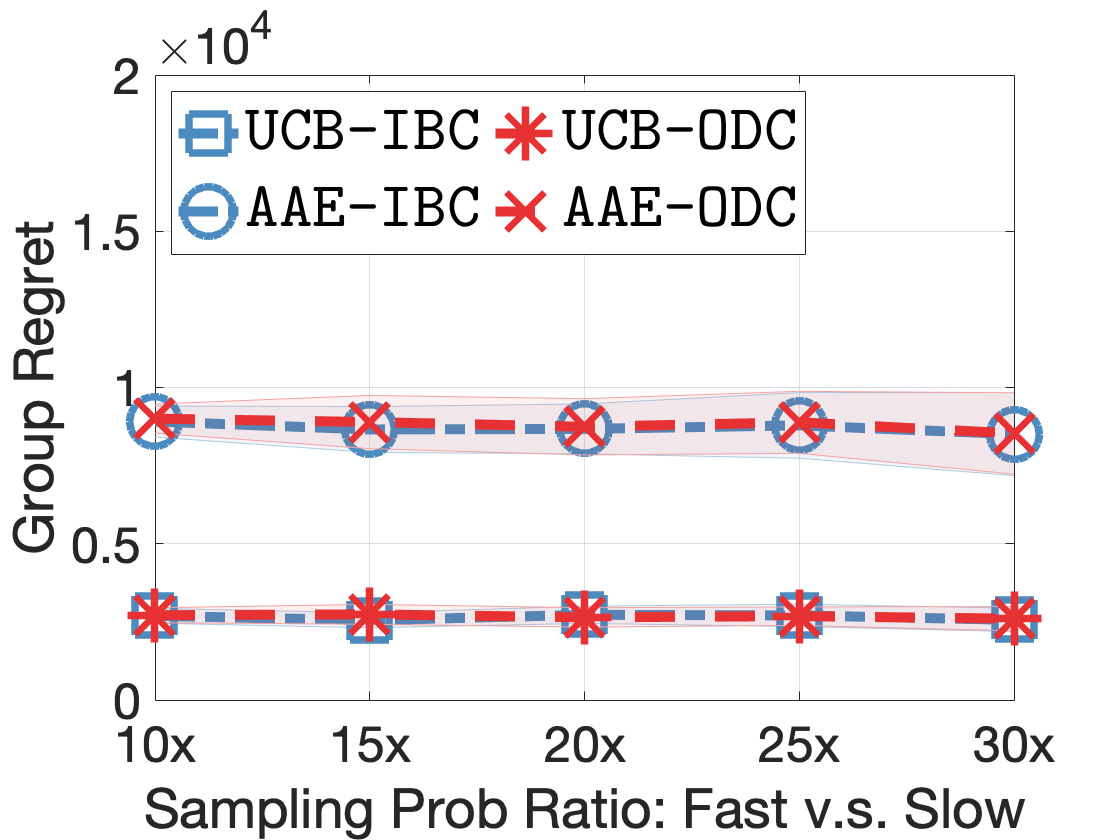}
 	\label{fig:exp6-reg}}
 	\hfill
 	\vspace{-2mm}
	\caption{Experiment 1 --- impact of the heterogeneity of agent speeds. Forty agents with fixed mean sampling probability and increasing sampling probability ratio between fast and slow agents.} %\rev{Notable observation: the number of communications between fast and slow agents (dark orange part) increases under immediate broadcast communication (\texttt{IBC}) while remains similar and relatively small under \ODC as the sampling probability ratio increases. \mh{be consistent with Fig.2}}}
	\label{fig:experiment1}
	\vspace{-2.5mm}
\end{figure*}
\begin{figure*}[t]
	\centering
	\hfill
 	\subfloat[UCB Communication]{\includegraphics[width=0.24\textwidth]{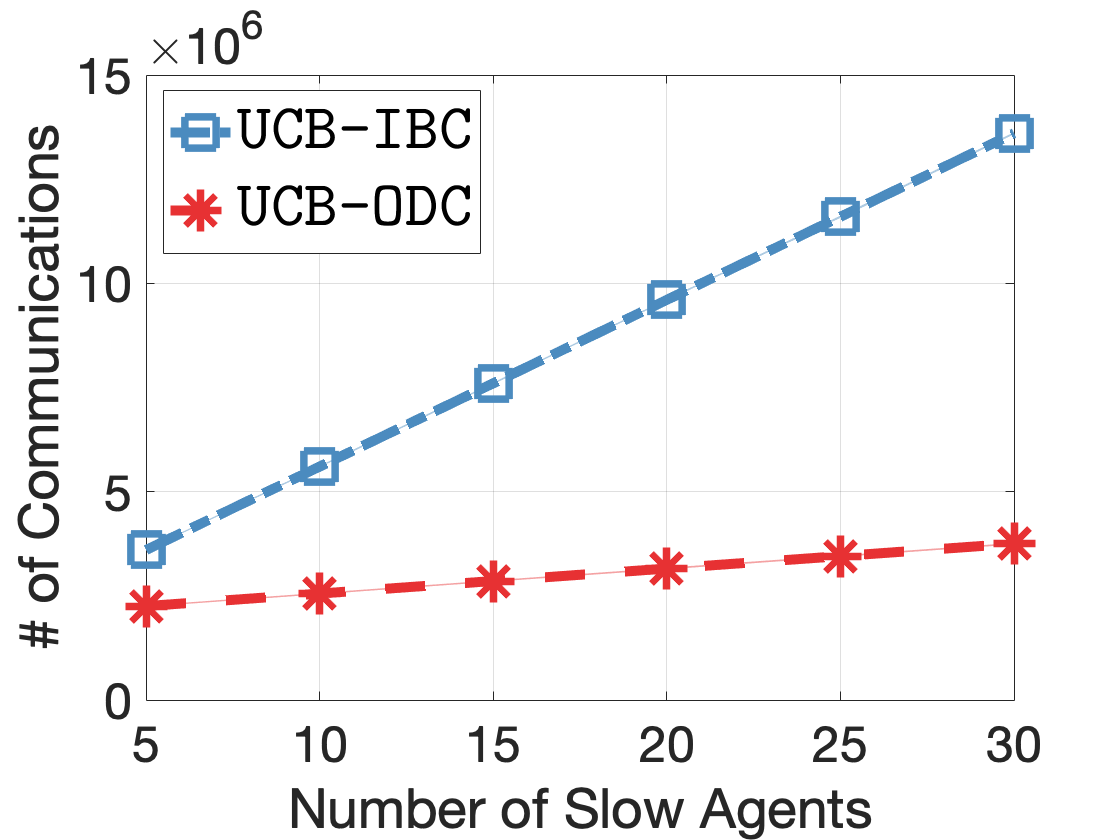}
	\label{fig:exp7-UCB-comm}}
	\hfill
	%\subfigure[UCB Regret]{\includegraphics[width=0.24\textwidth]{figure/exp2-ucb-regret.png}}
	\subfloat[AAE Communication]{\includegraphics[width=0.24\textwidth]{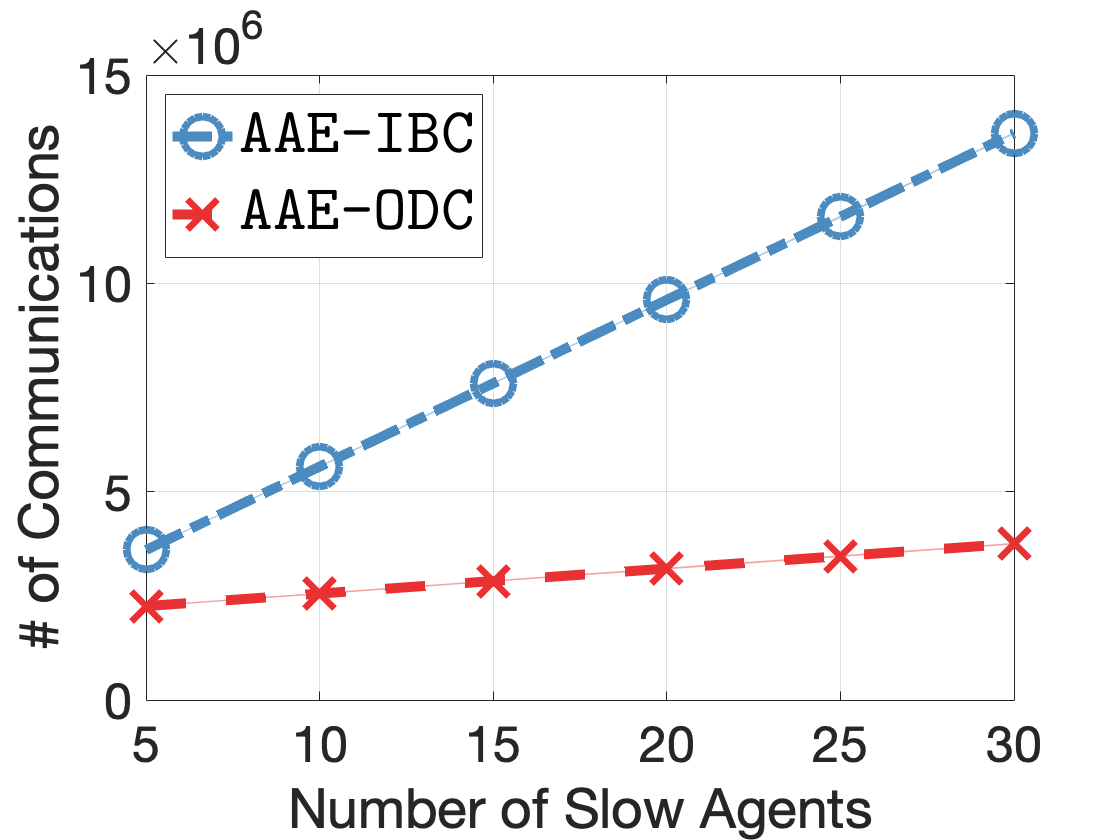}
	\label{fig:exp7-AAE-comm}}
	\hfill
	\subfloat[UCB and AAE Regret]{\includegraphics[width=0.24\textwidth]{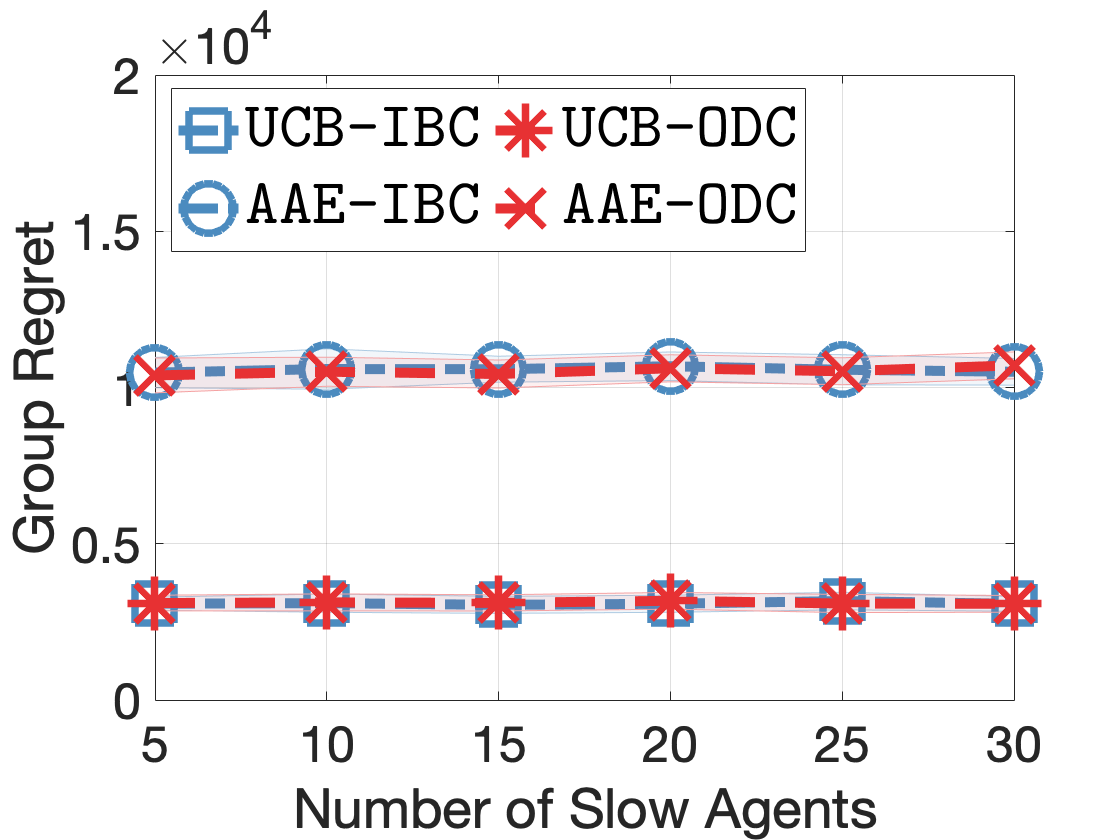}
	\label{fig:exp7-reg}}
	\hfill
	\vspace{-2mm}
	\caption{Experiment 2 --- impact of the number of slow agents. Increasing the number of slow agents while fixing the expected total number of decisions in the entire system. 
% 	$60,000$ time slots.
	}
	\label{fig:experiment2}
	\vspace{-2.5mm}
\end{figure*}

\section{NUMERICAL EXPERIMENTS}\label{sec:numerical} 

In this section, we first study the impact of differences in agent pull rates (Experiment 1) and number of slow agents in the system (Experiment 2) on communication complexity and group regret. We compare \UCBODC and \AAEODC with their counterparts that use immediate broadcast communication, labeled \texttt{UCB-IBC} and \texttt{AAE-IBC}
\footnote{\texttt{UCB-IBC} and \texttt{AAE-IBC} are essentially the same as \texttt{CO-UCB} and \texttt{CO-AAE} respectively proposed in~\cite{yang2022distributed}.
}  
in the first two experiments
with \BTs set to one
to better demonstrate the insights of \ODC.
In Experiment 3, we study the four above algorithms when \BTs are allowed to double after each communication.
%We present supplementary experiments and provide more insights on simulation results in Appendix~\ref{sec:extra-simulation}.

\textbf{Experimental Setup.} 
In our experiments, there are $M=40$ agents each with $K=16$ arms with Bernoulli rewards whose means are uniformly and randomly taken from \textit{Ad-Clicks}~\citep{adclicks}. 
% Each agent has $16$ arms uniformly and randomly selected from among the $K=80$ global arms.
% We consider a time slotted system where agents make decisions with certain probabilities at each time slot. 
We set $\alpha=3$ for all algorithms and 
report values averaged over $30$ independent trials. We report the average cumulative group regret after $T=80,000$ time slots for three experimental scenarios.

\textbf{Experiment 1.} In this experiment, we study the impact of differences in agent pull rates on communication complexity and group regret. Specifically, we fix the expected total number of decisions in the entire system by fixing the mean sampling probabilities among $40$ agents, and we increase the sampling probability ratio between fast and slow agents from $10\times$ to $30\times$ with a step size of $5$.  
Specifically we fix the sampling probability of each slow agent at $0.01$ and vary that of each fast agent from $0.1$ to $0.3$ with step size $0.05$. Note that we maintain the mean sampling probability among agents at $0.085$. Hence the number of fast (resp. slow) agents decreases (resp. increases) as the sampling probability ratio increases.

% Figure~\ref{fig:experiment1} shows the results of Experiment 1.
Figures~\ref{fig:exp6-UCB-comm} and~\ref{fig:exp6-AAE-comm} report the total amount of communication under the immediate broadcast communication (\texttt{IBC}) and the \ODC protocols. %\todo{one protocol name is all lower case, the other includes uppercase - should be consistent} 
We distinguish three types of communications: (1) between fast and slow agents (orange), (2) among slow agents (yellow), and (3) among fast agents (blue).
Figures~\ref{fig:exp6-UCB-comm} and~\ref{fig:exp6-AAE-comm} show that \ODC reduces the amount of communications in all three categories but most notably between fast and slow agents. 
The amount of communication between fast and slow agents increases under \texttt{IBC} while remaining relatively constant under \ODC as the sampling probability ratio between fast and slow agents increases. This demonstrates that \ODC is communication efficient when there is large difference in the pull rates of fast and slow agents. %a diversity of agent speeds.%; even if that were not the case, \ODC outperforms \texttt{IBC} because of the amount of traffic among fast agents.%JC removed the last part of the sentence because in the same arm set setting, there may not be improvement on the amount of traffic among fast agents.
Figure~\ref{fig:exp6-reg} shows that \UCBODC and \AAEODC exhibit similar group regrets to those of  \texttt{UCB-IBC} and \texttt{AAE-IBC} respectively.

\textbf{Experiment 2.} Next, we study the impact of the number of slow agents on communication complexity and group regret  while fixing the expected total number of decisions in the entire system. We fix the number of fast agents at $5$ and increase the number of slow agents from $5$ to $30$ with steps of $5$.  The sampling probability of a fast agent is always $0.8$. As the expected total number of decisions is fixed at $300,000$, the sampling probability of slow agents decreases from $0.2$ to $0.034$ as the number of slow agents increases.

Figure~\ref{fig:experiment2} reports the results of Experiment 2. 
Figures~\ref{fig:exp7-UCB-comm} and~\ref{fig:exp7-AAE-comm} show that the number of communications of \texttt{UCB-IBC} and \texttt{AAE-IBC} increase significantly as the number of slow agents increases even though the expected total number of decisions does not change; while the amount of communication of \UCBODC and \AAEODC do not change much as the number of slow agents increases.
Figure~\ref{fig:exp7-reg} shows that \UCBODC and \AAEODC still achieve similar group regrets as \texttt{UCB-IBC} and \texttt{AAE-IBC} respectively even though they require fewer message exchanges for cooperation.

\textbf{Experiment 3.} Finally, we study the performance of the four policies, \texttt{UCB-IBC}, \UCBODC, \texttt{AAE-IBC}, and \AAEODC, all with doubling \BTs, which we denote as \texttt{UCB-IBC-D}, \texttt{UCB-ODC-D}, \texttt{AAE-IBC-D}, and \texttt{AAE-ODC-D}, in a system with one fast agent with sampling probability one and nine slow agents with sampling probabilities $0.001$. 
Table~\ref{tab:exp3} summarizes the results, which again verifies our theoretical observation that
\ODC reduces the communications between asynchronous agents while achieving similar group regrets as \texttt{IBC}.

\begin{table}[h]
    \vskip -0.04 in
    \centering
    {\small
    \caption{Experiment 3}
    \begin{center}
    \vskip -0.135 in
    \begin{tabular}{||c || c | c||} 
     \hline
       & Communication & Group Regret \\ [0.1ex]
     \hline
     \texttt{UCB-IBC-D} & $629\pm2$ & $1474\pm89$ \\ 
     \hline
     \texttt{UCB-ODC-D} & $563\pm6$ & $1514\pm114$ \\
     \hline
     \texttt{AAE-IBC-D} & $629\pm2$ & $2679\pm254$ \\ 
     \hline
     \texttt{AAE-ODC-D} & $564\pm5$ & $2672\pm225$ \\
     \hline
    \end{tabular}
    \end{center}
    \label{tab:exp3}
    }
    \vskip -0.1 in
\end{table}

We present supplementary experiments and provide more insights on simulation results in Appendix~\ref{sec:extra-simulation}. Specifically, in Appendix~\ref{sec:extra-simu-thres} we numerically study the performance of \ODC with constant or doubling buffer threshold; in Appendix~\ref{sec:extra-simu-ind} we present the individual regrets in Experiment 1 and Experiment 2; in Appendix~\ref{sec:extra-simu-async} we numerically study the performance of \ODC under different types of asynchronicity; in Appendix~\ref{sec:extra-simu-algo} we demonstrate when \AAEODC would incur fewer communications than \UCBODC.
 
%!TEX root = ODC-CameraReady.tex

\section{CONCLUSION AND FUTURE DIRECTION}\label{sec:limitation}
This paper presented a communication protocol for efficient cooperation in asynchronous multi-agent bandits settings. The communication protocol explicitly adjust the amount of cooperation in proportion to agent pull rates and could be integrated into an underlying bandit algorithm. We combined the proposed communication protocol with two bandit algorithms and analyzed their performance in terms of regret and communication complexities. 

A limitation of this work is that we assume 
all messages are sent through reliable communication, e.g., TCP protocol. \ODC suffers potential performance degradation when it is used under unreliable communication, e.g., UDP protocol. Specifically, \ODC suffers performance degradation if there is packet loss in communication. For example, after agent $j$ sends a message to agent $j^\prime$ and sets $E^{j\rightarrow {j^\prime}} \gets \texttt{False}$, if this sharing message is lost without reaching agent $j^\prime$, then the cooperation between agent $j$ and $j^\prime$ will end. This is because from both agents' perspectives, each other's exchange demands are both $\texttt{False}$. Designing a loss-tolerant communication protocol for asynchronous MAMAB is an interesting open problem.

\subsubsection*{Acknowledgements}
We thank the anonymous reviewers for their useful comments. Yu-Zhen Janice Chen (yuzhenchen@cs.umass.edu) and Don Towsley's (towsley@cs.umass.edu) research is supported by U.S. Army Research Laboratory under Cooperative Agreement W911NF-17-2-0196 (IoBT CRA). Xuchuang Wang (xcwang@cse.cuhk.edu.hk), Xutong Liu (liuxt@cse.cuhk.edu), and John C.S. Lui's (cslui@cse.cuhk.edu.hk) research is supported by the RGC’s GRF 14215722. Mohammad Hajiesmaili's (hajiesmaili@cs.umass.edu) research is supported by NSF CAREER-2045641, CPS-2136199, CNS-2106299, and CNS-2102963. Correspondence to: Lin Yang (linyang@nju.edu.cn).

\bibliography{ref}

% If you have textual supplementary material
\appendix
\onecolumn

%!TEX root = ODC-CameraReady.tex

\section{LITERATURE REVIEW}\label{sec:related}

\paragraph{Collision or no collision.}
One of the extensively studied MAMAB settings is the \textit{collision} scenario~\citep{wang2020optimal, boursier2019sic, shi2021heterogeneous, bistritz2018distributed, bubeck2020non, besson2018multi},  where agents receive zero or degraded rewards if they pull the same arm simultaneously. This setting well models the opportunistic spectrum access applications with multiple users, where the objective is to choose the best channels while avoiding users communicate through the same channel at the same time.
On the other hand, the MAMAB setting with \textit{no collision}~\citep{shi2021federated, wang2019distributed, wang2020optimal, bar2019individual, chakraborty2017coordinated, dubey2020cooperative, szorenyi2013gossip, chawla2020gossiping, landgren2016distributed, buccapatnam2015information, martinez2019decentralized, bistritz2020cooperative, madhushani2021one, chakraborty2017coordinated, cesa2016delay, hillel2013distributed, dubey2020cooperative, yang2021cooperative, yang2022distributed, sankararaman2019social, feraud2019decentralized} has also attracted increasing research interest. 
In the MAMAB setting with no collision, agents receive independent rewards without any degradation even when they pull the same arm. 
This setting is more suitable for modeling applications like recommender systems, clinical trials, robotic taget searching, etc.
In this paper, we focus on the \textit{no collision} setting.

\paragraph{Cooperate with or without a coordinator.}
Regarding cooperation methods in cooperative MAMAB, there are two broad categories of prior work: (1) \textit{cooperation with coordinator}~\citep{shi2021federated, wang2019distributed, wang2020optimal, bar2019individual, chakraborty2017coordinated, dubey2020cooperative}, which utilizes a central server or elects leaders among agents to coordinate the learning process. (2) \textit{cooperation without coordinator}~\citep{szorenyi2013gossip, chawla2020gossiping, landgren2016distributed, buccapatnam2015information, martinez2019decentralized, bistritz2020cooperative, madhushani2021one, chakraborty2017coordinated, cesa2016delay, hillel2013distributed, dubey2020cooperative, yang2021cooperative, yang2022distributed, sankararaman2019social, feraud2019decentralized}, which addresses a decentralized learning scenario where agents communicate with each other to improve their learning performance.
In this work, we consider the \textit{cooperation without coordinator} (decentralized) approach in an asynchronous MAMAB setting where agent pull times and speeds being unknown, irregular, and different, hinders the application of a coordination approach. 

\paragraph{Reward assumptions.}
Similar to standard bandit problem, various reward assumptions are studied in decentralized cooperative MAMAB model.
For example,~\cite{szorenyi2013gossip, chawla2020gossiping, landgren2016distributed, buccapatnam2015information, martinez2019decentralized} consider stochastic bandit,~\cite{dubey2020cooperative} studies stochastic bandit with heavy tails, and~\cite{cesa2016delay} considers non-stochastic bandit.
In this work, we consider arms with stochastic rewards and assume they have Bernoulli distributions. 

\paragraph{Homogeneous or heterogeneous arm sets.}
In decentralized cooperative MAMAB model, agents can have homogeneous arm sets or heterogeneous arm sets.
Homogeneous arm sets setting~\citep{szorenyi2013gossip, landgren2016distributed, buccapatnam2015information, martinez2019decentralized}, i.e., same set of arms is available to each agent, is more extensively studied.
Regarding heterogeneous arm sets scenario, there are two different notions of heterogeneous arm sets as far as we notice. One refer to the scenario where agents have access to the same set of arms but each agent receive different expected reward from the same arm, e.g., in~\cite{hossain2021fair}. This setting models the opportunistic spectrum access application and mobile sensor environment estimating application, where the geographical location of agents influence the rewards they receive from the same arm. 
The other definition of heterogeneous arm sets models the scenario that agents receive same expected rewards from the same arm but each agent only have access to a subset of all the arms, e.g, in~\cite{yang2022distributed, chawla2020gossiping, yang2021cooperative}.
In this work, we mainly consider homogeneous arm sets setting. We provide extension of our results to account for agents pulling from different but overlapping subsets of arms in Appendix.

\paragraph{Synchronous or asynchronous agents.}
In decentralized cooperative MAMAB model, agents can operate synchronously or asynchronously. 
The synchronous setting~\citep{szorenyi2013gossip, chawla2020gossiping, landgren2016distributed, buccapatnam2015information, martinez2019decentralized} is more extensively studied. In the synchronous setting, there is a common clock among all agents, and every agent pulls an arm at every time slot. 
\cite{yang2021cooperative, yang2022distributed,sankararaman2019social, feraud2019decentralized} addresses MAMAB with asynchronous agents.
The model in~\cite{yang2021cooperative,yang2022distributed} assumes each agent periodically make decisions at different \textit{known} frequencies. \cite{sankararaman2019social} assumes each agent is equipped with a Poisson clock and agent pull when its clock rings. \cite{feraud2019decentralized} assumes there is a distribution determining which agent becomes active at each time slot.
Our paper assumes that pulling times are unknown, irregular, and not necessarily stochastic.
Last, asynchronous multi-agent learning has also been studied in related fields such as online (convex) optimization with full information or semi-bandit feedback~\citep{cesa2020cooperative, jiang2021asynchronous, joulani2019think, bedi2019asynchronous, della2021efficient}. 
  
\paragraph{Communication schemes.}  
Many different types of communication has been studied in decentralized cooperative MAMAB literature. 
For example,~\cite{szorenyi2013gossip} considers peer-to-peer network and let each agent communicate to only a fixed number of agents at each round.
\cite{chawla2020gossiping, sankararaman2019social} considers a gossip-style communication, where agents are assumed located on a graph and agents can only communicate with their neighbors. The gossip-style communication can be used to model the scenario that users in a social network explore restaurants and make recommendations to their friends. In~\cite{buccapatnam2015information, yang2021cooperative, yang2022distributed}, each agent is allowed to immediately broadcast the rewards to all other agents.
In this work, we consider that each agent is allowed to communicate with every other agent and design a communication protocol that is efficient when agents operate asynchronously. The proposed on-demand communication protocol, \ODC, is a fundamentally different idea from previously considered communication protocols in decentralized cooperative MAMAB literature, such as immediate broadcasting, peer-to-peer~\citep{dubey2020differentially}, consensus-based~\citep{martinez2019decentralized}, and gossip-style~\citep{sankararaman2019social} communication, under which agents spontaneously transmit information.

%!TEX root = ODC-CameraReady.tex

\section{PROOF OF THEOREM~\ref{thm:group-regrets}}\label{sec:group-regrets-proof}

\subsection{Proof of Theorem~\ref{thm:group-regrets}(a)}

To proceed with the proof of expected group regret of \UCBODC, we first state some intermediary lemmas and then use the lemmas to upper bound the group regret.
The first two lemmas are regarding two types of decisions, namely Type-I and Type-II.
\begin{myDef}
	At any decision round $t$, the decision of agent $j$ is a Type-I decision if the following equation holds
	\begin{equation}\label{eq:type-I}
		\mu(i) \in [\hat{\mu}^t_j(i) -\texttt{CI}^t_j(i), \hat{\mu}^t_j(i) +\texttt{CI}^t_j(i)], \quad \forall i \in \mathcal{K};
	\end{equation}
	otherwise the decision is a Type-II decision.
\end{myDef}

\begin{lemma}\label{lemma:decision-type}
	At any decision round $t$, an agent $j$ makes a Type-I decision with a probability at least $1 - 2KN_t{\delta^t_{j}}^\alpha$, where $N_t = \sum_{j\in \mathcal{A}} n_j^t$.
\end{lemma}

\begin{proof}{of Lemma~\ref{lemma:decision-type}}
	Note that for any arm $i$ with $n$ observations, by Hoeffding's inequality and union bound, we have
	\begin{align*}
	\mathbb{P}\left(\biggl|\mu(i) - \hat{\mu}(i, n)\biggl|> \sqrt{\frac{\alpha \log \delta^{-1}}{2n}}\right) \leq 2\delta^\alpha.
	\end{align*}
	Thus, the probability that the true mean value of arm $i$ is not in the confidence interval when agent $j$ makes a decision at time $t$ is at most $2 N^t{\delta^t_{j}}^\alpha$, where $N_t = \sum_{j\in \mathcal{A}} n_j^t$, as shown in the following.
	\begin{align*}
	&\mathbb{P}\left(\biggl| \mu(i) - \hat{\mu}(i, \hat{n}^t_j(i)) \biggl|>  \sqrt{\frac{\alpha \log {\delta^t_{j}}^{-1}}{2\hat{n}^t_j(i)}}\right) \notag\\
	&\leq \sum_{s=1}^{N_t}\mathbb{P}\left(\biggl| \mu(i) - \hat{\mu}(i, n) \biggl|>  \sqrt{\frac{\alpha \log {\delta^t_{j}}^{-1}}{2n}}\Biggl|n=s\right)\mathbb{P}(n=s)\leq 2N_t{\delta^t_{j}}^\alpha.
	\end{align*}
	Hence, the probability that~\eqref{eq:type-I} holds for all arm $i\in\mathcal{K}$ is lower bounded by $1 - \sum\limits_{i \in \mathcal{K}}2N_{t}{\delta^t_{j}}^\alpha = 1-2KN_t{\delta^t_{j}}^\alpha$.
\end{proof}

By Lemma~\ref{lemma:decision-type}, with probability at most $2KN_t{\delta^t_{j}}^\alpha$, an agent makes a Type-II decision at a decision round.
With $\delta^t_{j} = 1/N$,
the expected number of Type-II decisions made by all agents over the entire time horizon, denoted by $\mathbb{E}[Q_{\text{II}}]$ is upper bounded by
\begin{align}
& \mathbb{E}[Q_{\text{II}}] \leq \sum_{j\in\mathcal{A}}\sum_{l=1}^{N_j} 2K N_{t_l} {\delta^{t_l}_{j}}^\alpha
= \sum_{j\in\mathcal{A}}\sum_{l=1}^{N_j} 2K \frac{N_{t_l}}{N^\alpha}
\leq\sum_{j\in\mathcal{A}}\sum_{l=1}^{N_j} 2K \frac{1}{N^{\alpha-1}}\nonumber\\
\label{eq:qii}
& \overset{\text{(a)}}{\leq} \sum_{j\in\mathcal{A}} \frac{2K}{\alpha-2}\left(1-\frac{1}{N^{\alpha-2}}\right)
\equiv q_{\text{II}} \overset{\text{(b)}}{\leq} 2KM.
\end{align}
In Eq.~\eqref{eq:qii}, (a) holds because $l\leq N$ for each $l$ and
(b) holds if $\alpha \geq 3$.

\begin{lemma}\label{lemma:type-I-suboptimal-ucb}
	If agent $j \in \mathcal{A}$ makes a Type-I decision and pulls suboptimal arm $i \in \mathcal{K}$ by the \emph{\UCBODC} algorithm, at that decision round $t$ we have
	\begin{align*}
		\hat{n}^t_j(i) \leq \frac{2\alpha \log ({1/\delta^t_j})}{\Delta^2_i}.
	\end{align*}
\end{lemma}
\begin{proof}{of Lemma~\ref{lemma:type-I-suboptimal-ucb}}
	If agent $j \in \mathcal{A}$ makes a Type-I decision and pulls suboptimal arm $i \in \mathcal{K}$ at time $t$ by the \UCBODC algorithm, we have
	\begin{align}\label{eq:CI-delta}
		2\texttt{CI}^t_j(i) \geq \Delta_i.
	\end{align}
	Because otherwise,
	\begin{align*}
		\hat{\mu}(i^*) + \texttt{CI}^t_j(i^*) \geq \mu(i^*) = \mu(i) + \Delta_i > \mu(i) + 2\texttt{CI}^t_j(i) > \hat{\mu}(i) +\texttt{CI}^t_j(i),
	\end{align*}
	contradicting the fact that arm $i$ is pulled by \UCBODC as it has the highest UCB. Rewrite~\eqref{eq:CI-delta} using the definition of $\texttt{CI}^t_j(i)$ in~\eqref{eq:confidence-interval}, we have
	%\begin{align*}
	$$\hat{n}^t_j(i) \leq \frac{2\alpha \log ({1/\delta^t_j})}{\Delta^2_i}.$$
	%\end{align*}
	%This completes the proof.
\end{proof}

Recall that $n_j^t(i)$ denotes the number of times agent $j$ has pulled arm $i$ up to time $t$.
In the cooperative learning process, there must exist a time slot $\tau_i$ for each subooptimal arm $i\in \mathcal{K} \setminus \{i^*\}$ such that
\begin{align}
	%&\frac{2\alpha \log N}{\Delta_i^2} \geq \sum_{j^\prime \in \mathcal{A}} n_{j^\prime}^{\tau_i-1}(i),\\
	\frac{2\alpha \log N}{\Delta_i^2} + M \geq \sum_{j^\prime \in \mathcal{A}} n_{j^\prime}^{\tau_i}(i) > \frac{2\alpha \log N}{\Delta_i^2} \geq \sum_{j^\prime \in \mathcal{A}} n_{j^\prime}^{\tau_i-1}(i). \label{eq:obs-at-tau}
\end{align}
%After time $\tau$, the extra number of times that arm $i$ is pulled can be under Type-I or Type-II decisions.
The number of times arm $i$ pulled after time $\tau_i$ are considered as extra number of pulls.
These extra pulls are because of three possible causes: 1) Type-I decision due to delayed transmission for waiting for exchange demands, 2) Type-I decision due to delayed transmission for waiting for \BTs to be satisfied, 3) Type-II decisions.

We first examine Type-I decision cases.
Note that $\hat{n}_j^t(i)$ is the total number of observations of arm $i$ agent $j$ possessed at time $t$, including both the number of times agent $j$ pulls arm $i$ and some or all number of times other agents in $\mathcal{A}$ pull arm $i$.
We define $B_t^{j\rightarrow j^\prime}(i)$ as the number of reward samples of arm $i$ stored in agent $j$'s buffer for agent $j^\prime$ (and not yet been sent) at time $t$, and define $B_t^{j\rightarrow j^\prime}$ as the total number of observations stored in agent $j$'s buffer for agent $j^\prime$ at time $t$.
Consider an agent $j \in \mathcal{A}$ and a suboptimal arm $i$ such that, at time $\tau_i$,
\begin{align}
	\frac{2\alpha \log N}{\Delta^2_i}\geq \frac{2\alpha \log 1/\delta^{\tau_i}_j}{\Delta^2_i}
	 & \geq \hat{n}^{\tau_i}_j(i) = n^{\tau_i}_j(i) + \sum_{j^\prime \in \mathcal{A}\setminus \{j\}} n^{\tau_i}_{j^\prime}(i) - B_{\tau_i}^{j^\prime \rightarrow j}(i)                                                                                                                                                                                                                                           \\
	 & \overset{(a)}{>} \frac{2\alpha \log N}{\Delta_i^2} - \sum_{j^\prime \in \mathcal{A}\setminus \{j\}} B_{\tau_i}^{j^\prime \rightarrow j}(i)                                                                                                                                                                                                                                                                \\
	 & = \frac{2\alpha \log N}{\Delta_i^2} - \sum_{j^\prime \in \mathcal{A}\setminus \{j\}} B_{\tau_i}^{j^\prime \rightarrow j}(i) \mathds{1}_{E_{\tau_i}^{j^\prime \rightarrow j} = \texttt{false}}- \sum_{j^\prime \in \mathcal{A}\setminus \{j\}} B_{\tau_i}^{j^\prime \rightarrow j}(i) \mathds{1}_{E_{\tau_i}^{j^\prime \rightarrow j} = \texttt{true}}                                                   \\
	 & \overset{(b)}{\geq} \frac{2\alpha \log N}{\Delta_i^2} - \sum_{j^\prime \in \mathcal{A}\setminus \{j\}} B_{\tau_i}^{j^\prime \rightarrow j}(i) \mathds{1}_{E_{\tau_i}^{j^\prime \rightarrow j} = \texttt{false}}- \sum_{j^\prime \in \mathcal{A}\setminus \{j\}} f(c_{\tau_i}^{j^\prime \rightarrow j})\mathds{1}_{E_{\tau_i}^{j^\prime \rightarrow j} = \texttt{true}}, \label{eq:delayed-transmission}
\end{align}
where inequality (a) is because of (\ref{eq:obs-at-tau}); inequality (b) is because, for agent $j^\prime \in \mathcal{A}\setminus \{j\}$ such that $E_{\tau_i}^{j^\prime \rightarrow j} = \texttt{true}$, we have $f(c_{\tau_i}^{j^\prime\rightarrow j}) \geq B_{\tau_i}^{j^\prime \rightarrow j} \geq B_{\tau_i}^{j^\prime \rightarrow j}(i) \geq 0$.
According to Lemma~\ref{lemma:type-I-suboptimal-ucb}, such an agent $j$ makes Type-I decisions to pull arm $i$ after time $\tau_i$.

In the following, we bound the extra number of times agent $j$ pulls arm $i$ to make up for the delayed transmission from other agents $j^\prime$. For an agent $j^\prime\in \mathcal{A}\setminus \{j\}$ such that $E_{\tau_i}^{j^\prime \rightarrow j} = \texttt{false}$, if $B_{\tau_i}^{j^\prime \rightarrow j}(i) < f(c_{\tau_i}^{j^\prime \rightarrow j})$, agent $j$ has to make at most $f(c_{\tau_i}^{j^\prime \rightarrow j})$ extra pulls of $i$ to make up for agent $j^\prime$'s delay; if $B_{\tau_i}^{j^\prime \rightarrow j}(i) \geq  f(c_{\tau_i}^{j^\prime \rightarrow j})$, agent $j$ can receive those observations from $j^\prime$ once agent $j$ buffers $f(c_{\tau_i}^{j\rightarrow j^\prime})$ observations for $j^\prime$ and sends a message to $j^\prime$. Hence, because of the delayed transmission from agents $j^\prime \in \mathcal{A}\setminus \{j\}: E_{\tau_i}^{j^\prime \rightarrow j} = \texttt{false}$, agent $j$ pulls arm $i$ after time $\tau_i$ at most following number of times:
\begin{align}\label{eq:obs-delay-demand}
	\sum_{j^\prime \in \mathcal{A}\setminus \{j\}} f(\max\{c_{\tau_i}^{j^\prime \rightarrow j}, c_{\tau_i}^{j \rightarrow j^\prime}\}) \mathds{1}_{E_{\tau_i}^{j^\prime \rightarrow j} = \texttt{false}} \leq \sum_{j^\prime \in \mathcal{A}\setminus \{j\}} f(c_{\tau_i}^{j^\prime \rightarrow j}) \mathds{1}_{E_{\tau_i}^{j^\prime \rightarrow j} = \texttt{false}},
\end{align}
where the inequality is because, by the definition of the \ODC, for any pair of agents $j, j^\prime \in \mathcal{A}$ at any time $t$, if $E_t^{j^\prime \rightarrow j} = \texttt{false}$, $1\geq c_t^{j^\prime \rightarrow j} - c_t^{j \rightarrow j^\prime} \geq 0$.
On the other hand, agents $j^\prime\in \mathcal{A}\setminus \{j\}$ such that $E_{\tau_i}^{j^\prime \rightarrow j} = \texttt{true}$ delay transmission of $\sum_{j^\prime \in \mathcal{A}\setminus \{j\}} B_{\tau_i}^{j^\prime \rightarrow j}(i) \mathds{1}_{E_{\tau_i}^{j^\prime \rightarrow j} = \texttt{true}}$ observations of $i$ to agent $j$ at time $\tau_i$ due to waiting for the \BTs to be satisfied.  To make up for this type of delay, agent $j$ pulls arm $i$ after time $\tau_i$ at most following number of times:
\begin{align}\label{eq:obs-delay-buffer}
	\sum_{j^\prime \in \mathcal{A}\setminus \{j\}} f(c_{\tau_i}^{j^\prime \rightarrow j})\mathds{1}_{E_{\tau_i}^{j^\prime \rightarrow j} = \texttt{true}}.
\end{align}

By~\eqref{eq:obs-delay-demand}~\eqref{eq:obs-delay-buffer} and Lemma~\ref{lemma:type-I-suboptimal-ucb}, agent $j$ contributes at most $F^j_i$ extra numbers of pullings of arm $i$ after time $\tau_i$, where
\begin{align}
	F^j_i = \min\Big\{\Big(\sum_{j^\prime \in \mathcal{A}\setminus\{j\}}f(c_{\tau_i}^{j^\prime \rightarrow j})\Big), \frac{2\alpha \log N}{\Delta_i^2}\Big\}.
\end{align}

We now examine Type-II decision case. According to Lemma~\ref{lemma:decision-type} and (\ref{eq:qii}), the expected number of Type-II decisions made by all agents over the entire time horizon is upper bounded by $2KM$. Since, in our case, $\Delta_i \leq 1, \forall i\in\mathcal{K}$, the regret incurred by Type-II decisions is upper bounded by $2KM$.

The expected group regret can be bounded by
\begin{align}
	\mathbb{E}[R] & = \sum_{j\in\mathcal{A}} \mathbb{E}[R^j_{N_j}]
	= \sum_{j\in\mathcal{A}} \sum_{i\in\mathcal{K}} \Delta_i \mathbb{E}[n_j^T(i)]= \sum_{i\in\mathcal{K}} \Delta_i \Big(\sum_{j\in\mathcal{A}} \sum_{\ell=1}^{N_j} \mathbb{P}[I^j_{t^j_\ell} = i]\Big)                                                               \\
	              & \leq 2KM+\sum_{i\in\mathcal{K}} \Delta_i \Big(\frac{2\alpha \log N}{\Delta_i^2} +M + \sum_{j\in\mathcal{A}}F^j_i\Big)\leq 3KM+\sum_{i\in\mathcal{K}:\Delta_i>0}\Big(\frac{2\alpha \log N}{\Delta_i} +\sum_{j\in\mathcal{A}} F^j_i \Delta_i\Big).
\end{align}
This completes the proof of Theorem~\ref{thm:group-regrets}(a).

\subsection{Proof of Theorem~\ref{thm:group-regrets}(b)}

Similar to the analysis of regret of \UCBODC in previous subsection, we utilize intermediary lemmas regarding Type-I and Type-II decisions to upper bound the group regret.
Agent $j$ makes a Type-I decision if~\eqref{eq:type-I} holds, otherwise it is a Type-II decision.

%As \AAEODC also selects arms based on confidence interval as defined in~\eqref{eq:confidence-interval}, Lemma~\ref{lemma:decision-type} holds for \AAEODC. That is, the upper bound of the expected number of Type-II decisions~\eqref{eq:qii} also holds for \AAEODC.

\begin{lemma}\label{lemma:type-I-suboptimal-aae}
	If agent $j \in \mathcal{A}$ makes a Type-I decision and pulls suboptimal arm $i \in \mathcal{K}$ by \AAEODC algorithm, at that decision round $t$ we have
	\begin{align}
		\hat{n}^t_j(i) \leq \frac{8\alpha \log (1/\delta_j^t)}{\Delta^2_i}.
	\end{align}
\end{lemma}
\begin{proof}{of Lemma \ref{lemma:type-I-suboptimal-aae}}
	If agent $j \in \mathcal{A}$ makes a Type-I decision and pulls suboptimal arm $i \in \mathcal{K}$ at time $t$ by \AAEODC algorithm, we have
	\begin{align}
		2\texttt{CI}^t_j(i^*) + 2\texttt{CI}^t_j(i) \geq \Delta_i.
	\end{align}
	Because otherwise,
	\begin{align}
		\hat{\mu}(i^*) - \texttt{CI}^t_j(i^*) & = \hat{\mu}(i^*) + \texttt{CI}^t_j(i^*) - 2 \texttt{CI}^t_j(i^*)\notag                          \\
		                                      & \geq \mu(i^*) - 2 \texttt{CI}^t_j(i^*) =\mu(i) + \Delta_i - 2 \texttt{CI}^t_j(i^*)\notag        \\
		                                      & > \mu(i) + 2\texttt{CI}^t_j(i)\geq \hat{\mu}(i) + \texttt{CI}^t_j(i), \label{eq:aae-contradict}
	\end{align}
	contradicting the fact that arm $i$ is pulled by \AAEODC as it is in the candidate set (if \eqref{eq:aae-contradict} holds, arm $i$ should not be in the candidate set).
	Since \AAEODC pulls the arm with least observations in the candidate set, we have $\hat{n}^t_j(i) \leq \hat{n}^t_j(i^*)$ and thereby $\texttt{CI}^t_j(i) \geq \texttt{CI}^t_j(i^*)$. Rewrite $4\texttt{CI}^t_j(i) \geq \Delta_i$ using the definition of $\texttt{CI}^t_j(i)$ in~\eqref{eq:confidence-interval}, we obtain
	\begin{align*}
		\hat{n}^t_j(i) \leq \frac{8\alpha \log 1/\delta_j^t}{\Delta^2_i}.
	\end{align*}
\end{proof}

We then upper bound the group regret of \AAEODC by similar steps in previous subsection.

Let $\tau_i$ be the time slot for suboptimal arm $i$ that
\begin{align}
	%&\frac{8\alpha \log N}{\Delta^2_i} \geq \sum_{j^\prime \in \mathcal{A}} n^{\tau_i-1}_{j^\prime}(i),\label{eq:aae-begin}\\
	 & \frac{16\alpha \log N}{\Delta^2_i} +M \geq  \sum_{j^\prime \in \mathcal{A}} n^{\tau_i}_{j^\prime}(i) > \frac{16\alpha \log N}{\Delta^2_i}\geq \sum_{j^\prime \in \mathcal{A}} n^{\tau_i-1}_{j^\prime}(i).\label{eq:obs-at-tau-aae}
\end{align}
Consider an agent $j$ and a suboptimal arm $i$ such that at time $\tau_i$,
\begin{align}
	\frac{16\alpha \log N}{\Delta^2_i} \geq \frac{16\alpha \log 1/\delta^{\tau_i}_j}{\Delta^2_i}
	 & \geq \hat{n}^{\tau_i}_j(i) = n^{\tau_i}_j(i) + \sum_{j^\prime \in \mathcal{A}\setminus \{j\}} n^{\tau_i}_{j^\prime}(i) - B_{\tau_i}^{j^\prime \rightarrow j}(i)                                                                                                                                                                                                                                               \\
	%&\overset{(a)}{>} \frac{8\alpha \log N}{\Delta_i^2} - \sum_{j^\prime \in \mathcal{A}\setminus \{j\}} B_{\tau_i}^{j^\prime \rightarrow j}(i)\\
	 & \overset{(a)}{>} \frac{16\alpha \log N}{\Delta_i^2} - \sum_{j^\prime \in \mathcal{A}\setminus \{j\}} B_{\tau_i}^{j^\prime \rightarrow j}(i) \mathds{1}_{E_{\tau_i}^{j^\prime \rightarrow j} = \texttt{false}}\notag- \sum_{j^\prime \in \mathcal{A}\setminus \{j\}} B_{\tau_i}^{j^\prime \rightarrow j}(i) \mathds{1}_{E_{\tau_i}^{j^\prime \rightarrow j} = \texttt{true}}                                  \\
	 & \overset{(b)}{\geq} \frac{16\alpha \log N}{\Delta_i^2} - \sum_{j^\prime \in \mathcal{A}\setminus \{j\}} B_{\tau_i}^{j^\prime \rightarrow j}(i) \mathds{1}_{E_{\tau_i}^{j^\prime \rightarrow j} = \texttt{false}}- \sum_{j^\prime \in \mathcal{A}\setminus \{j\}} f(c_{\tau_i}^{j^\prime \rightarrow j})\mathds{1}_{E_{\tau_i}^{j^\prime \rightarrow j} = \texttt{true}}, \label{eq:delayed-transmission-aae}
\end{align}
where $B_t^{j\rightarrow j^\prime}(i)$ denotes the number of reward samples of arm $i$ stored in agent $j$'s buffer for agent $j^\prime$ (and not yet been sent) at time $t$; $B_t^{j\rightarrow j^\prime}$ denotes the total number of observations stored in agent $j$'s buffer for agent $j^\prime$; inequality (a) is because of (\ref{eq:obs-at-tau-aae}); inequality (b) is because, for agent $j^\prime \in \mathcal{A}\setminus \{j\}$ such that $E_{\tau_i}^{j^\prime \rightarrow j} = \texttt{true}$, we have $f(c_{\tau_i}^{j^\prime\rightarrow j}) \geq B_{\tau_i}^{j^\prime \rightarrow j} \geq B_{\tau_i}^{j^\prime \rightarrow j}(i) \geq 0$.
According to Lemma~\ref{lemma:type-I-suboptimal-aae}, such an agent $j$ makes Type-I decisions to pull arm $i$ after time $\tau_i$.

In the following, we bound the extra number of times agent $j$ pulls arm $i$ to make up for the delayed transmission from other agents $j^\prime$. For an agent $j^\prime\in \mathcal{A}\setminus \{j\}$ such that $E_{\tau_i}^{j^\prime \rightarrow j} = \texttt{false}$, if $B_{\tau_i}^{j^\prime \rightarrow j}(i) < f(c_{\tau_i}^{j^\prime \rightarrow j})$, agent $j$ has to make at most $f(c_{\tau_i}^{j^\prime \rightarrow j})$ extra pullings of $i$ to make up for agent $j^\prime$'s delay; if $B_{\tau_i}^{j^\prime \rightarrow j}(i) \geq  f(c_{\tau_i}^{j^\prime \rightarrow j})$, agent $j$ can receive those observations from $j^\prime$ once agent $j$ buffers $f(c_{\tau_i}^{j\rightarrow j^\prime})$ observations for $j^\prime$ and send a message to $j^\prime$. Hence, because of the delayed transmission from agents $j^\prime \in \mathcal{A}\setminus \{j\}: E_{\tau_i}^{j^\prime \rightarrow j} = \texttt{false}$, agent $j$ pulls arm $i$ after time $\tau_i$ at most following number of times:
\begin{align}\label{eq:obs-delay-demand-aae}
	\sum_{j^\prime \in \mathcal{A}\setminus \{j\}} f(\max\{c_{\tau_i}^{j^\prime \rightarrow j}, c_{\tau_i}^{j \rightarrow j^\prime}\}) \mathds{1}_{E_{\tau_i}^{j^\prime \rightarrow j} = \texttt{false}}\leq \sum_{j^\prime \in \mathcal{A}\setminus \{j\}} f(c_{\tau_i}^{j^\prime \rightarrow j}) \mathds{1}_{E_{\tau_i}^{j^\prime \rightarrow j} = \texttt{false}},
\end{align}
where the inequality is because, by the definition of the \ODC, for any pair of agents $j, j^\prime \in \mathcal{A}$ at any time $t$, if $E_t^{j^\prime \rightarrow j} = \texttt{false}$, $1\geq c_t^{j^\prime \rightarrow j} - c_t^{j \rightarrow j^\prime} \geq 0$.
On the other hand, agents $j^\prime\in \mathcal{A}\setminus \{j\}$ such that $E_{\tau_i}^{j^\prime \rightarrow j} = \texttt{true}$ delay transmission of $\sum_{j^\prime \in \mathcal{A}\setminus \{j\}} B_{\tau_i}^{j^\prime \rightarrow j}(i) \mathds{1}_{E_{\tau_i}^{j^\prime \rightarrow j} = \texttt{true}}$ observations of $i$ to agent $j$ at time $\tau_i$ due to waiting for the \BTs to be satisfied.  To make up for this type of delay, agent $j$ pulls arm $i$ after time $\tau_i$ at most following number of times:
\begin{align}\label{eq:obs-delay-buffer-aae}
	\sum_{j^\prime \in \mathcal{A}\setminus \{j\}} f(c_{\tau_i}^{j^\prime \rightarrow j})\mathds{1}_{E_{\tau_i}^{j^\prime \rightarrow j} = \texttt{true}}.
\end{align}

By~\eqref{eq:obs-delay-demand-aae}~\eqref{eq:obs-delay-buffer-aae}) and Lemma~\ref{lemma:type-I-suboptimal-aae}, agent $j$ contributes at most $G^j_i$ extra numbers of pulls of arm $i$ after time $\tau_i$, where
\begin{align}
	G^j_i = \min\Big\{\Big(\sum_{j^\prime \in \mathcal{A}\setminus\{j\}}f(c_{\tau_i}^{j^\prime \rightarrow j})\Big), \frac{16\alpha \log N}{\Delta_i^2}\Big\}. %\label{eq:aae-end}
\end{align}

We now examine Type-II decision case. 
As \AAEODC also selects arms based on confidence interval as defined in~\eqref{eq:confidence-interval}, Lemma~\ref{lemma:decision-type} holds for \AAEODC. 
%That is, the upper bound of the expected number of Type-II decisions~\eqref{eq:qii} also holds for \AAEODC.
According to Lemma~\ref{lemma:decision-type}, with probability at most $2KN_t{\delta^t_{j}}^\alpha$, an agent makes a Type-II decision at a decision round. 
With $\delta^t_j = 1/N^2$, the regret incurred by Type-II decisions is upper bounded by $2KM$.

The expected group regret can be bounded by
\begin{align}
	\mathbb{E}[R] & = \sum_{j\in\mathcal{A}} \mathbb{E}[R^j_{N_j}]
	= \sum_{j\in\mathcal{A}} \sum_{i\in\mathcal{K}} \Delta_i \mathbb{E}[n_j^T(i)]= \sum_{i\in\mathcal{K}} \Delta_i \Big(\sum_{j\in\mathcal{A}} \sum_{\ell=1}^{N_j} \mathbb{P}[I^j_{t^j_\ell} = i]\Big)                                                                   \\
	              & \leq 2KM + \sum_{i\in\mathcal{K}} \Delta_i \Big(\frac{16\alpha \log N}{\Delta_i^2} +M + \sum_{j\in\mathcal{A}}G^j_i\Big)\leq 3KM + \sum_{i\in\mathcal{K}:\Delta_i>0}\Big(\frac{16\alpha \log N}{\Delta_i} +\sum_{j\in\mathcal{A}} G^j_i \Delta_i\Big).
\end{align}
This completes the proof of Theorem~\ref{thm:group-regrets}(b).

\subsection{Recovery of near-optimal regret in synchronous setting} \label{sec:recover-regret}

When applied to a MAMAB setting with synchronous agents where every agent makes a decision at every time slot, our asynchronous algorithm \AAEODC can recover a near-optimal regret $$O\Big(\sum_{i\in\mathcal{K}:\Delta_i>0}\log N/\Delta_i\Big)$$ with the \BTs set to be doubled and proportional to the number of arms remaining in the candidate set, $\mathcal{C}$, i.e., $f(c^{j\rightarrow j^\prime}) = |\mathcal{C}|\times2^{c^{j\rightarrow j^\prime}-1}, \forall j, j^\prime \in \mathcal{A}$, $c^{j\rightarrow j^\prime}=1,2,\ldots$.

Note that, in a synchronous setting, the exchanges demands $E^{j \rightarrow j^\prime}$ and $E^{j^\prime \rightarrow j}$ are both always \texttt{true}. This is because both exchange demands are \texttt{true} at the beginning, and every time agent $j$ sends a message to agent $j^\prime$, agent $j^\prime$ also sends a message to agent $j$ as the \BTs for all (pairs of) agents are the same in a synchronous setting.

Consider the time $\tau_i$ for a suboptimal arm $i$ such that
\begin{align}
	%&\frac{8\alpha \log N}{\Delta^2_i} \geq \sum_{j^\prime \in \mathcal{A}} n^{\tau_i-1}_{j^\prime}(i),\\
	 & \frac{16\alpha \log N}{\Delta^2_i} +M \geq  \sum_{j^\prime \in \mathcal{A}} n^{\tau_i}_{j^\prime}(i) > \frac{16\alpha \log N}{\Delta^2_i}\geq \sum_{j^\prime \in \mathcal{A}} n^{\tau_i-1}_{j^\prime}(i).\label{eq:tau-def-aae}
\end{align}
Consider an agent $j$ such that at time $\tau_i$,
\begin{align*}
	\frac{16\alpha \log N}{\Delta^2_i} \geq \frac{8\alpha \log 1/\delta^{\tau_i}_j}{\Delta^2_i} \geq \hat{n}^{\tau_i}_j(i) = n^{\tau_i}_j(i) + \sum_{j^\prime \in \mathcal{A}\setminus \{j\}} n^{\tau_i}_{j^\prime}(i) - B_{\tau_i}^{j^\prime \rightarrow j}(i).
\end{align*}
Under \AAEODC, the maximum extra number of times agent $j$ pulls arm $i$ after time $\tau_i$ is at most $f(c_{\tau_i})$. Because after agent $j$ makes $f(c_{\tau_i})$ number of observations it sends a message to all other agents and receives the outstanding observations of arm $i$, $\sum_{j^\prime \in \mathcal{A}\setminus \{j\}} B_{\tau_i}^{j^\prime \rightarrow j}(i)$. 

Hence, the total amount of extra number of times agents pull arm $i$ after time $\tau_i$ can be upper bounded by
\begin{align}
	\sum_{j\in\mathcal{A}} f(c_{\tau_i}) \overset{(a)}{\leq} \sum_{j \in\mathcal{A}}  n_j^{\tau_i}(i)+K \overset{(b)}{\leq} \frac{16\alpha \log N}{\Delta^2_i}+M +K,
\end{align}
where inequality (a) is because, by setting the \BTs to be doubled and proportional to the number of arms remaining in the candidate set, the current \BT of an agent, $f(c_{\tau_i})$, is smaller than or equal to $K$ plus the amount of observations that agent have ever made before $\tau_i$, i.e., $f(c_{\tau_i}) \leq n_j^{\tau_i}(i)+K$; inequality (b) is because of the definition of $\tau_i$ in~\eqref{eq:tau-def-aae}.

%!TEX root = ODC-CameraReady.tex

\section{ASYMPTOTIC GROUP REGRET LOWER BOUND FOR ASYNCHRONOUS MAMAB}\label{sec:lower-bound}

The proof techniques for single-agent multi-armed bandit, e.g., in~\cite{bubeck2012regret}, and those for synchronous multi-agent multi-armed bandit, e.g., in~\cite{dubey2020cooperative}, can be applied to asynchronous multi-agent multi-armed bandit by slight modification. For completion of analysis, we provide the details as follows. 

Let $\mathcal{E}_K$ denote the class of $K$-armed bandit where each arm has a Bernoulli reward distribution and there is no collision, i.e., reward realization of arms is not influenced by actions of agents. Let $\nu = (P_1,...,P_K) \in \mathcal{E}_K$, $\nu^\prime = (P_1^\prime,...,P_K^\prime) \in \mathcal{E}_K$ be two $K$-armed bandit instances such that $P_i = P_i^\prime, \forall i\in\mathcal{K}\setminus\{k\}$, where $k$ is a suboptimal arm.
Specifically, $P_k^\prime$ is chosen to be $\text{Bernoulli}(\mu_k+\lambda)$ and $\lambda > \Delta_k$.
Let $\pi$ denote a consistent cooperative policy for asynchronous $M$-agent multi-armed bandit. We have the following divergence decomposition:
\begin{align}
    \mathbb{D}_\text{KL}(\mathbb{P}_{\nu\pi}, \mathbb{P}_{\nu^\prime \pi})&=
    \mathbb{E}_{\nu\pi}\Bigg[\log \frac{d\mathbb{P}_{\nu\pi}}{d\mathbb{P}_{\nu^\prime \pi}}\Big(I_1^{j:t^j_1=1}, x_1(I_1^{j:t^j_1=1}),...,I_T^{j:t^j_{n_j}=T}, x_T(I_T^{j:t^j_{n_j}=T})\Big)\Bigg]\notag\\
    &=\sum_{i\in\mathcal{K}} \mathbb{E}_{\nu\pi}\left[\sum_{j\in\mathcal{A}} n^T_j(i)\right] \mathbb{D}_{\text{KL}}(P_i, P_i^\prime)= \mathbb{E}_{\nu\pi}\left[\sum_{j\in\mathcal{A}} n^T_j(k)\right] \mathbb{D}_{\text{KL}}(P_k, P_k^\prime),\label{eq:decomposition}
\end{align}
where $n^T_j(i)$ is the total number of times agent $j$ pulls arm $i$, and $\mathbb{P}_{\nu\pi}$, $\mathbb{P}_{\nu^\prime \pi}$ are the distributions of the action-reward history induced by the interaction of policy $\pi$ with bandit instances $\nu$ and $\nu^\prime$ respectively.

By high-probability Pinsker inequality, we have the following for any event $A$:
\begin{align}\label{eq:pinsker}
    \mathbb{D}_{\text{KL}}(\mathbb{P}_{\nu\pi}, \mathbb{P}_{\nu^\prime \pi}) \geq \log \frac{1}{2(\mathbb{P}_{\nu\pi}(A)+\mathbb{P}_{\nu^\prime \pi}(A^c))}.
\end{align}
Let $R$ and $R^\prime$ be the (group) regret obtained by policy $\pi$ on bandit instances $\nu$ and $\nu^\prime$ respectively given the asynchronous pulling times of agents $(t^j_1, t^j_2, ...t^j_{N_j}), \forall j\in\mathcal{A}$.
By~\eqref{eq:decomposition}~\eqref{eq:pinsker} and by choosing $A=\left\{\sum_{j\in\mathcal{A}} n^T_j(k) \geq \frac{1}{2}\sum_{j\in\mathcal{A}} N_j =\frac{N}{2}\right\}$, we have
\begin{align}
    R + R^\prime &\geq \frac{N}{2}\Delta_k\mathbb{P}_{\nu\pi}(A)+ \frac{N}{2}(\lambda -\Delta_k) \mathbb{P}_{\nu^\prime \pi}(A^c) \geq \frac{N}{2} \min\{\Delta_k, \lambda-\Delta_k\} (\mathbb{P}_{\nu\pi}(A)+\mathbb{P}_{\nu^\prime \pi}(A^c))\notag\\
    &\geq \frac{N}{4} \min\{\Delta_k, \lambda-\Delta_k\} \exp\left(-\mathbb{D}_{\text{KL}}(\mathbb{P}_{\nu\pi}, \mathbb{P}_{\nu^\prime \pi})\right)\notag\\
    &= \frac{N}{4} \min\{\Delta_k, \lambda-\Delta_k\}\times\exp\left(-\mathbb{E}_{\nu\pi}\left[\sum_{j\in\mathcal{A}} n^T_j(k)\right] \mathbb{D}_{\text{KL}}(P_k, P_k^\prime)\right). \label{eq:event}
\end{align}
Rearranging~\eqref{eq:event} and taking limit inferior, we have
\begin{align*}
    \liminf_{N\rightarrow \infty}\frac{\mathbb{E}_{\nu\pi}\left[\sum_{j\in\mathcal{A}} n^T_j(k)\right]}{\log(N)}&\geq \frac{1}{\mathbb{D}_{\text{KL}}(P_k, P_k^\prime)}\liminf_{N\rightarrow \infty}\frac{\log(\frac{N \min\{\Delta_k, \lambda-\Delta_k\}}{R+R^\prime})}{\log(N)}\\
    &\geq \frac{1}{\mathbb{D}_{\text{KL}}(P_k, P_k^\prime)}\left(1-\limsup_{N\rightarrow\infty}\frac{\log(R+R^\prime)}{\log(N)}\right).
\end{align*}
By the fact that $\pi$ is consistent, we have some constants $\sigma > 0$ and $C_\sigma$, 
\begin{align}\label{eq:asym}
    \liminf_{N\rightarrow \infty}\frac{\mathbb{E}_{\nu\pi}\left[\sum_{j\in\mathcal{A}} n^T_j(k)\right]}{\log(N)}\geq \frac{1}{\mathbb{D}_{\text{KL}}(P_k, P_k^\prime)}\left(1-\limsup_{N\rightarrow\infty}\frac{a\log N+\log C_\sigma}{\log(N)}\right).
\end{align}
Plugging~\eqref{eq:asym} into the definition of group regret, we have 
\begin{align*}
    \liminf_{N\rightarrow\infty}\frac{R}{\log(N)} \geq \liminf_{N\rightarrow\infty}\frac{\sum_i \mathbb{E}_{\nu\pi}\left[\sum_{j\in\mathcal{A}} n^T_j(k)\right] \Delta_i}{\log(N)}\geq \sum_i \frac{\Delta_i}{\mathbb{D}_{\text{KL}}(P_i, P_i^\prime)}.
\end{align*}
This completes the proof.
%!TEX root = ODC-CameraReady.tex

\section{PROOF OF THEOREM~\ref{thm:comm}}\label{sec:comm-proof}

\subsection{Proof of Theorem~\ref{thm:comm}(a)}\label{sec:ucb-comm-proof}
We claim that, according to the rules of \ODC, an agent $j\in\mathcal{A}$ would not send in total more than 
$$\min\{C_j, C_{j^\prime}\}+1$$ messages to agent $j^\prime \in \mathcal{A}\setminus\{j\}$, where 
$$C_j = \max \Big\{C\in \{1,...,N_j\}: \Big(\sum_{c=1}^C f(c) \Big)\leq N_j\Big\},$$
upper bounds the number of times agent $j$ can fulfill the \BTs when \BTs are updated according to a monotonically increasing function $f$.

Suppose $C_j \leq C_{j^\prime}$. Under \ODC, the number of observations in buffer $\sum_{i\in\mathcal{K}} b_n^{j\rightarrow j^\prime}(i)$ must be greater than or equal to the \BT $f(c)$ when (right before) agent $j$ sends the $c$-th message to agent $j^\prime$. Hence, agent $j$ can send in total at most $C_j$ messages to agent $j^\prime$ when $C_j \leq C_{j^\prime}$. Because if agent $j$ sends in total more than $C_j$ messages, e.g., $C_j +1$ messages, to agent $j^\prime$, that means at least one message transmission violates the rule of \ODC as $\Big(\sum_{c=1}^{C_j+1} f(c) \Big)> N_j$.

Suppose $C_j > C_{j^\prime}$. Under \ODC, the exchange demand $E^{j\rightarrow j^\prime}$ must be \texttt{true} when (right before) an agent $j$ sends a message to agent $j^\prime$. According to the rules of \ODC, the exchange demand $E^{j\rightarrow j^\prime}$ is set to \texttt{true} when: 1) during algorithm initialization, 2) agent $j^\prime$ sends agent $j$ a message. Agent $j^\prime$ can send agent $j$ at most $C_{j^\prime}$ messages when $C_j > C_{j^\prime}$ for otherwise it must violate the \BTs rule of \ODC.
Hence, $E^{j\rightarrow j^\prime}$ is set to \texttt{true} at most $C_{j^\prime} + 1$ times. Then, agent $j$ can send agent $j^\prime$ in total at most $C_{j^\prime} + 1$ messages if agent $j$ follows the exchange demand rule of \ODC.

Now take into account the communications between all pairs of agents, we have the communication complexity:
$$C \leq \sum\nolimits_{j \in \mathcal{A}} \sum\nolimits_{j^\prime \in \mathcal{A}\setminus\{j\}} \min\{C_j, C_{j^\prime}\}+1.$$

This completes the proof.

\subsection{Proof of Theorem~\ref{thm:comm}(b)}\label{sec:aae-comm-proof}

We first upper bound the expected total numbers of Type-I decisions and Type-II decisions made by an agent before \AAEODC stops communication, and then 
we follow similar steps in previous subsection to upper bound the communication complexity of \AAEODC.

Note that agent $j$ makes a Type-I decision if~\eqref{eq:type-I} holds, otherwise it is a Type-II decision. By Lemma~\ref{lemma:decision-type}, with $\delta^t_j =  1/N^2$, the expected number of Type-II decisions made by an agent in the entire learning horizon is upper bounded by $2K$. By Lemma~\ref{lemma:type-I-suboptimal-aae}, with $\delta^t_j = 1/N^2$, the expected number of Type-I decisions made by an agent $j$ before the candidate set size reduces to one can be upper bounded by 
\begin{align*}
    \sum_{i\in\mathcal{K}} \frac{ 16\alpha\log N}{\max\{\Delta^2_i, \Delta^2\}}.
\end{align*}
Note that, $N_j$ is the total number of decisions made by agent $j$.
Hence, the expected number of decisions an agent $j$ makes before agents stop communicate with one another can be upper bounded by  
\begin{align*}
    \min \Big\{2K+\sum_{i\in\mathcal{K}} \frac{ 16\alpha\log N}{\max\{\Delta^2_i, \Delta^2\}}, N_j \Big\}.
\end{align*}

Then, we claim that, under \AAEODC, an agent $j\in\mathcal{A}$ would not send in total more than 
$$\min\{C_j, C_{j^\prime}\}+1$$ messages to agent $j^\prime \in \mathcal{A}\setminus\{j\}$, where 
$$C_j = \max \Big\{C\in \{1,...,N_j\}: \Big(\sum_{c=1}^C f(c) \Big)\leq \min \Big\{2K+\sum_{i\in\mathcal{K}} \frac{ 16\alpha\log N}{\max\{\Delta^2_i, \Delta^2\}}, N_j \Big\}\Big\},$$
upper bounds the number of times agent $j$ executing \AAEODC can fulfill the \BTs when \BTs are updated according to a monotonically increasing function $f$.

Suppose $C_j \leq C_{j^\prime}$. Under \ODC, the number of observations in buffer $\sum_{i\in\mathcal{K}} b_n^{j\rightarrow j^\prime}(i)$ must be greater than or equal to the \BT $f(c)$ when (right before) agent $j$ sends the $c$-th message to agent $j^\prime$. Hence, agent $j$ executing \AAEODC can send in total at most $C_j$ messages to agent $j^\prime$ when $C_j \leq C_{j^\prime}$. Because if agent $j$ sends in total more than $C_j$ messages, e.g., $C_j +1$ messages, to agent $j^\prime$, that means at least one message transmission violates the rule of \ODC or \AAEODC as $\Big(\sum_{c=1}^{C_j+1} f(c) \Big)> \min \Big\{2K+\sum_{i\in\mathcal{K}} \frac{ 16\alpha\log N}{\max\{\Delta^2_i, \Delta^2\}}, N_j \Big\}$.

Suppose $C_j > C_{j^\prime}$. Under \ODC, the exchange demand $E^{j\rightarrow j^\prime}$ must be \texttt{true} when (right before) an agent $j$ sends a message to agent $j^\prime$. According to the rules of \ODC, the exchange demand $E^{j\rightarrow j^\prime}$ is set to \texttt{true} when: 1) during algorithm initialization, 2) agent $j^\prime$ sends agent $j$ a message. Agent $j^\prime$ can send agent $j$ at most $C_{j^\prime}$ messages when $C_j > C_{j^\prime}$ for otherwise it must violate the \BTs rule of \ODC.
Hence, $E^{j\rightarrow j^\prime}$ is set to \texttt{true} at most $C_{j^\prime} + 1$ times. Then, agent $j$ can send agent $j^\prime$ in total at most $C_{j^\prime} + 1$ messages if agent $j$ follows the exchange demand rule of \ODC.

Now take into account the communications between all pairs of agents, we have the communication complexity:
$$C \leq \sum\nolimits_{j \in \mathcal{A}} \sum\nolimits_{j^\prime \in \mathcal{A}\setminus\{j\}} \min\{C_j, C_{j^\prime}\}+1.$$

This completes the proof.

%!TEX root = ODC-CameraReady.tex

\section{PSEUDO CODE OF \UCBODC}\label{sec:pseudo-UCB}
We present \UCBODC in Algorithm~\ref{alg:UCB-ODC}.
\begin{algorithm}
\small
	\caption{The \UCBODC Algorithm for Agent $j$}
	\label{alg:UCB-ODC}
	\begin{algorithmic}[1]
		\State \textbf{Initialize:} 
		exchange demands $E^{j\rightarrow j^\prime} \gets \texttt{True}$, $\forall j^\prime \in \mathcal{A}\setminus \{j\}$, buffers $b_n^{j \rightarrow j^\prime}(i) \gets 0$, $b_\mu^{j \rightarrow j^\prime}(i) \gets 0$, $\forall j^\prime \in \mathcal{A}\setminus \{j\}, i\in \mathcal{K}$, number of communications $c^{j\rightarrow j^\prime} \gets 1$, $\forall j^\prime \in \mathcal{A}\setminus \{j\}$, buffer thresholds $f(c^{j\rightarrow j^\prime}) \gets f(1)$, $\forall j^\prime \in \mathcal{A}\setminus \{j\}$, \texttt{UCB} parameters $\hat{n}_j(i)=0$, $\hat{\mu}_j(i) = 0$, $\forall i\in \mathcal{K}$, $n_j=0$, $\delta_j^t=1/n_j$, $\alpha\geq2$
	    \For{$t = 1...T$}
	        \If{$t$ is a decision time slot of agent $j$, i.e., $t \in \{t^j_1,...,t^j_{N_j}\}$}
	            \State Pull arm $I_t^{j}$ with highest UCB, i.e., $I^j_t \equiv \arg\max_{i\in\mathcal{K}} \hat{\mu}(i) + \texttt{CI}^t_j(i)$, and receive instantaneous reward  $x_t(I^{j}_t)$
	            \State Increase $\hat{n}_j(I_t^j)$ and $n_j$ by $1$, and update the empirical mean value, $\hat{\mu}(I_t^j)$, with instantaneous reward  $x_t(I^{j}_t)$ \label{algline:update-para} 
	            \State Reconstruct the UCBs based on the updated values of $\hat{n}_j(I_t)$, $n_j$, and $\hat{\mu}_j(I_t^j)$ by using Equation~\eqref{eq:confidence-interval}\label{algline:reconstruct} 
	            \For{each agent $j^\prime \in \mathcal{A}\setminus \{j\}$}
	                \State Update the buffer for  agent $j^\prime$: $b_n^{j\rightarrow j^\prime}(I^j_t) \gets b_n^{j\rightarrow j^\prime}(I^j_t) +1$, $b_\mu^{j\rightarrow j^\prime}(I^j_t) \gets b_\mu^{j\rightarrow j^\prime}(I^j_t) + x_t(I^j_t)$
	                \If{$E^{j\rightarrow{j^\prime}}$ is $\texttt{True}$ and $\sum_{i\in \mathcal{K}}b_n^{j\rightarrow j^\prime}(i) \geq f(c^{j \rightarrow j^\prime})$} 
            		\State Share the buffered information with $j^\prime$, i.e., send a message as defined in Definition~\ref{def:msg}, Set $c^{j\rightarrow j^\prime} \gets c^{j\rightarrow j^\prime} + 1$
            		\State Set exchange demand  $E^{j\rightarrow{j^\prime}} \gets \texttt{False}$ and renew the buffer for agent $j^\prime$ 
            		\State Update \BT $f(c^{j\rightarrow j^\prime})$, e.g., double it $f(c^{j\rightarrow j^\prime}) \gets 2f(c^{j\rightarrow j^\prime}-1)$ or keep it the same 
            		\EndIf
	            \EndFor
	        \EndIf
	        \For{each new message received from any agent $j^\prime \in \mathcal{A}\setminus\{j\}$}
	            \State Increase $\hat{n}_j(i), \forall i \in\mathcal{K}$ and update empirical means, $\hat{\mu}_j(i), \forall i \in\mathcal{K}$, according to the information in the message 
	            \State Execute Line (\ref{algline:reconstruct}) to reconstruct UCBs
        		\If{agent $j$ has buffered $f(c^{j \rightarrow j^\prime})$ observations for $j^\prime$, i.e., $\sum_{i\in \mathcal{K}}b_n^{j\rightarrow j^\prime}(i) \geq f(c^{j \rightarrow j^\prime})$}
        		\State Share information by sending a message as defined in Definition~\ref{def:msg} to $j^\prime$, Set $c^{j\rightarrow j^\prime} \gets c^{j\rightarrow j^\prime} + 1$, renew buffer for $j^\prime$
        		\State Update \BT $f(c^{j\rightarrow j^\prime})$, e.g., double it $f(c^{j\rightarrow j^\prime}) \gets 2f(c^{j\rightarrow j^\prime}-1)$ or keep it the same
        		\Else
        		\State Set exchange demand $E^{j\rightarrow{j^\prime}} \gets \texttt{True}$
        		\EndIf
    		\EndFor
	    \EndFor
	\end{algorithmic}
\end{algorithm}

\clearpage
\section{PSEUDO CODE OF \AAEODC}\label{sec:pseudo-AAE}
We present \AAEODC in Algorithm~\ref{alg:AAE-ODC}.
\begin{algorithm}
\small
	\caption{The \AAEODC Algorithm for Agent $j$ }
	\label{alg:AAE-ODC}
	\begin{algorithmic}[1]
		\State \textbf{Initialize:} exchange demands $E^{j\rightarrow j^\prime} \gets \texttt{True}$, $\forall j^\prime \in \mathcal{A}\setminus \{j\}$, buffers $b_n^{j \rightarrow j^\prime}(i) \gets 0$, $b_\mu^{j \rightarrow j^\prime}(i) \gets 0$, $\forall j^\prime \in \mathcal{A}\setminus \{j\}, i\in \mathcal{K}$, number of communications $c^{j\rightarrow j^\prime} \gets 1$, $\forall j^\prime \in \mathcal{A}\setminus \{j\}$, buffer thresholds $f(c^{j\rightarrow j^\prime}) \gets f(1)$, $\forall j^\prime \in \mathcal{A}\setminus \{j\}$, \texttt{AAE} parameters
		$\hat{n}_j(i)=0$, $\hat{\mu}_j(i)=0$, $\forall i\in \mathcal{K}$, $n_j = 0$, $\delta_j^t=1/n_j$, $\alpha\geq 2$, candidate set $\mathcal{C} = \{1, 2,... , K\}$
		\For{$t=1...T$}
		    \If{$t$ is a decision time slot of agent $j$, i.e., $t \in \{t^j_1,...,t^j_{N_j}\}$}
		        \State Recompute confidence intervals $\texttt{CI}^t_j(i), \forall i \in \mathcal{K}$ as defined in~\eqref{eq:confidence-interval}
		        \For{$i\in\mathcal{C}$}
		            \If{$|\mathcal{C}|>1$ and $\exists i^\prime \in \mathcal{K} \text{ s.t. } \hat{\mu}^t_j(i) + \texttt{CI}^t_j(i) < \hat{\mu}^t_j(i^\prime) -\texttt{CI}^t_j(i^\prime)$}
		                \State Eliminate arm $i$ from the candidate set, i.e., $\mathcal{C}\gets \mathcal{C} \setminus \{i\}$
		                \State Broadcast index of arm $i$ to all agents $j \in \mathcal{A}$
		            \EndIf
		        \EndFor
		        \State Pull arm $I_t^j$ from the candidate set $\mathcal{C}$ with the least observations, and receive instantaneous reward  $x_t(I^{j}_t)$
		        \State Increase $\hat{n}_j(I_t^j)$ and $n_j$ by $1$, and update empirical mean, $\hat{\mu}(I_t^j)$, with instantaneous reward  $x_t(I^{j}_t)$
		        \If{$|\mathcal{C}|>1$}
    		        \For{each agent $j^\prime \in \mathcal{A}\setminus \{j\}$}
    	                \State Update the buffer for  agent $j^\prime$: $b_n^{j\rightarrow j^\prime}(I^j_t) \gets b_n^{j\rightarrow j^\prime}(I^j_t) +1$, $b_\mu^{j\rightarrow j^\prime}(I^j_t) \gets b_\mu^{j\rightarrow j^\prime}(I^j_t) + x_t(I^j_t)$
    	                \If{$E^{j\rightarrow{j^\prime}}$ is $\texttt{True}$ and $\sum_{i\in \mathcal{K}}b_n^{j\rightarrow j^\prime}(i) \geq f(c^{j \rightarrow j^\prime})$} 
                		\State Share the buffered information with $j^\prime$, i.e., send a message as defined in Definition~\ref{def:msg}, Set $c^{j\rightarrow j^\prime} \gets c^{j\rightarrow j^\prime} + 1$
                		\State Set exchange demand  $E^{j\rightarrow{j^\prime}} \gets \texttt{False}$ and renew the buffer for agent $j^\prime$ 
                		\State Update \BT $f(c^{j\rightarrow j^\prime})$, e.g., double it $f(c^{j\rightarrow j^\prime}) \gets 2f(c^{j\rightarrow j^\prime}-1)$ or keep it the same 
                		\EndIf
    	            \EndFor
	            \EndIf
		    \EndIf
		    \For{each new message received from any agent $j^\prime \in \mathcal{A}\setminus\{j\}$}
		        \If{it is an elimination notice of arm $i$}
		            \State Eliminate arm $i$ from the candidate set, i.e., $\mathcal{C}\gets \mathcal{C} \setminus \{i\}$
		        \Else
    		        \State Increase $\hat{n}_j(i), \forall i \in\mathcal{K}$ and update empirical means, $\hat{\mu}_j(i), \forall i \in\mathcal{K}$, according to the information in the message 
    		        \If{agent $j$ has buffered $f(c^{j \rightarrow j^\prime})$ observations for $j^\prime$, i.e., $\sum_{i\in \mathcal{K}}b_n^{j\rightarrow j^\prime}(i) \geq f(c^{j \rightarrow j^\prime})$}
            		\State Share information by sending a message as defined in Definition~\ref{def:msg} to $j^\prime$, Set $c^{j\rightarrow j^\prime} \gets c^{j\rightarrow j^\prime} + 1$, renew buffer for $j^\prime$
            		\State Update \BT $f(c^{j\rightarrow j^\prime})$, e.g., double it $f(c^{j\rightarrow j^\prime}) \gets 2f(c^{j\rightarrow j^\prime}-1)$ or keep it the same
            		\Else
            		\State Set exchange demand $E^{j\rightarrow{j^\prime}} \gets \texttt{True}$
            		\EndIf
        		\EndIf
		    \EndFor
		\EndFor
	\end{algorithmic}
\end{algorithm}

%!TEX root = ODC-CameraReady.tex

\clearpage
\section{ACCOUNTING FOR COMMUNICATION DELAY}\label{sec:delay}

Suppose message transmission between agents suffers a deterministic delay, $d$. In the following, we discuss how the communication delays affect the group regrets and communication complexities of \UCBODC and \AAEODC.

For \UCBODC (resp. \AAEODC), we consider time slot $\tau_i$ for each suboptimal arm $i$ such that~\eqref{eq:obs-at-tau} (resp.~\eqref{eq:obs-at-tau-aae}) holds, and consider agent $j$ such that, at time $\tau_i$,~\eqref{eq:delayed-transmission} (resp.~\eqref{eq:delayed-transmission-aae}) holds; by Lemma~\ref{lemma:type-I-suboptimal-ucb} (resp. Lemma~\ref{lemma:type-I-suboptimal-aae}), agent $j$ makes Type-I decisions to pull arm $i$ after time $\tau_i$. In the following, we upper bound the extra number of times (under deterministic communication delays) agent $j$ pulls arm $i$ to make up for the delayed transmission of observations from other agents.

Recall that $B_t^{j\rightarrow j^\prime}(i)$ denotes the number of reward samples of arm $i$ stored in agent $j$'s buffer for agent $j^\prime$ (and not yet been sent) at time $t$, and $B_t^{j\rightarrow j^\prime}$ denotes the total number of observations stored in agent $j$'s buffer for agent $j^\prime$ at time $t$.
For an agent $j^\prime\in \mathcal{A}\setminus \{j\}$ such that $E_{\tau_i}^{j^\prime \rightarrow j} = \texttt{false}$, if $B_{\tau_i}^{j^\prime \rightarrow j}(i) < f(c_{\tau_i}^{j^\prime \rightarrow j})$, agent $j$ has to make at most $f(c_{\tau_i}^{j^\prime \rightarrow j})$ extra pulls of $i$ to make up for agent $j^\prime$'s delay. If $B_{\tau_i}^{j^\prime \rightarrow j}(i) \geq  f(c_{\tau_i}^{j^\prime \rightarrow j})$, agent $j$ can send a message to $j^\prime$ once agent $j$ buffers $f(c_{\tau_i}^{j\rightarrow j^\prime})$ observations for $j^\prime$; the message takes $d$ time slots to reach agent $j^\prime$, and the reply from agent $j^\prime$ with the outstanding observations takes $d$ time slots to reach agent $j$. During the $2d$ time slots, agent $j$ makes at most $2d$ pulls on arm $i$.
Hence, because of the delayed transmission from agents $j^\prime \in \mathcal{A}\setminus \{j\}: E_{\tau_i}^{j^\prime \rightarrow j} = \texttt{false}$, agent $j$ pulls arm $i$ after time $\tau_i$ at most following number of times:
\begin{align}
    \sum_{j^\prime \in \mathcal{A}\setminus \{j\}} 2d+f(\max\{c_{\tau_i}^{j^\prime \rightarrow j}, c_{\tau_i}^{j \rightarrow j^\prime}\}) \mathds{1}_{E_{\tau_i}^{j^\prime \rightarrow j} = \texttt{false}} \leq \sum_{j^\prime \in \mathcal{A}\setminus \{j\}} 2d+f(c_{\tau_i}^{j^\prime \rightarrow j}) \mathds{1}_{E_{\tau_i}^{j^\prime \rightarrow j} = \texttt{false}},
\end{align}
where the inequality is because, by the definition of the \ODC, for any pair of agents $j, j^\prime \in \mathcal{A}$ at any time $t$, if $E_t^{j^\prime \rightarrow j} = \texttt{false}$, $1\geq c_t^{j^\prime \rightarrow j} - c_t^{j \rightarrow j^\prime} \geq 0$.
On the other hand, agents $j^\prime\in \mathcal{A}\setminus \{j\}$ such that $E_{\tau_i}^{j^\prime \rightarrow j} = \texttt{true}$ delay transmission of $\sum_{j^\prime \in \mathcal{A}\setminus \{j\}} B_{\tau_i}^{j^\prime \rightarrow j}(i) \mathds{1}_{E_{\tau_i}^{j^\prime \rightarrow j} = \texttt{true}}$ observations of $i$ to agent $j$ at time $\tau_i$ due to waiting for the \BTs to be satisfied.  To make up for this type of delay, agent $j$ pulls arm $i$ after time $\tau_i$ at most following number of times:
\begin{align}
    \sum_{j^\prime \in \mathcal{A}\setminus \{j\}} f(c_{\tau_i}^{j^\prime \rightarrow j})\mathds{1}_{E_{\tau_i}^{j^\prime \rightarrow j} = \texttt{true}}.
\end{align}

Therefore, the upper bound of expected group regret of \UCBODC under deterministic communication delay $d$ has the same form as~\eqref{eq:UCBODC-regret} in Theorem~\ref{thm:group-regrets}(a) but with $F^j_i$ defined as follows:
\begin{align}
    F^j_i = \min\Big\{\Big(\sum_{j^\prime \in \mathcal{A}\setminus\{j\}}2d+f(c_{\tau_i}^{j^\prime \rightarrow j})\Big), \frac{2\alpha \log N}{\Delta_i^2}\Big\}.
\end{align}
The upper bound of expected group regret of \AAEODC under deterministic communication delay $d$ has the same form as~\eqref{eq:AAEODC-regret} in Theorem~\ref{thm:group-regrets}(b) but with $G^j_i$ defined as follows:
\begin{align}
    G^j_i = \min\Big\{\Big(\sum_{j^\prime \in \mathcal{A}\setminus\{j\}}2d+f(c_{\tau_i}^{j^\prime \rightarrow j})\Big), \frac{16\alpha \log N}{\Delta_i^2}\Big\}. %\label{eq:aae-end}
\end{align}

Under \ODC, any agent $j$ needs the exchange demand $E^{j\rightarrow j^\prime}$ to be set to \texttt{true} to be allowed to send a message to another agent $j^\prime$. Communication delay would never increase the number of times exchange demands be set to \texttt{true}. Hence, the communication complexity upper bounds in Theorem~\ref{thm:comm}(a) and Theorem~\ref{thm:comm}(b) still hold for \UCBODC and \AAEODC respectively
under deterministic communication delay $d$.

%!TEX root = ODC-CameraReady.tex

\section{ACCOUNTING FOR HETEROGENEOUS ARM SETS}\label{sec:heterogeneous}
Agents having different but overlapping arm sets is a practical scenario in MAMAB. In the following, we discuss how to generalize \UCBODC and \AAEODC to account for heterogeneous arm sets.

\subsection{Model Formulation}

We need additional notations for formulating heterogeneous arm set scenario. In this scenario, agents receive the same expected rewards from the same arms but each agent only has access to a \textit{local} subset of the $K$ arms, as in~\cite{yang2022distributed, chawla2020gossiping, yang2021cooperative}. Specifically, agent $j \in \mathcal{A}$ has access to a subset of arms $\mathcal{K}_j \subseteq \mathcal{K}$ \textit{known} to every agent. 
We refer to arms in $\mathcal{K}_j$ as \textit{local} arms of agent $j$. Let $K_j = |\mathcal{K}_j|$. Without loss of generality, we assume that at least two arm sets overlap; i.e., $\exists j, j^\prime \in \mathcal{A}$ s.t. $\mathcal{K}_j \cap \mathcal{K}_{j^\prime} \neq \emptyset$.

Let $i^*_j=\arg\max\nolimits_{i\in \mathcal{K}_j} \mu_i$ denote the local optimal arm of agent $j$.
Let $\mathcal{A}_i$ denote the set of agents whose local arm set includes arm $i$, i.e., $\mathcal{A}_i \equiv \{j \in \mathcal{A}: i \in \mathcal{K}_j\}$. Let $\mathcal{A}_i^*$ denote the set of agents whose local optimal arm is $i$, i.e., $\mathcal{A}_i^* \equiv \{j\in \mathcal{A}_i: i = i^*_j\}$ and let $\mathcal{A}^*_{-i} = \mathcal{A}_i \setminus \mathcal{A}_i^*$. Note that $\mathcal{A}_i^*$ or $\mathcal{A}_{-i}^*$ may be empty. Let $\mathcal{A}^{(j)}$ denote the set of agents that share arms with agent $j$, i.e., $\mathcal{A}^{(j)} \equiv \cup_{i\in\mathcal{K}_j} \mathcal{A}_i\setminus \{j\}$. Let $M_i = |\mathcal{A}_i|$, $M_i^* = |\mathcal{A}_i^*|$, $M_{-i}^* = |\mathcal{A}_{-i}^*|$, and $M^{(j)} = |\mathcal{A}^{(j)}|$.
Let $\Delta(i^*_j, i)$ denote the suboptimality gap of arm $i$ in agent $j$'s local arm set $\mathcal{K}_j$, $i \in \mathcal{K}_j$.
We further denote the smallest suboptimality gap of arm $i$ as $\tilde{\Delta}_i$ and denote the agent that contains $\tilde{\Delta}_i$ as $\tilde{j}_i$, i.e.,
\begin{equation*}
\tilde{\Delta}_i \equiv \begin{cases}
\min_{j \in \mathcal{A}^*_{-i}} \Delta(i^*_j, i),  & \mathcal{A}^*_{-i} \neq \emptyset, \\
0, &\text{otherwise},
\end{cases}
\quad \text{and} \quad
\tilde{j}_i \equiv \begin{cases}
\arg\min_{j \in \mathcal{A}^*_{-i}} \Delta(i^*_j, i),  & \mathcal{A}^*_{-i} \neq \emptyset, \\
0, &\text{otherwise}.
\end{cases}
\end{equation*}
The expected cumulative regret of each agent $j$ becomes
\begin{align*}
   \mathbb{E}[R^j_{N_j}] = \mu(i^*_j)N_j - \mathbb{E}[\sum\nolimits_{t \in\{t^j_1, t^j_2,..., t^j_{N_j}\}} x_t(I^j_t)].
\end{align*}
Expected group regret is $\mathbb{E}[R] = \sum_{j\in\mathcal{A}}\mathbb{E}[R^j_{N_j}]$.

\subsection{Algorithm}
We present the extension of \UCBODC in Algorithm~\ref{alg:UCB-ODC-heterogeneous}. 
\begin{algorithm}
\small
	\caption{The \UCBODC Algorithm for Agent $j$ (with heterogeneous arm sets)}
	\label{alg:UCB-ODC-heterogeneous}
	\begin{algorithmic}[1]
	    \State \textbf{Input:} other agents' local arm sets $(\mathcal{K}_1,..., \mathcal{K}_M)$
		\State \textbf{Initialize:} 
		exchange demands $E^{j\rightarrow j^\prime} \gets \texttt{True}$, $\forall j^\prime \in \mathcal{A}^{(j)}$, buffers $b_n^{j \rightarrow j^\prime}(i) \gets 0$, $b_\mu^{j \rightarrow j^\prime}(i) \gets 0$, $\forall j^\prime \in \mathcal{A}^{(j)}, i\in \mathcal{K}_j\cap\mathcal{K}_{j^\prime}$, number of communications $c^{j\rightarrow j^\prime} \gets 1$, $\forall j^\prime \in \mathcal{A}^{(j)}$, buffer thresholds $f(c^{j\rightarrow j^\prime}) \gets f(1)$, $\forall j^\prime \in \mathcal{A}^{(j)}$, \texttt{UCB} parameters $\hat{n}_j(i)=0$, $\hat{\mu}_j(i) = 0$, $\forall i\in \mathcal{K}_j$, $n_j=0$, $\delta_j^t=1/n_j$, $\alpha\geq2$
	    \For{$t = 1...T$}
	        \If{$t$ is a decision time slot of agent $j$, i.e., $t \in \{t^j_1,...,t^j_{N_j}\}$}
	            \State Pull arm $I_t^{j}$ with highest UCB, i.e., $I^j_t \equiv \arg\max_{i\in\mathcal{K}_j} \hat{\mu}(i) + \texttt{CI}^t_j(i)$, and receive instantaneous reward  $x_t(I^{j}_t)$
	            \State Increase $\hat{n}_j(I_t^j)$ and $n_j$ by $1$, and update the empirical mean value, $\hat{\mu}(I_t^j)$, with instantaneous reward  $x_t(I^{j}_t)$ \label{algline:update-para-heterogeneous} 
	            \State Reconstruct the UCBs based on the updated values of $\hat{n}_j(I_t)$, $n_j$, and $\hat{\mu}_j(I_t^j)$ by using Equation~\eqref{eq:confidence-interval}\label{algline:reconstruct-heterogeneous} 
	            \For{each agent $j^\prime \in \mathcal{A}_{I^j_t}$}
	                \State Update the buffer for  agent $j^\prime$: $b_n^{j\rightarrow j^\prime}(I^j_t) \gets b_n^{j\rightarrow j^\prime}(I^j_t) +1$, $b_\mu^{j\rightarrow j^\prime}(I^j_t) \gets b_\mu^{j\rightarrow j^\prime}(I^j_t) + x_t(I^j_t)$
	                \If{$E^{j\rightarrow{j^\prime}}$ is $\texttt{True}$ and $\sum_{i\in \mathcal{K}_j\cap\mathcal{K}_{j^\prime}}b_n^{j\rightarrow j^\prime}(i) \geq f(c^{j \rightarrow j^\prime})$} 
            		\State Share the buffered information with $j^\prime$, i.e., send a message as defined in Definition~\ref{def:msg}, Set $c^{j\rightarrow j^\prime} \gets c^{j\rightarrow j^\prime} + 1$
            		\State Set exchange demand  $E^{j\rightarrow{j^\prime}} \gets \texttt{False}$ and renew the buffer for agent $j^\prime$ 
            		\State Update \BT $f(c^{j\rightarrow j^\prime})$, e.g., double it $f(c^{j\rightarrow j^\prime}) \gets 2f(c^{j\rightarrow j^\prime}-1)$ or keep it the same 
            		\EndIf
	            \EndFor
	        \EndIf
	        \For{each new message received from any agent $j^\prime \in \mathcal{A}^{(j)}$}
	            \State Increase $\hat{n}_j(i), \forall i \in\mathcal{K}_j \cap \mathcal{K}_{j^\prime}$ and update empirical means, $\hat{\mu}_j(i), \forall i \in\mathcal{K}_j \cap \mathcal{K}_{j^\prime}$, according to the message 
	            \State Execute Line (\ref{algline:reconstruct-heterogeneous}) to reconstruct UCBs
        		\If{agent $j$ has buffered $f(c^{j \rightarrow j^\prime})$ observations for $j^\prime$, i.e., $\sum_{i\in \mathcal{K}_j\cap\mathcal{K}_{j^\prime}}b_n^{j\rightarrow j^\prime}(i) \geq f(c^{j \rightarrow j^\prime})$}
        		\State Share information by sending a message as defined in Definition~\ref{def:msg} to $j^\prime$, Set $c^{j\rightarrow j^\prime} \gets c^{j\rightarrow j^\prime} + 1$, renew buffer for $j^\prime$
        		\State Update \BT $f(c^{j\rightarrow j^\prime})$, e.g., double it $f(c^{j\rightarrow j^\prime}) \gets 2f(c^{j\rightarrow j^\prime}-1)$ or keep it the same
        		\Else
        		\State Set exchange demand $E^{j\rightarrow{j^\prime}} \gets \texttt{True}$
        		\EndIf
    		\EndFor
	    \EndFor
	\end{algorithmic}
\end{algorithm}

\clearpage
We present the extension of \AAEODC in Algorithm~\ref{alg:AAE-ODC-heterogeneous}. Note that, in heterogeneous arm sets setting, each agent needs to maintain a \textit{candidate set}; two agents stop communicating once both of their candidate set sizes reduce one. 

\begin{algorithm}
\small
	\caption{The \AAEODC Algorithm for Agent $j$ (with heterogeneous arm sets)}
	\label{alg:AAE-ODC-heterogeneous}
	\begin{algorithmic}[1]
	    \State \textbf{Input:} Other agents' local arm sets $(\mathcal{K}_1,... , \mathcal{K}_M)$
		\State \textbf{Initialize:} exchange demands $E^{j\rightarrow j^\prime} \gets \texttt{True}$, $\forall j^\prime \in \mathcal{A}^{(j)}$, buffers $b_n^{j \rightarrow j^\prime}(i) \gets 0$, $b_\mu^{j \rightarrow j^\prime}(i) \gets 0$, $\forall j^\prime \in \mathcal{A}^{(j)}, i \in \mathcal{K}_j\cap\mathcal{K}_{j^\prime}$, number of communications $c^{j\rightarrow j^\prime} \gets 1$, $\forall j^\prime \in \mathcal{A}^{(j)}$, buffer thresholds $f(c^{j\rightarrow j^\prime}) \gets f(1)$, $\forall j^\prime \in \mathcal{A}^{(j)}$, \texttt{AAE} parameters
		$\hat{n}_j(i)=0$, $\hat{\mu}_j(i)=0$, $\forall i\in \mathcal{K}_j$, $n_j = 0$, $\delta_j^t=1/n_j$, $\alpha\geq 2$, candidate sets $\mathcal{C}_j = \mathcal{K}_j$ and $\mathcal{C}_{j^\prime} = \mathcal{K}_{j^\prime}, \forall j^\prime \in \mathcal{A}^{(j)}$
		\For{$t=1...T$}
		    \If{$t$ is a decision time slot of agent $j$, i.e., $t \in \{t^j_1,...,t^j_{N_j}\}$}
		        \State Recompute confidence intervals $\texttt{CI}^t_j(i), \forall i \in \mathcal{K}_j$ as defined in~\eqref{eq:confidence-interval}
		        \For{$i\in\mathcal{C}_j$}
		            \If{$|\mathcal{C}_j|>1$ and $\exists i^\prime \in \mathcal{K}_j \text{ s.t. } \hat{\mu}^t_j(i) + \texttt{CI}^t_j(i) < \hat{\mu}^t_j(i^\prime) -\texttt{CI}^t_j(i^\prime)$}
		                \State Eliminate arm $i$ from the candidate set, i.e., $\mathcal{C}_j\gets \mathcal{C}_j \setminus \{i\}$
		                \State Broadcast index of arm $i$ to all agents $j \in \mathcal{A}_i$
		            \EndIf
		        \EndFor
		        \State Pull arm $I_t^j$ from the candidate set $\mathcal{C}_j$ with the least observations, and receive instantaneous reward  $x_t(I^{j}_t)$
		        \State Increase $\hat{n}_j(I_t^j)$ and $n_j$ by $1$, and update empirical mean, $\hat{\mu}(I_t^j)$, with instantaneous reward  $x_t(I^{j}_t)$
		        \For{each agent $j^\prime \in \mathcal{A}_{I^j_t}$ that $|\mathcal{C}_{j^\prime}|>1$ or $|\mathcal{C}_j|>1$}
	                \State Update the buffer for  agent $j^\prime$: $b_n^{j\rightarrow j^\prime}(I^j_t) \gets b_n^{j\rightarrow j^\prime}(I^j_t) +1$, $b_\mu^{j\rightarrow j^\prime}(I^j_t) \gets b_\mu^{j\rightarrow j^\prime}(I^j_t) + x_t(I^j_t)$
	                \If{$E^{j\rightarrow{j^\prime}}$ is $\texttt{True}$ and $\sum_{i\in \mathcal{K}_j\cap\mathcal{K}_{j^\prime}}b_n^{j\rightarrow j^\prime}(i) \geq f(c^{j \rightarrow j^\prime})$} 
            		\State Share the buffered information with $j^\prime$, i.e., send a message as defined in Definition~\ref{def:msg}, Set $c^{j\rightarrow j^\prime} \gets c^{j\rightarrow j^\prime} + 1$
            		\State Set exchange demand  $E^{j\rightarrow{j^\prime}} \gets \texttt{False}$ and renew the buffer for agent $j^\prime$ 
            		\State Update \BT $f(c^{j\rightarrow j^\prime})$, e.g., double it $f(c^{j\rightarrow j^\prime}) \gets 2f(c^{j\rightarrow j^\prime}-1)$ or keep it the same 
            		\EndIf
	            \EndFor
		    \EndIf
		    \For{each new message received from any agent $j^\prime \in \mathcal{A}^{(j)}$}
		        \If{it is an elimination notice of arm $i$ from agent $j^\prime$}
		            \State Eliminate arm $i$ from the candidate set, i.e., $\mathcal{C}_{j^\prime}\gets \mathcal{C}_{j^\prime} \setminus \{i\}$
		        \Else
    		        \State Increase $\hat{n}_j(i), \forall i \in\mathcal{K}_j\cap\mathcal{K}_{j^\prime}$ and update empirical means, $\hat{\mu}_j(i), \forall i \in\mathcal{K}_j\cap\mathcal{K}_{j^\prime}$, according to the message 
    		        \If{agent $j$ has buffered $f(c^{j \rightarrow j^\prime})$ observations for $j^\prime$, i.e., $\sum_{i\in \mathcal{K}_j\cap\mathcal{K}_{j^\prime}}b_n^{j\rightarrow j^\prime}(i) \geq f(c^{j \rightarrow j^\prime})$}
            		\State Share information by sending a message as defined in Definition~\ref{def:msg} to $j^\prime$, Set $c^{j\rightarrow j^\prime} \gets c^{j\rightarrow j^\prime} + 1$, renew buffer for $j^\prime$
            		\State Update \BT $f(c^{j\rightarrow j^\prime})$, e.g., double it $f(c^{j\rightarrow j^\prime}) \gets 2f(c^{j\rightarrow j^\prime}-1)$ or keep it the same
            		\Else
            		\State Set exchange demand $E^{j\rightarrow{j^\prime}} \gets \texttt{True}$
            		\EndIf
        		\EndIf
		    \EndFor
		\EndFor
	\end{algorithmic}
\end{algorithm}

\clearpage
\subsection{Analysis of Regret and Communication Complexity}

\textbf{Expected Group Regret of \UCBODC under Heterogeneous Arm Sets.} 
With algorithm parameters $\delta^t_j =  1/N$ and $\alpha \geq 3$, the expected group regret of \emph{\UCBODC} under heterogeneous arm sets satisfies
	\begin{align}\label{eq:UCBODC-regret-heterogeneous}
	\mathbb{E}[R] 
	\leq 3KM + \sum\limits_{i\in \mathcal{K}: \tilde{\Delta}_i > 0} \Bigg(\frac{4 \alpha \log N }{\tilde{\Delta}_i}+\sum_{j\in\mathcal{A}^*_{-i}}\min\Bigg\{\Big(\sum_{j^\prime \in \mathcal{A}_i\setminus\{j\}}f(c_{\tau_i}^{j^\prime \rightarrow j})\Big), \frac{2\alpha \log N}{\Delta^2(i^*_j, i)}\Bigg\} \Delta(i^*_j, i)\Bigg).
	\end{align}

Recall that $\mathcal{A}^*_{-i}$ is the set of agent with arm $i$ as a local suboptimal arm. Following similar arguments in the proof of Lemma~\ref{lemma:type-I-suboptimal-ucb}, if agent $j\in\mathcal{A}^*_{-i}$ makes a Type-I decision and pulls arm $i\in\mathcal{K}_j$ by \UCBODC algorithm at time $t$, we have that 
\begin{equation}
    \hat{n}^t_j(i) \leq \frac{2\alpha \log (1/\delta^{t}_j)}{\Delta^2(i^*_j, i)}.\label{eq:aae-heterogeneous-type-I}
\end{equation}
Without loss of generality, we let $\mathcal{A}^*_{-i} = \{j_m: m=1,2,...,M_{-i}\}$, where ${\Delta(i^*_{j_1}, i) \geq \Delta(i^*_{j_2}. i) \geq \dots\geq \Delta(i^*_{j_{M_{-i}}}, i)}$ and $M_{-i} = |\mathcal{A}^*_{-i}|$. Agent $j_{M_{-i}}$ needs the most number of observations of arm $i$ to differentiate it from its local optimal arm because $j_{M_{-i}}$ is agent with the smallest $\Delta(i^*_j, i)$ among all $j \in \mathcal{A}^*_{-i}$. Though $j_{M_{-i}}$ is the agent in $\mathcal{A}^*_{-i}$ that needs the most number of observations of arm $i$, each time agent $j_{M_{-i}}$ pulls arm $i$ in fact incur the smallest regret than each time other agents in $\mathcal{A}^*_{-i}$ pull arm $i$ because it has the smallest $\Delta(i^*_j, i)$ among all $j \in \mathcal{A}^*_{-i}$. When those agents $j_m \in \mathcal{A}^*_{-i}$ with largest $\Delta(i^*_{j_m}, i)$s make the most number of pulls of arm $i$, the largest regret on arm $i$ is incurred with the same number of times arm $i$ being pulled. %$\frac{2\alpha \log N}{\tilde{\Delta}^2_i}$. 
With $\delta^t_j \geq 1/N, \forall j \in \mathcal{A}_{-i}$, and $A_{m} = \frac{2 \alpha \log N}{\Delta^2(i^*_{j_m}, i)}$, we have
\begin{align}
&A_1 \Delta(i^*_{j_1}, i) + \sum_{m=1}^{M_{-i}-1}(A_{m+1} -A_m)\Delta(i^*_{j_{m+1}}, i)\\
&=  \sum_{m=1}^{M_{-i}-1} A_m (\Delta(i^*_{j_{m}}, i) -\Delta(i^*_{j_{m+1}}, i)) + A_{M_{-i}}\Delta(i^*_{j_{M_{-i}}}, i)\\
&\leq \int_{\Delta(i^*_{j_{M_{-i}}}, i)}^{\Delta(i^*_{j_1}, i)} \frac{2\alpha\log N}{z^2}dz+\frac{2 \alpha \log N}{\Delta(i^*_{j_{M_{-i}}}, i)}
\leq \frac{4 \alpha \log N}{\Delta(i^*_{j_{M_{-i}}}, i)} = \frac{4 \alpha \log N}{\tilde{\Delta}_i}.
\end{align}

Consider time slot $\tau_i$ for each suboptimal arm $i$ such that 
$$ \frac{2\alpha \log N}{\tilde{\Delta}_i^2} + M \geq \sum_{j^\prime \in \mathcal{A}_i} n_{j^\prime}^{\tau_i}(i) > \frac{2\alpha \log N}{\tilde{\Delta}_i^2} \geq \sum_{j^\prime \in \mathcal{A}_i} n_{j^\prime}^{\tau_i-1}(i).$$
Consider agent $j\in\mathcal{A}_{-i}$ such that, at time $\tau_i$,
\begin{align}
    \frac{2\alpha \log N}{\tilde{\Delta}_i^2}\geq \frac{2\alpha \log 1/\delta^{\tau_i}_j}{\Delta^2(i^*_{j}, i)} &\geq \hat{n}^{\tau_i}_j(i) = n^{\tau_i}_j(i) + \sum_{j^\prime \in \mathcal{A}_i\setminus \{j\}} n^{\tau_i}_{j^\prime}(i) - B_{\tau_i}^{j^\prime \rightarrow j}(i)\\
    &\geq \frac{2\alpha \log N}{\tilde{\Delta}_i^2} - \sum_{j^\prime \in \mathcal{A}_i\setminus \{j\}} B_{\tau_i}^{j^\prime \rightarrow j}(i) \mathds{1}_{E_{\tau_i}^{j^\prime \rightarrow j} = \texttt{false}}- \sum_{j^\prime \in \mathcal{A}_i\setminus \{j\}} f(c_{\tau_i}^{j^\prime \rightarrow j})\mathds{1}_{E_{\tau_i}^{j^\prime \rightarrow j} = \texttt{true}},
\end{align}
where $B_t^{j\rightarrow j^\prime}(i)$ denotes the number of reward samples of arm $i$ stored in agent $j$'s buffer for agent $j^\prime$ (and not yet been sent) at time $t$; $B_t^{j\rightarrow j^\prime}$ denotes the total number of observations stored in agent $j$'s buffer for agent $j^\prime$.
By~\eqref{eq:aae-heterogeneous-type-I}, such agent $j\in\mathcal{A}_{-i}$ makes Type-I decisions to pull arm $i$ after time $\tau_i$. 

In the following, we bound the extra number of times agent $j\in\mathcal{A}_{-i}$ pulls arm $i$ to make up for the delayed transmission from other agents $j^\prime\in\mathcal{A}_{i}$. For an agent $j^\prime\in\mathcal{A}_{i}\setminus \{j\}$ such that $E_{\tau_i}^{j^\prime \rightarrow j} = \texttt{false}$, if $B_{\tau_i}^{j^\prime \rightarrow j}(i) < f(c_{\tau_i}^{j^\prime \rightarrow j})$, agent $j$ has to make at most $f(c_{\tau_i}^{j^\prime \rightarrow j})$ extra pulls of $i$ to make up for agent $j^\prime$'s delay; if $B_{\tau_i}^{j^\prime \rightarrow j}(i) \geq  f(c_{\tau_i}^{j^\prime \rightarrow j})$, agent $j$ can receive those observations from $j^\prime$ once agent $j$ buffers $f(c_{\tau_i}^{j\rightarrow j^\prime})$ observations for $j^\prime$ and sends a message to $j^\prime$. Hence, because of the delayed transmission from agents $j^\prime \in \mathcal{A}_i\setminus \{j\}: E_{\tau_i}^{j^\prime \rightarrow j} = \texttt{false}$, agent $j$ pulls arm $i$ after time $\tau_i$ at most following number of times:
\begin{align}
    \sum_{j^\prime \in \mathcal{A}_i\setminus \{j\}} f(\max\{c_{\tau_i}^{j^\prime \rightarrow j}, c_{\tau_i}^{j \rightarrow j^\prime}\}) \mathds{1}_{E_{\tau_i}^{j^\prime \rightarrow j} = \texttt{false}} \leq \sum_{j^\prime \in \mathcal{A}_i\setminus \{j\}} f(c_{\tau_i}^{j^\prime \rightarrow j}) \mathds{1}_{E_{\tau_i}^{j^\prime \rightarrow j} = \texttt{false}},
\end{align}
where the inequality is because, by the definition of the \ODC, for any pair of agents $j, j^\prime \in \mathcal{A}$ at any time $t$, if $E_t^{j^\prime \rightarrow j} = \texttt{false}$, $1\geq c_t^{j^\prime \rightarrow j} - c_t^{j \rightarrow j^\prime} \geq 0$.
On the other hand, agents $j^\prime\in \mathcal{A}_i\setminus \{j\}$ such that $E_{\tau_i}^{j^\prime \rightarrow j} = \texttt{true}$ delay transmission of $\sum_{j^\prime \in \mathcal{A}_i\setminus \{j\}} B_{\tau_i}^{j^\prime \rightarrow j}(i) \mathds{1}_{E_{\tau_i}^{j^\prime \rightarrow j} = \texttt{true}}$ observations of $i$ to agent $j$ at time $\tau_i$ due to waiting for the \BTs to be satisfied.  To make up for this type of delay, agent $j$ pulls arm $i$ after time $\tau_i$ at most following number of times:
\begin{align}
    \sum_{j^\prime \in \mathcal{A}_i\setminus \{j\}} f(c_{\tau_i}^{j^\prime \rightarrow j})\mathds{1}_{E_{\tau_i}^{j^\prime \rightarrow j} = \texttt{true}}.
\end{align}
Hence, agent $j\in\mathcal{A}_{-i}$ incur at most 
\begin{align}
    \min\Bigg\{\Big(\sum_{j^\prime \in \mathcal{A}_i\setminus\{j\}}f(c_{\tau_i}^{j^\prime \rightarrow j})\Big), \frac{2\alpha \log N}{\Delta^2(i^*_j, i)}\Bigg\} \Delta(i^*_j, i)
\end{align}
extra regret by pulling arm $i$ after time $\tau_i$.

As for Type-II decisions, Lemma~\ref{lemma:decision-type} still holds under heterogeneous arm sets. Thus, the expected regret incurred under Type-II decisions can still be upper bounded by $2KM$. 

Combining the regret upper bounds for Type-II and Type-I decisions, we obtain Eq. (\ref{eq:UCBODC-regret-heterogeneous}).

\textbf{Expected Group Regret of \AAEODC under Heterogeneous Arm Sets.}
With algorithm parameters $\delta^t_j = 1/N^2$ and $\alpha \geq 3$,
    the expected group regret of \emph{\AAEODC} under heterogeneous arm sets satisfies
\begin{equation} 
	\mathbb{E}[R] 
	\leq 3KM + \sum\limits_{i\in \mathcal{K}:\tilde{\Delta}_i>0} \Bigg(\frac{32 \alpha \log N}{\tilde{\Delta}_i} + \sum_{j\in\mathcal{A}^*_{-i}} \min\Bigg\{\Big(\sum_{j^\prime \in \mathcal{A}_i\setminus\{j\}}f(c_{\tau_i}^{j^\prime \rightarrow j})\Big), \frac{16\alpha \log N}{\Delta^2(i^*_j, i)}\Bigg\} \Delta(i^*_j, i)\Bigg).
\end{equation}
The analysis of the expected group regret of \AAEODC under heterogeneous arm sets follows similar steps as the analysis for \UCBODC.

\textbf{Communication Complexity of \UCBODC under Heterogeneous Arm Sets.}
When \BTs are updated according to a positive and monotonically increasing function $f$, the communication complexity \UCBODC under heterogeneous arm sets satisfies:\\
\begin{equation}
    C \leq \sum\nolimits_{j \in \mathcal{A}} \sum\nolimits_{j^\prime \in \mathcal{A}^{(j)}\setminus\{j\}} \min\{C_j, C_{j^\prime}\}+1,
\end{equation}
where $C_j$ is the largest integer in set $\{1,...,N_j\}$ such that
$\Big(\sum_{c=1}^{C_j} f(c)\Big) \leq N_j$.

Under \ODC, any agent $j$ needs the exchange demand $E^{j\rightarrow j^\prime}$ to be set to \texttt{true} to be allowed to send a message to another agent $j^\prime$. Having heterogeneous arm sets would never increase the number of times exchange demands be set to \texttt{true}. Under heterogeneous arm sets, an agent $j$ may make $f(c^{j \rightarrow j^\prime})$ observations but still cannot fulfill the \BT because some of those observations may not be of arm $i \in \mathcal{K}_j\cap\mathcal{K}_{j^\prime}$ and we need $\sum_{i\in \mathcal{K}_j\cap\mathcal{K}_{j^\prime}}b_n^{j\rightarrow j^\prime}(i) \geq f(c^{j \rightarrow j^\prime})$.

\textbf{Communication Complexity of \AAEODC under Heterogeneous Arm Sets.}
When \BTs are updated according to a positive and monotonically increasing function $f$, the communication complexity \AAEODC under heterogeneous arm sets satisfies:\\
\begin{equation}
    C \leq \sum\nolimits_{j \in \mathcal{A}} \sum\nolimits_{j^\prime \in \mathcal{A}^{(j)}\setminus\{j\}} \min\{C_j, C_{j^\prime}\}+1,
\end{equation}
where $C_j$ is the largest integer in set $\{1,...,N_j\}$ such that
\begin{equation}
\Big(\sum_{c=1}^{C_j} f(c)\Big) \leq \min \Big\{ 2K+\sum_{i\in\mathcal{K}_j}\frac{16\alpha\log N}{\max\{\Delta^2(i^*_j, i), \min_{i\in\mathcal{K}_j\setminus\{i^*_j\}}\Delta^2(i^*_j, i)\}}, N_j\Big\}.
\end{equation}

%!TEX root = ODC-CameraReady.tex

\clearpage
\section{SUPPLEMENTARY EXPERIMENTAL RESULTS}\label{sec:extra-simulation}

In this section, we present supplementary numerical experimental results to provide more insights about \ODC protocol.

\begin{figure}
\vskip -0.2 in
    \centering
    \subfloat[UCB Group Regret]{\includegraphics[width=0.245\textwidth]{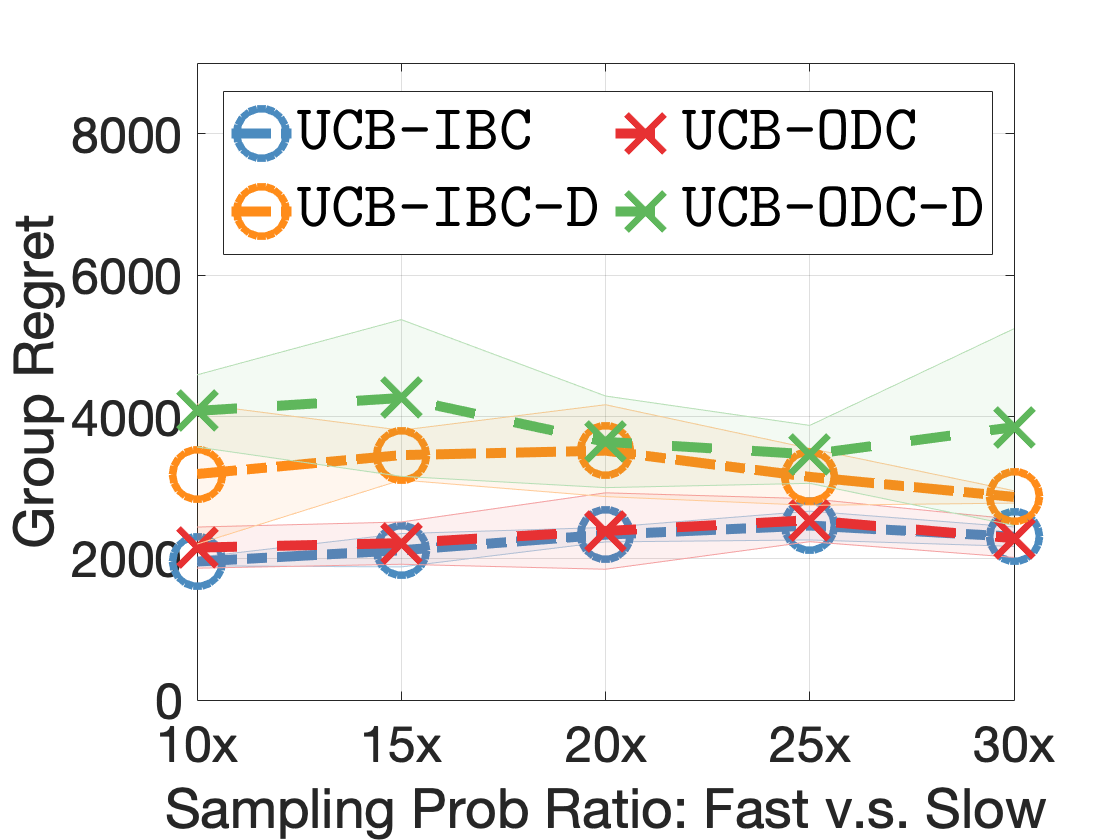}}
    \subfloat[UCB Comm Between]{\includegraphics[width=0.245\textwidth]{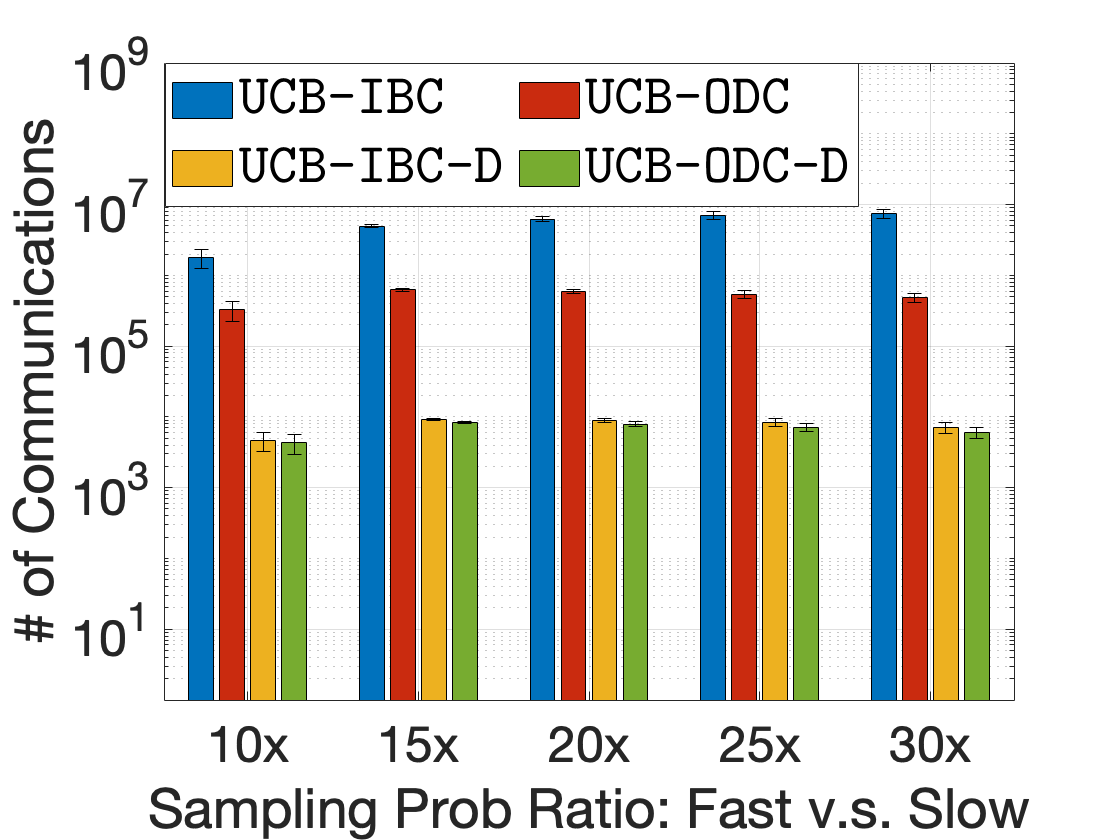}}
    \subfloat[UCB Comm Slow]{\includegraphics[width=0.245\textwidth]{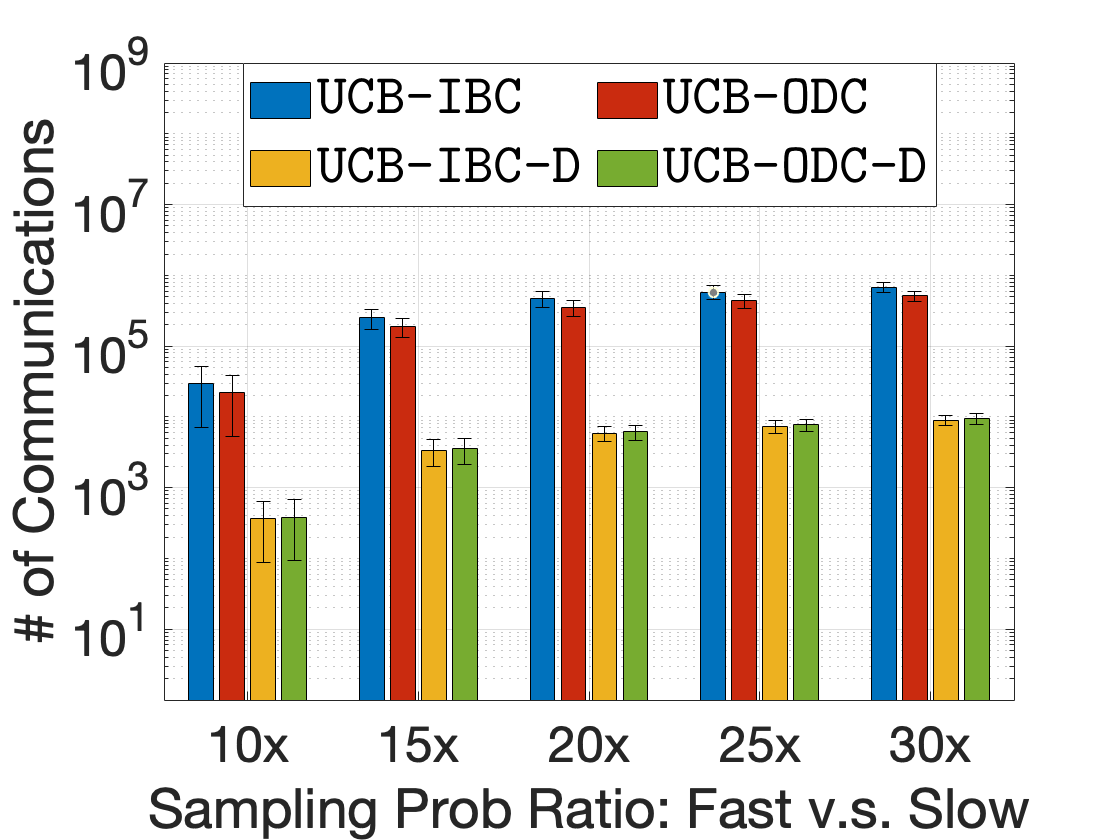}}
    \subfloat[UCB Comm Fast]{\includegraphics[width=0.245\textwidth]{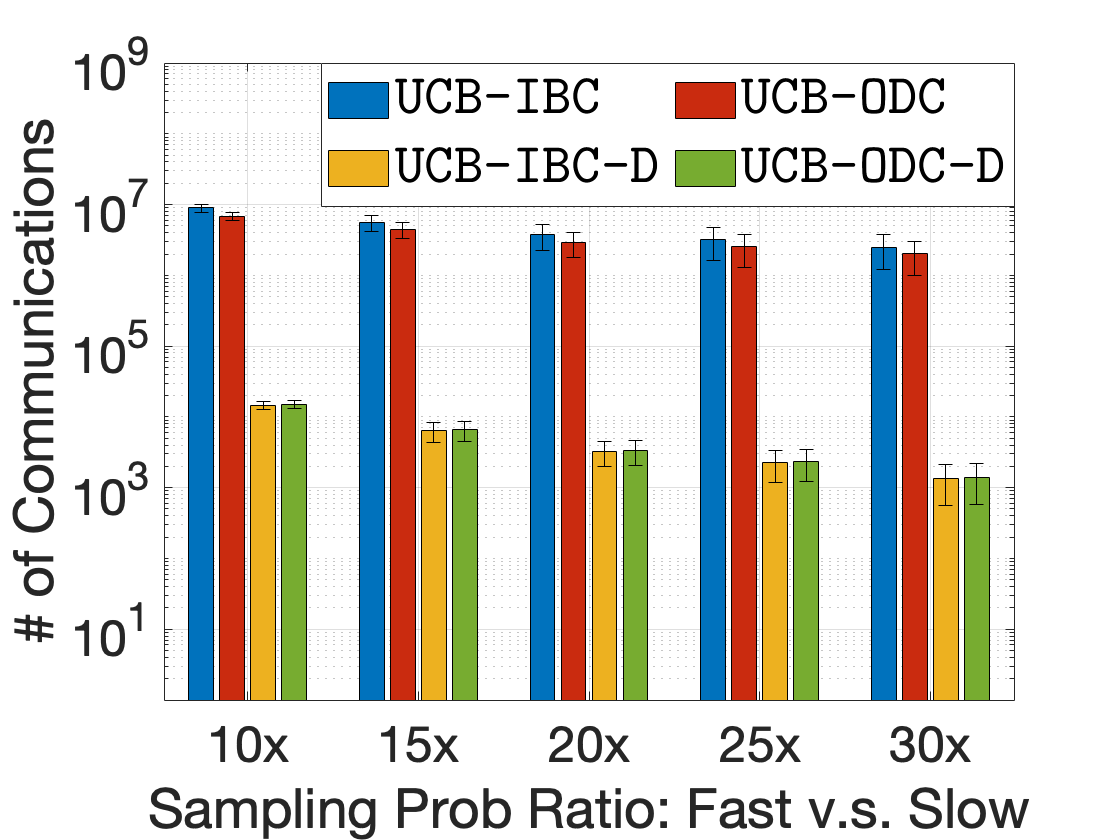}}\\
    \vskip -0.1 in
    \subfloat[AAE Group Regret]{\includegraphics[width=0.245\textwidth]{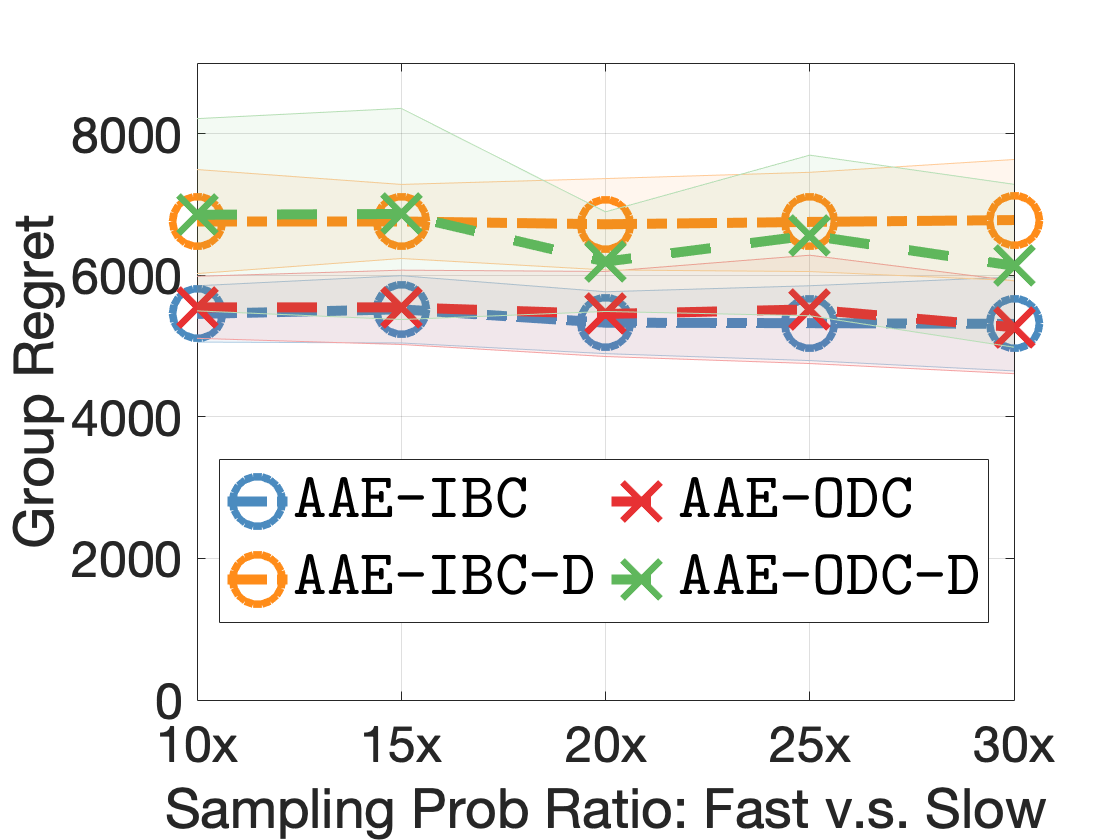}}
    \subfloat[AAE Comm Between]{\includegraphics[width=0.245\textwidth]{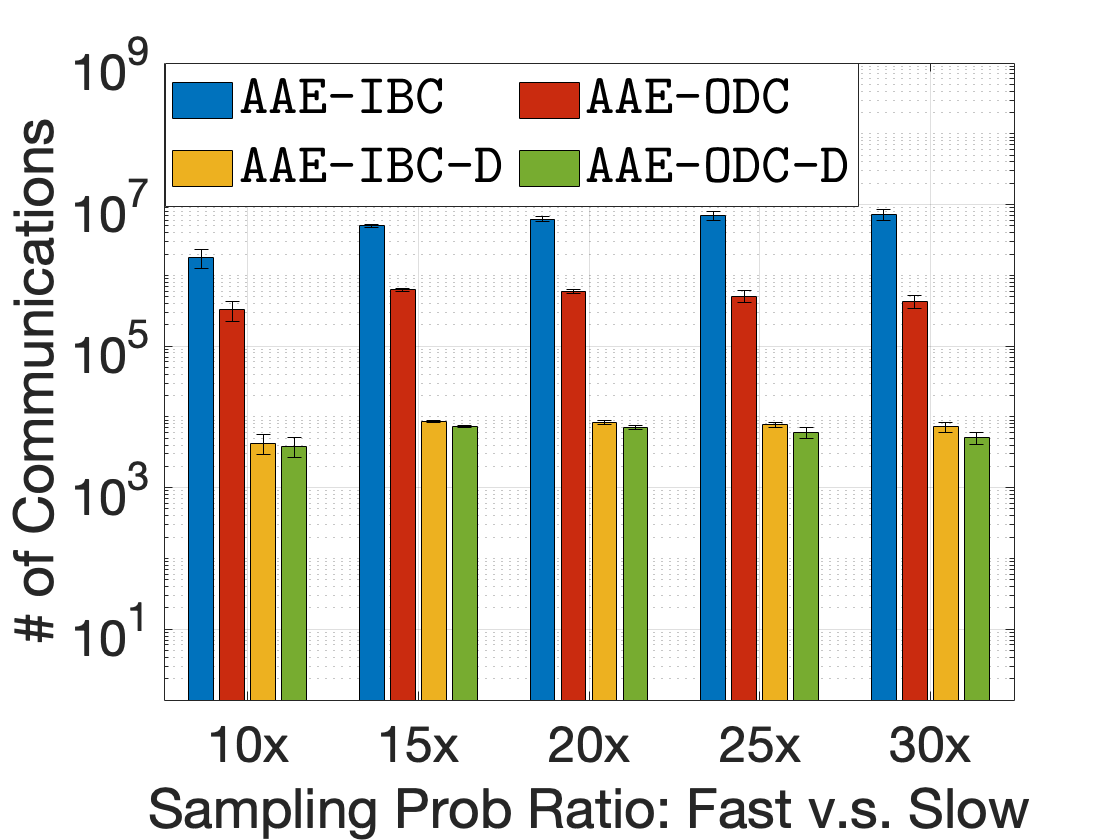}}
    \subfloat[AAE Comm Slow]{\includegraphics[width=0.245\textwidth]{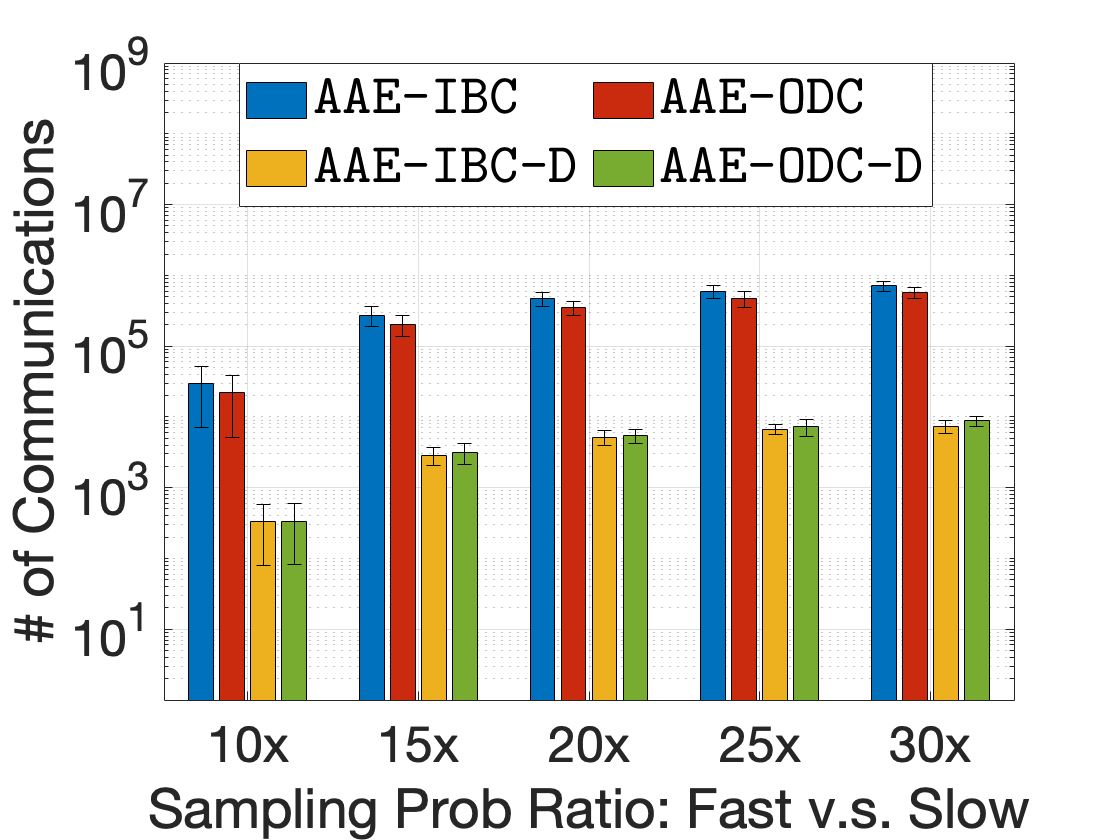}}
    \subfloat[AAE Comm Fast]{\includegraphics[width=0.245\textwidth]{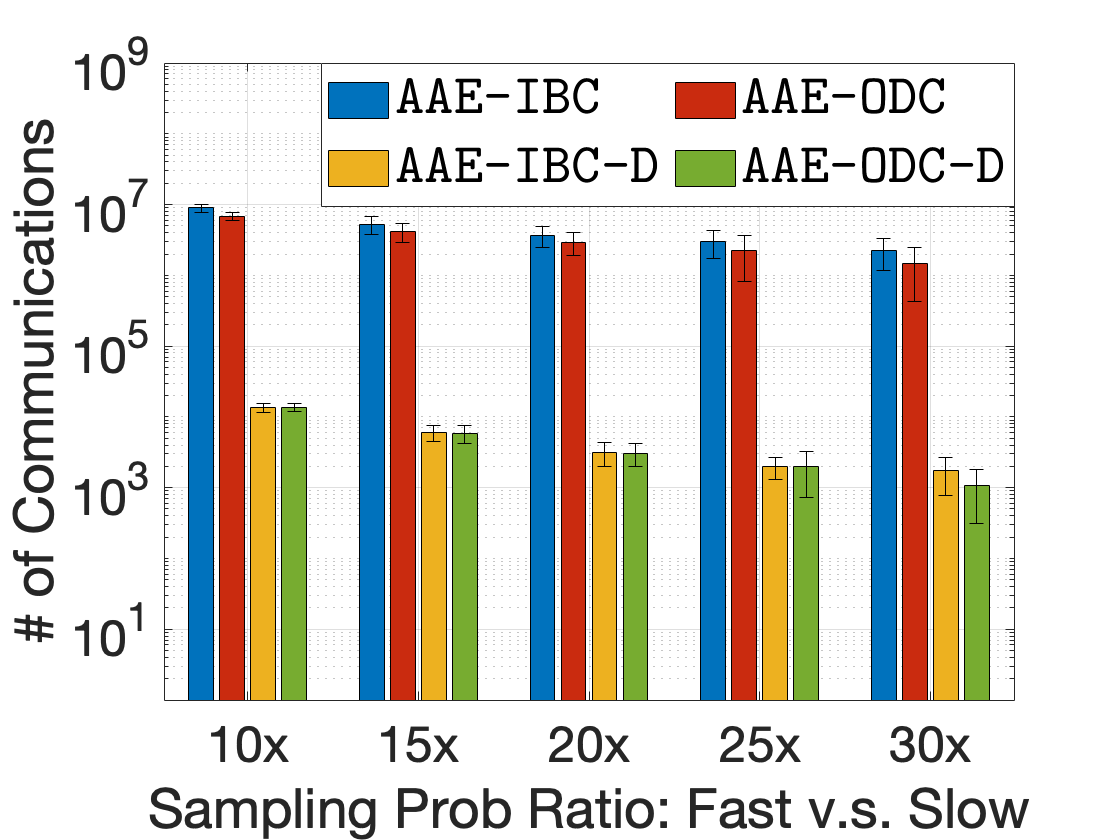}}
    \vskip -0.05 in
    \caption{Experiment 1  --- impact of the heterogeneity of agent speeds. Comparison between \texttt{IBC} and \ODC with \BTs set to one as well as \texttt{IBC} and \ODC with \BTs set to be doubling. For communication complexities, we present the numbers of communications between fast and slow agents, among slow agents, and among fast agents separately in different subfigures. Note that, in Subfigures (b)(c)(d) and (f)(g)(h), the Y axis is in Log scale.}
    \label{fig:exp1-supplementary}
    \vskip -0.1 in
\end{figure}

\begin{wrapfigure}{r}{0.26\textwidth}
\vskip -0.55 in
    \centering
    \subfloat[UCB Regret]{\includegraphics[width=0.245\textwidth]{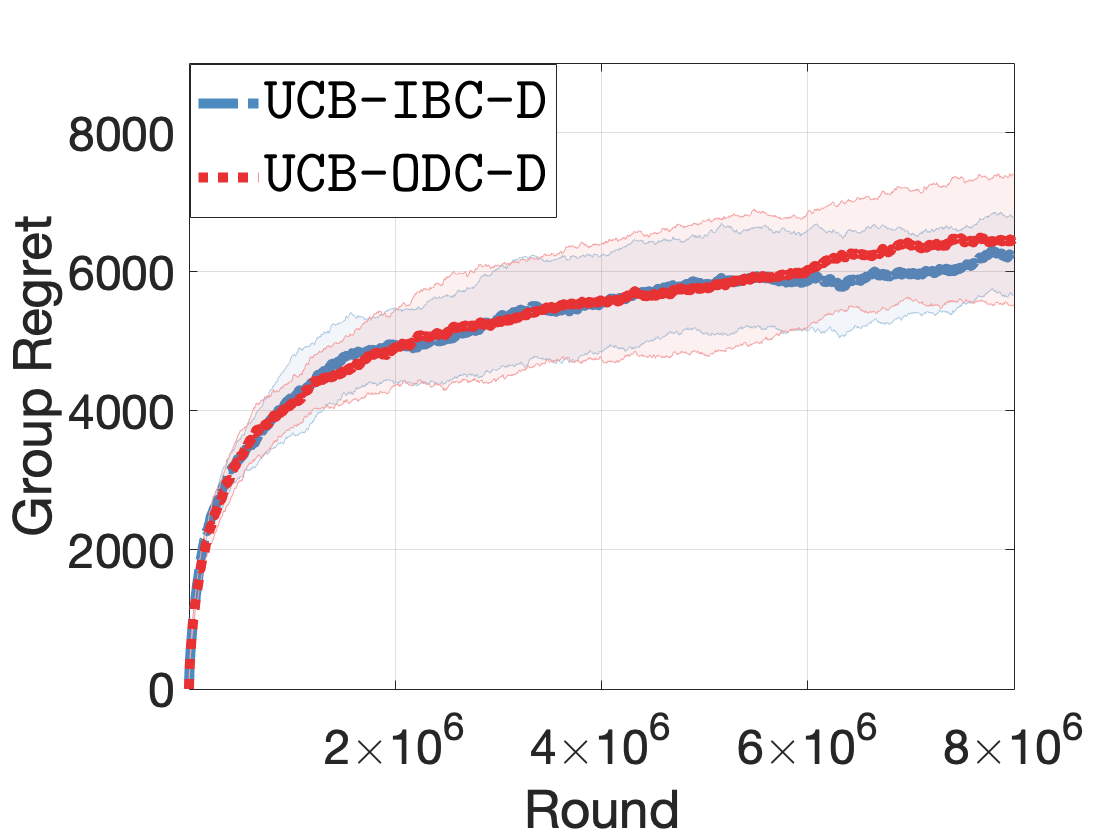}}\\
    \vskip -0.005 in
    \subfloat[UCB Comm]{\includegraphics[width=0.245\textwidth]{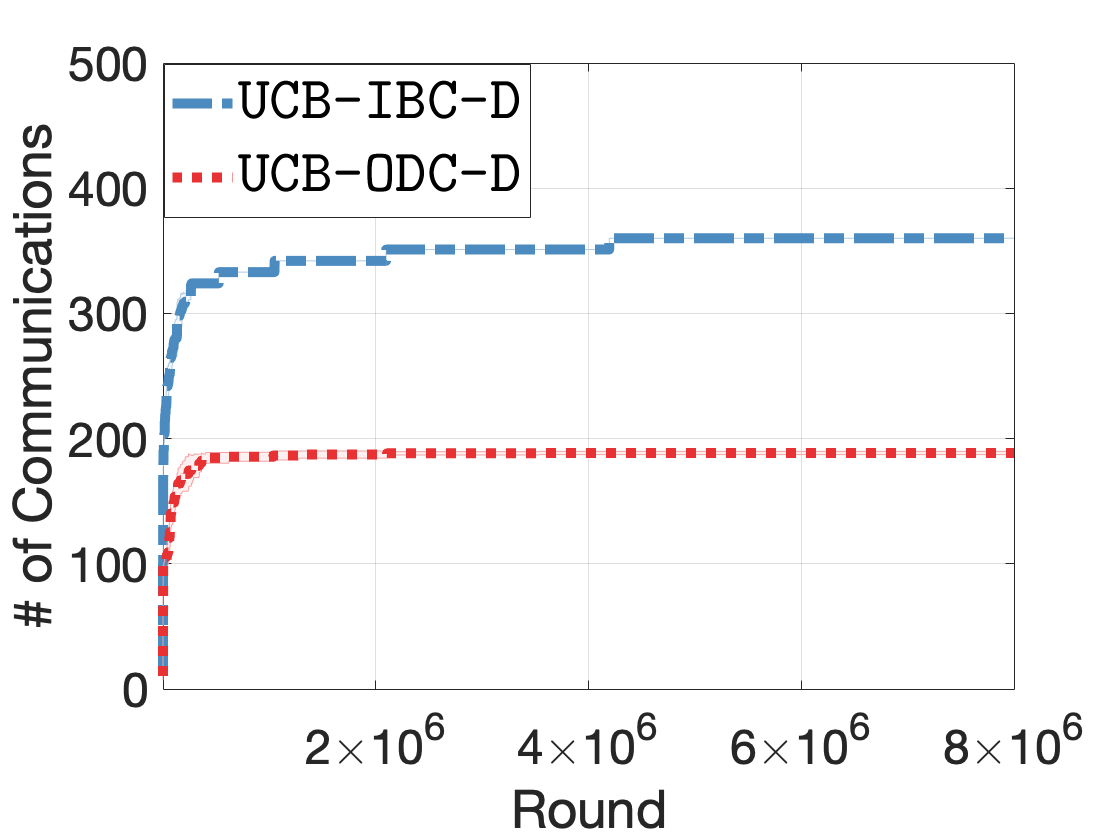}}
    \vskip -0.01 in
    %\subfloat[AAE Regret]{\includegraphics[width=0.245\textwidth]{figure/exp3-aae-reg.png}}
    %\subfloat[AAE Comm]{\includegraphics[width=0.245\textwidth]{figure/exp3-aae-comm.png}}
    \caption{Experiment 4 --- A system with agents that have exponentially large differences in their sampling probabilities}
    \label{fig:exp-large-differences}
    \vskip -0.4 in
\end{wrapfigure}

\subsection{Performance of \ODC with constant or doubling \BTs}\label{sec:extra-simu-thres}

In the Experiment 1 and Experiment 2 results presented in Section~\ref{sec:numerical}, we observe that, when agent pull speeds are highly diversified and when there exist many slow agents, the on-demand rule of \ODC saves communication overheads in contrast to \texttt{IBC} while achieving similar group regrets as \texttt{IBC} when both of them have constant \BTs. In Figure~\ref{fig:exp1-supplementary} and Figure~\ref{fig:exp2-supplementary}, we compare the performance of both \texttt{IBC} and \ODC with both constant (size one) \BTs (denoted as \texttt{AAE-IBC}, \AAEODC, \texttt{UCB-IBC}, \UCBODC) and doubling \BTs (denoted as \texttt{AAE-IBC-D}, \texttt{AAE-ODC-D}, \texttt{UCB-IBC-D}, \texttt{UCB-ODC-D}) under Experiment 1 and Experiment 2 setups respectively.

From Figures~\ref{fig:exp1-supplementary}(a),~\ref{fig:exp1-supplementary}(e) and Figures~\ref{fig:exp2-supplementary}(a),~\ref{fig:exp2-supplementary}(c), we observe that, with doubling \BTs, both policies under \texttt{IBC} and under \ODC have higher group regrets than those with constant \BTs. From the communication complexities results in Figure~\ref{fig:exp1-supplementary} and Figure~\ref{fig:exp2-supplementary}, we observe that, with doubling \BTs, both policies under \texttt{IBC} and under \ODC incur logarithmic communication overheads than those with constant \BTs. With doubling \BTs, policies under \ODC incur slightly smaller communication overheads than policies under \texttt{IBC} but the improvements are not as significant as when their \BTs are all set to be constant. This is because the ratio of sampling probabilities between fast and slow agents are at most $30$ times in both experimental setups; hence, when doubling \BTs is applied, the effect of the on-demand rule of \ODC is diminished. 

The advantage of the on-demand rule of \ODC is obvious, even with doubling \BTs applied, when the differences of pull rates are exponentially large, as shown in Experiment 4 (Figure~\ref{fig:exp-large-differences}). Figure~\ref{fig:exp-large-differences} shows the results of simulations of a system with $10$ agents, where there is a fast agent with sampling probability set to be always $1$ and nine slow agents with sampling probabilities initially set to be $0.1$ and halved after each message transmission. We report the cumulative group regret and number of communication over $T=8,000,000$ rounds. Figure~\ref{fig:exp-large-differences}(b) shows that \texttt{UCB-ODC-D} effectively saves communication overheads and Figure~\ref{fig:exp-large-differences}(a) shows that \texttt{UCB-ODC-D} still achieves similar group regrets as \texttt{UCB-IBC-D}.

%\clearpage
\begin{figure}
    \centering
    \subfloat[UCB Group Regret]{\includegraphics[width=0.245\textwidth]{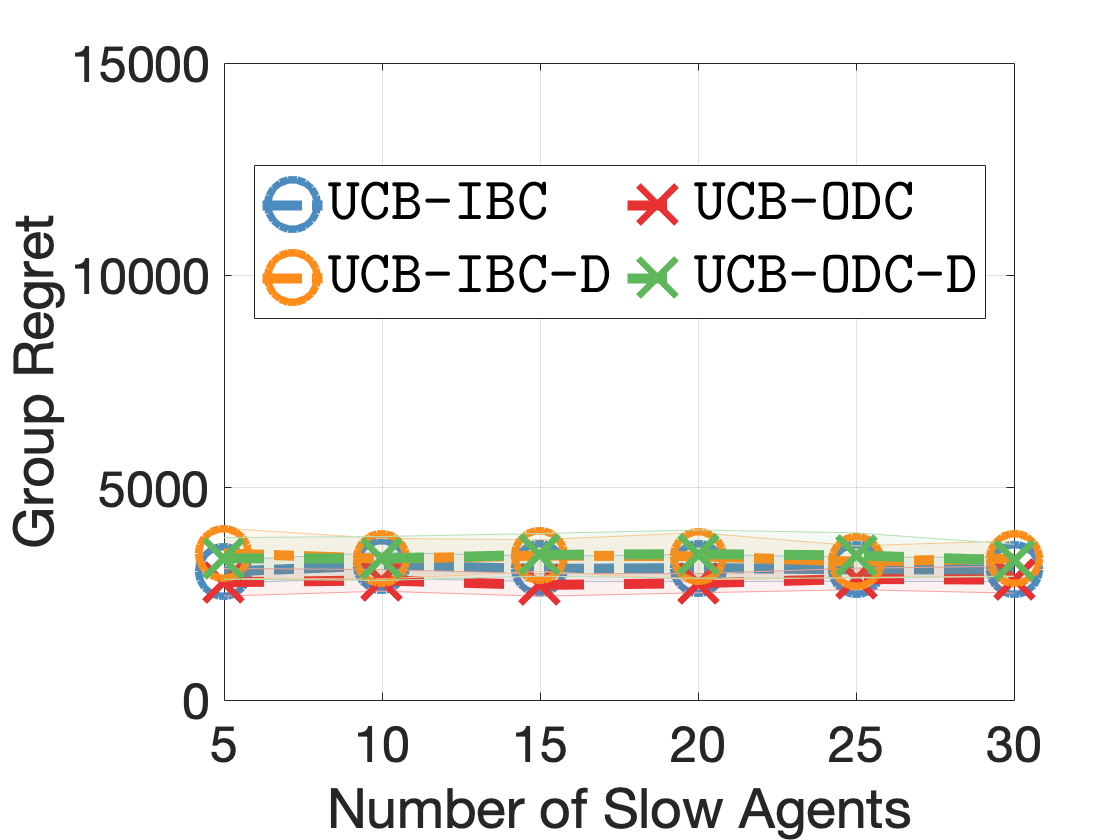}}
    \subfloat[UCB Comm]{\includegraphics[width=0.245\textwidth]{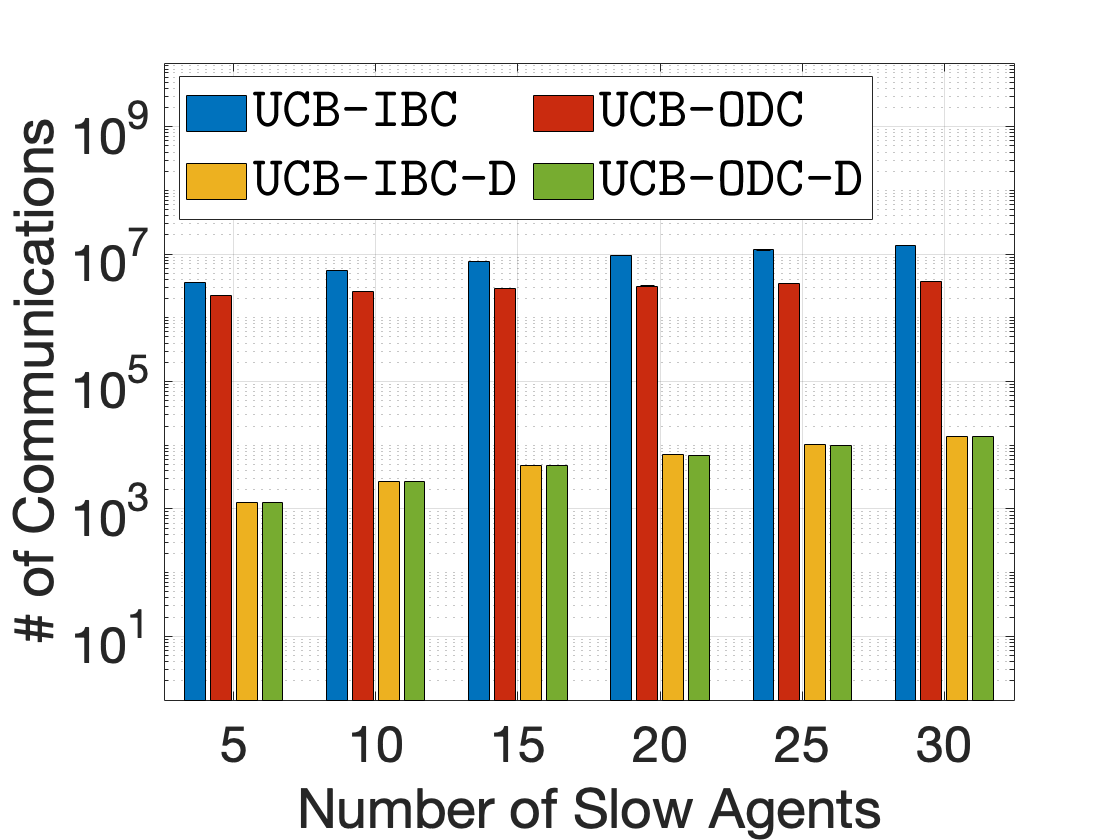}}
    \subfloat[AAE Group Regret]{\includegraphics[width=0.245\textwidth]{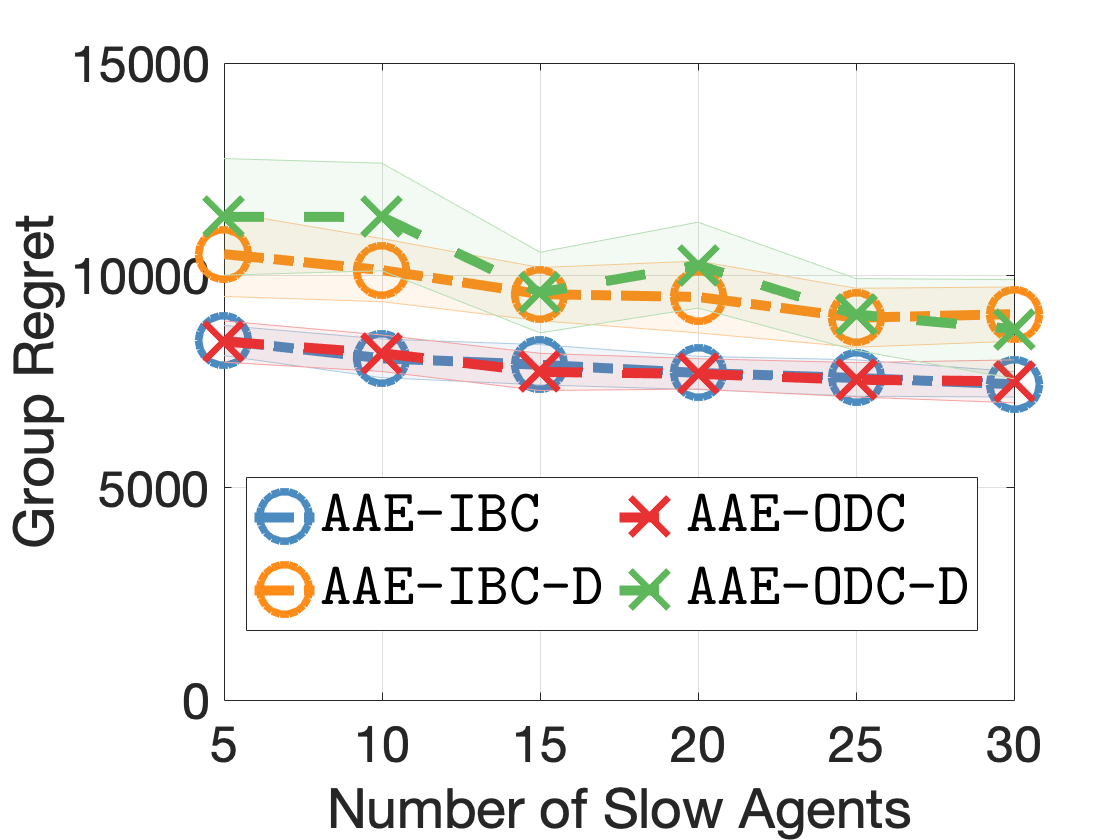}}
    \subfloat[AAE Comm]{\includegraphics[width=0.245\textwidth]{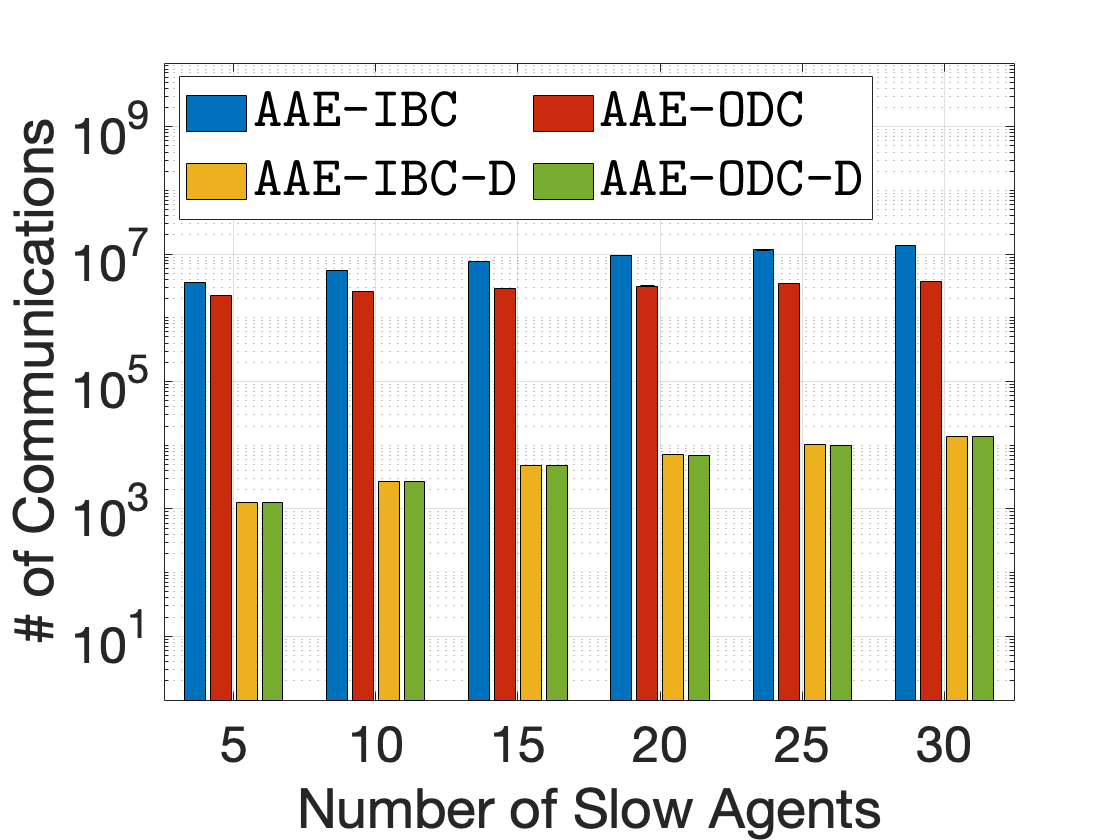}}
    \caption{Experiment 2  --- impact of the number of slow agents in the system. Comparison between \texttt{IBC} and \ODC with \BTs set to one as well as \texttt{IBC} and \ODC with \BTs set to be doubling. Note that, in Subfigures (b) and (d), the Y axis is in Log scale.}
    \label{fig:exp2-supplementary}
\end{figure}

\begin{figure}
    \begin{minipage}{0.49\textwidth}
        \centering
        \hfill
     	\subfloat[Individual Regret Mean]{\includegraphics[width=0.49\textwidth]{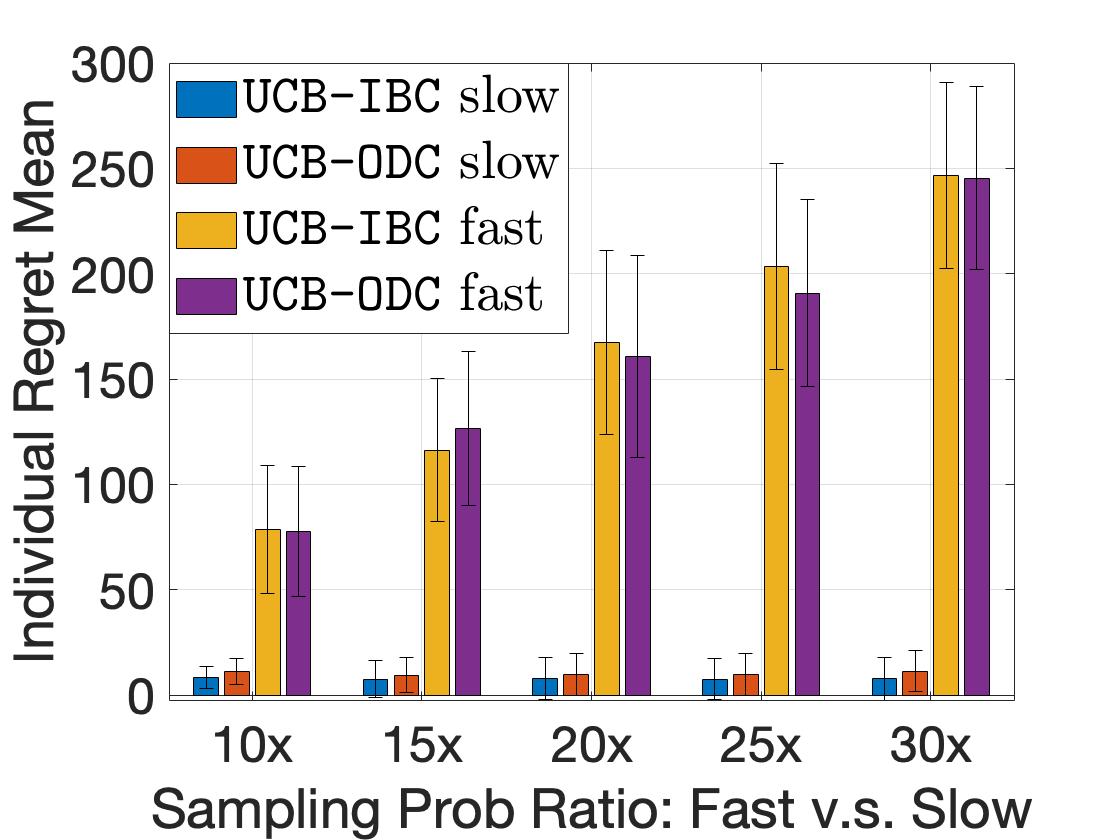}}
     	\hfill
    	\subfloat[Individual Regret Mean]{\includegraphics[width=0.49\textwidth]{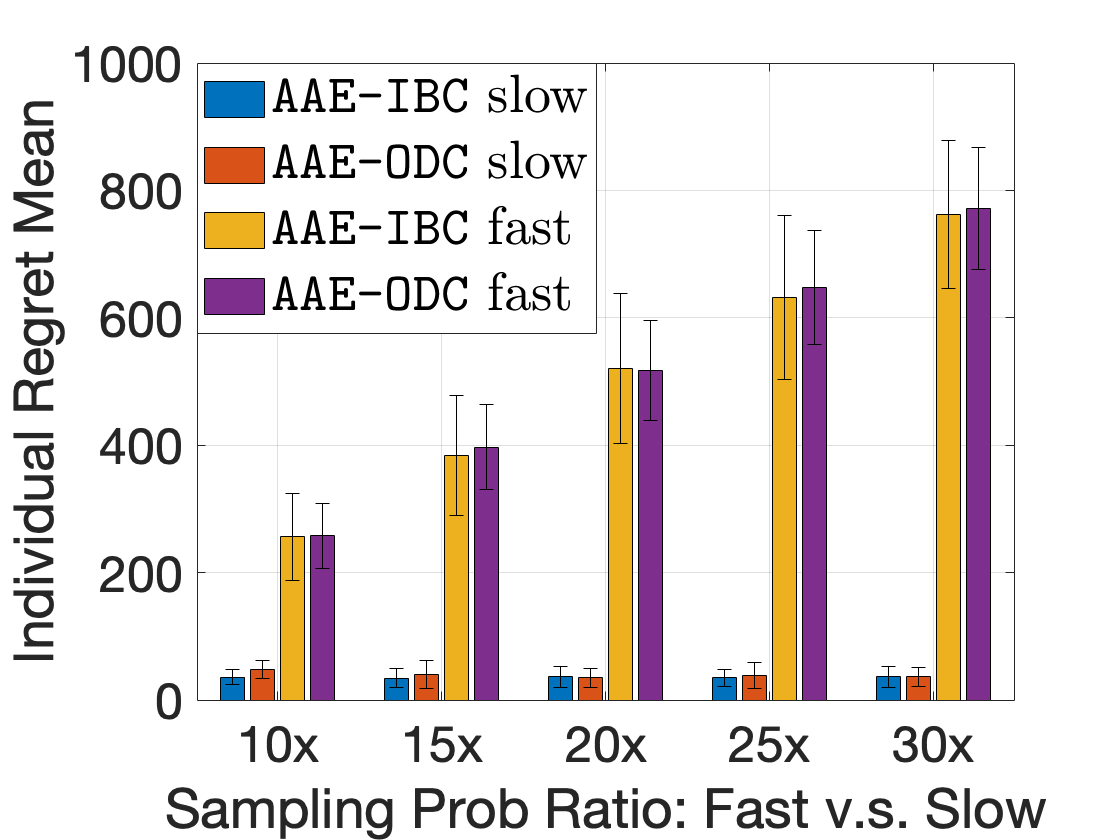}}
    	\hfill
    	\caption{Experiment 1 --- impact of the heterogeneity of agent speeds. %Forty agents with fixed mean sampling probability and increasing sampling probability ratio between fast and slow agents.
    	} 
    	\label{fig:experiment1-indi-regret}
    \end{minipage}
    \hfill
    \begin{minipage}{0.49\textwidth}
        \centering
        \hfill
     	\subfloat[Individual Regret]{\includegraphics[width=0.49\textwidth]{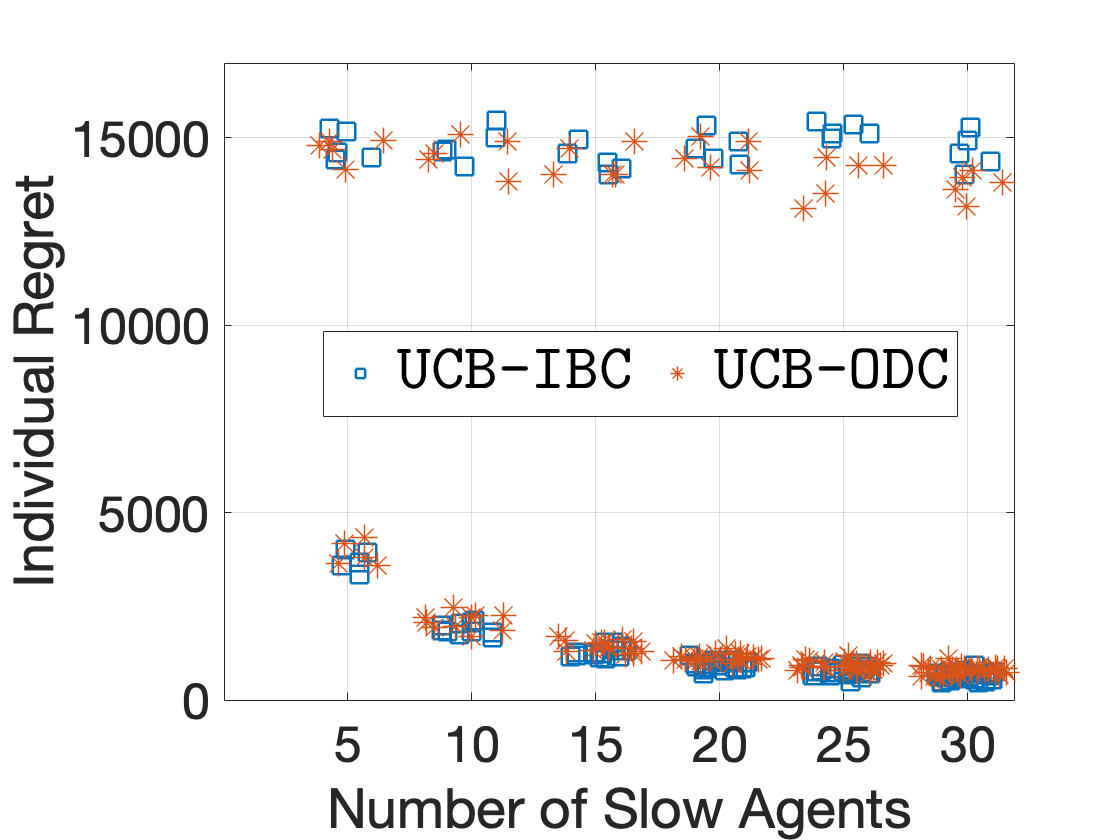}}
     	\hfill
    	\subfloat[Individual Regret]{\includegraphics[width=0.49\textwidth]{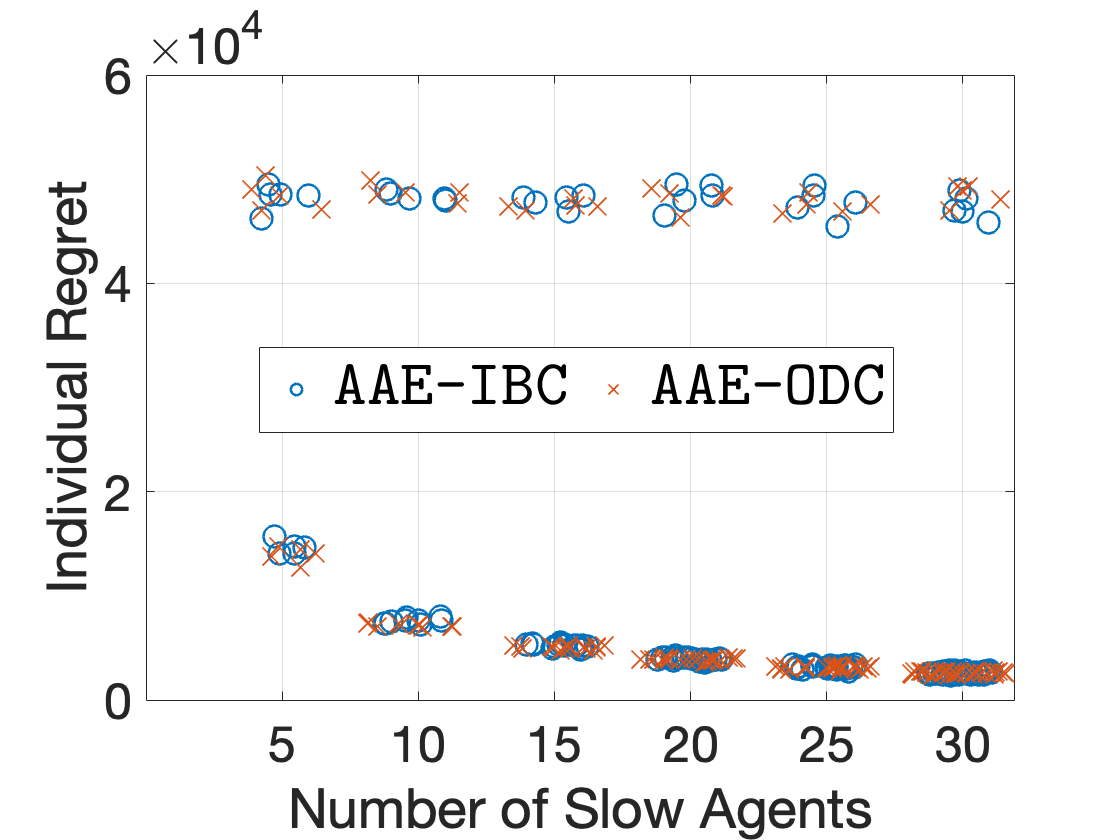}}
     	\hfill
    	\caption{Experiment 2 --- impact of the number of slow agents in the system. 
    	%Increasing the number of slow agents while fixing the expected total number of decisions in the entire system.
    	} 
    	\label{fig:experiment2-indi-regret}
    \end{minipage}
\end{figure}

\subsection{Individual Regrets in Experiment 1 and Experiment 2}\label{sec:extra-simu-ind}

In asynchronous MAMAB setting, individual agent's expected regret varies as the pulling times and the total number of decision rounds of the agent, $N_j$, vary. 
%A \textit{fast} agent who makes more number of decisions, $n_j$, in general has larger individual regret than a \textit{slow} agent with smaller $n_j$.

For Experiment 1 and Experiment 2 in Section~\ref{sec:numerical}, we add Figure~\ref{fig:experiment1-indi-regret} and Figure~\ref{fig:experiment2-indi-regret} respectively here to provide experimental observations about individual regrets.
Specifically, Figure~\ref{fig:experiment1-indi-regret} (for Experiment 1) contains two bar charts to present the mean and variance of individual regrets in between \textit{fast} agents and in between \textit{slow} agents after $T=80,000$ time slots. The height of a bar shows the individual regret mean of agents with same sampling probability and the error bar on each bar denotes mean plus/minus one standard deviation of the individual regrets of agents with same sampling probability. 
Figure~\ref{fig:experiment2-indi-regret} (for Experiment 2) contains two scatter charts on which each dot represents the individual regret of an agent.

Following are the experimental observations about individual regret.

In Experiment 1, we fix the sampling probability of each slow agent and vary the sampling probability of fast agents. In Figure~\ref{fig:experiment1-indi-regret}, the individual regret mean among slow agents stays almost the same when the difference in sampling probabilities of fast and slow agents increases; the variance of individual regrets among slow agents also stays almost the same. The individual regret mean of fast agents increases as the sampling probability of fast agent increases while the variance of individual regrets among fast agents stays almost the same.

In Experiment 2, we fix the number of fast agents as well as the sampling probability of fast agents and increase the number of slow agents. In Figure~\ref{fig:experiment2-indi-regret}, dots are clustered into two groups; the five dots with larger individual regrets are of the fast agents, and the other dots with smaller individual regrets are of the slow agents.

Figure~\ref{fig:experiment1-indi-regret} and Figure~\ref{fig:experiment2-indi-regret} show that \ODC achieve similar regret performance as \texttt{IBC} not only in terms of group regret but also in terms of individual regrets.

\subsection{Performance of \ODC under different types of asynchronicity}\label{sec:extra-simu-async}

In this subsection, we study the performance of \ODC under three more variants of asynchronicity, which are different from the stochastic asynchronicity considered in Experiments 1, 2, and 3.

\textbf{Experiment 5.} In this experiment, we study the impact of agents going offline and online, which models wireless sensing devices with sleeping/active modes for power saving. Specifically, there are five slow agents each with sampling probability $0.2$ and five fast agents each with sampling probability $0.8$. Fast agents, while having high pull rates when they are online, may go offline for a long time. Specifically, fast agents stay online or offline both according to a geometric distribution with parameter $0.01$ in this experiment. We report the number of communications and group regret after $T=80,000$ time slots averaged over $30$ independent trials in Table~\ref{tab:exp5}. 

\begin{table}[H]
    \vskip -0.06 in
    \centering
    {\small
    \caption{Experiment 5}
    \begin{center}
    \vskip -0.135 in
    \begin{tabular}{||c || c | c||} 
     \hline
       & Communication & Group Regret \\ [0.1ex]
     \hline
     \texttt{UCB-IBC} & $(2.1604 \pm 0.0247)\times10^6$ & $2442\pm267$ \\ 
     \hline
     \UCBODC & $(1.0119\pm 0.0135)\times 10^6$ & $2225\pm232$ \\
     \hline
     \texttt{AAE-IBC} & $(2.1605\pm0.0209)\times 10^6$ & $6788\pm412$ \\ 
     \hline
     \AAEODC & $(1.0105\pm0.0165)\times 10^6$ & $6957\pm446$ \\
     \hline
    \end{tabular}
    \end{center}
    \label{tab:exp5}
    }
    \vskip -0.2 in
\end{table}

\textbf{Experiment 6.} In this experiment, we study the impact of less learning horizons overlapping among agents. We have five slow agents each with sampling probability $0.1$ and five fast agents each with sampling probability $0.7$. In Experiment 6(a), we let the five slow agents go online from the very beginning and let the five fast agents go online at time slot $t = 40,000$. We do the other way around in Experiment 6(b) -- we let the five fast agents go online from the very beginning and let the five slow agents go online at time slot $t = 40,000$. We report the number of communications and group regrets after $T=80,000$ time slots averaged over $30$ independent trials in Tables~\ref{tab:exp6-a} and~\ref{tab:exp6-b} respectively.

\begin{table}[H]
    \vskip -0.06 in
    \centering
    {\small
    \begin{center}
    \begin{minipage}{.5\linewidth}
    \caption{Experiment 6(a)}
    \label{tab:exp6-a}
    \begin{tabular}{||c || c | c||} 
     \hline
       & Communication & Group Regret \\ [0.1ex]
     \hline
     \texttt{UCB-IBC} & $(1.6196 \pm 0.0025)\times 10^6$ & $1931 \pm 257$ \\ 
     \hline
     \UCBODC & $(0.8332 \pm 0.0021)\times 10^6$ & $2052 \pm 228$ \\
     \hline
     \texttt{AAE-IBC} & $(1.6199 \pm 0.0024)\times 10^6$ & $3906 \pm 813$ \\ 
     \hline
     \AAEODC & $(0.8325 \pm 0.0023)\times 10^6$ & $5021 \pm 641$ \\
     \hline
    \end{tabular}
    \end{minipage}%
    \begin{minipage}{.5\linewidth}
    \caption{Experiment 6(b)}
    \label{tab:exp6-b}
    \begin{tabular}{||c || c | c||} 
     \hline
       & Communication & Group Regret \\ [0.1ex]
     \hline
     \texttt{UCB-IBC} & $(2.7000 \pm 0.0025)\times 10^6$ & $2803 \pm 283$ \\ 
     \hline
     \UCBODC & $(1.2804 \pm 0.0022) \times 10^6$ &  $2568 \pm 285$\\
     \hline
     \texttt{AAE-IBC} & $(2.5073 \pm 0.4978)\times 10^6$ &  $6264 \pm 568$\\ 
     \hline
     \AAEODC & $(1.2797 \pm 0.0023)\times 10^6$ & $6741 \pm 447$ \\
     \hline
    \end{tabular}
    \end{minipage}%
    \end{center}
    }
    \vskip -0.2 in
\end{table}

\textbf{Experiment 7.} In this experiment, we study the impact of non-stationary asynchronicity. Specifically, we have ten agents and the sampling probability of agent $j$ follows a sine function, $\sin(\theta_j + t/30)$, where the phase shifts  $\theta_j = j/5, j\in \{1,...,10\}$ are different for different agents. We report the number of communications and group regrets after $T=80,000$ time slots averaged over $30$ independent trials in Table~\ref{tab:exp7}.

\begin{table}[H]
    \vskip -0.06 in
    \centering
    {\small
    \caption{Experiment 7}
    \begin{center}
    \vskip -0.135 in
    \begin{tabular}{||c || c | c||} 
     \hline
       & Communication & Group Regret \\ [0.1ex]
     \hline
     \texttt{UCB-IBC} & $(2.2936 \pm 0.0020)\times 10^6$ & $2411\pm 296$  \\ 
     \hline
     \UCBODC & $(1.5762 \pm 0.0021)\times 10^6$ & $2335\pm 224$ \\
     \hline
     \texttt{AAE-IBC} & $(2.2934 \pm 0.0023)\times 10^6$ & $7327 \pm 341$ \\ 
     \hline
     \AAEODC & $(1.5764 \pm 0.0026)\times 10^6$ & $7526 \pm 422$ \\
     \hline
    \end{tabular}
    \end{center}
    \label{tab:exp7}
    }
    \vskip -0.2 in
\end{table}

Results of Experiment 5, 6, and 7 in Table~\ref{tab:exp5},~\ref{tab:exp6-a},~\ref{tab:exp6-b}, and~\ref{tab:exp7} support our theoretical and experimental observations that \ODC incurs less communication than \texttt{IBC} while achieving similar group regret, and further show that \ODC is affective under various kinds of asynchronicity.

\subsection{When would \AAEODC incur fewer communications than \UCBODC?}\label{sec:extra-simu-algo}

Note that our theoretical analysis suggests that \AAEODC outperforms \UCBODC in terms of communication complexity. However, this has not been clearly shown in previous experiments because, in previous experiments, the time horizon $T$ is comparatively small for the arm reward suboptimality gap considered. In Figure~\ref{fig:easy}, we present the number of communications and group regret (averaged over 30 independent trials) incurred by \AAEODC and \UCBODC in the setting with $K=16$ arms, $5$ fast agents each with sampling probability $0.8$, $10$ slow agents each with sampling probability $0.1$, and $T = 80\,000$. Note that this setting is same as one of the cases in Experiment 2, except that here we experiment with an easier arm reward instance (with larger suboptimality gap). 

\begin{figure}[H]
	\centering
	\subfloat[Communication]{\includegraphics[width=0.3\textwidth]{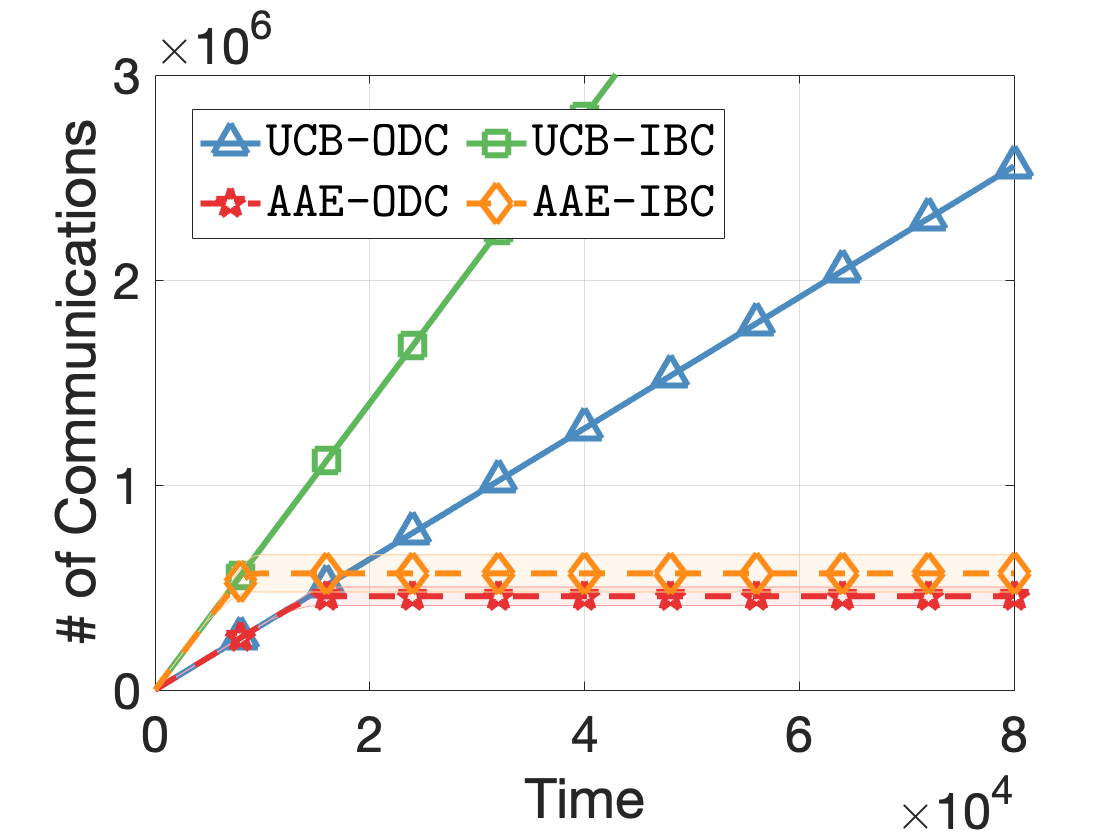}}
	\subfloat[Group Regret]{\includegraphics[width=0.3\textwidth]{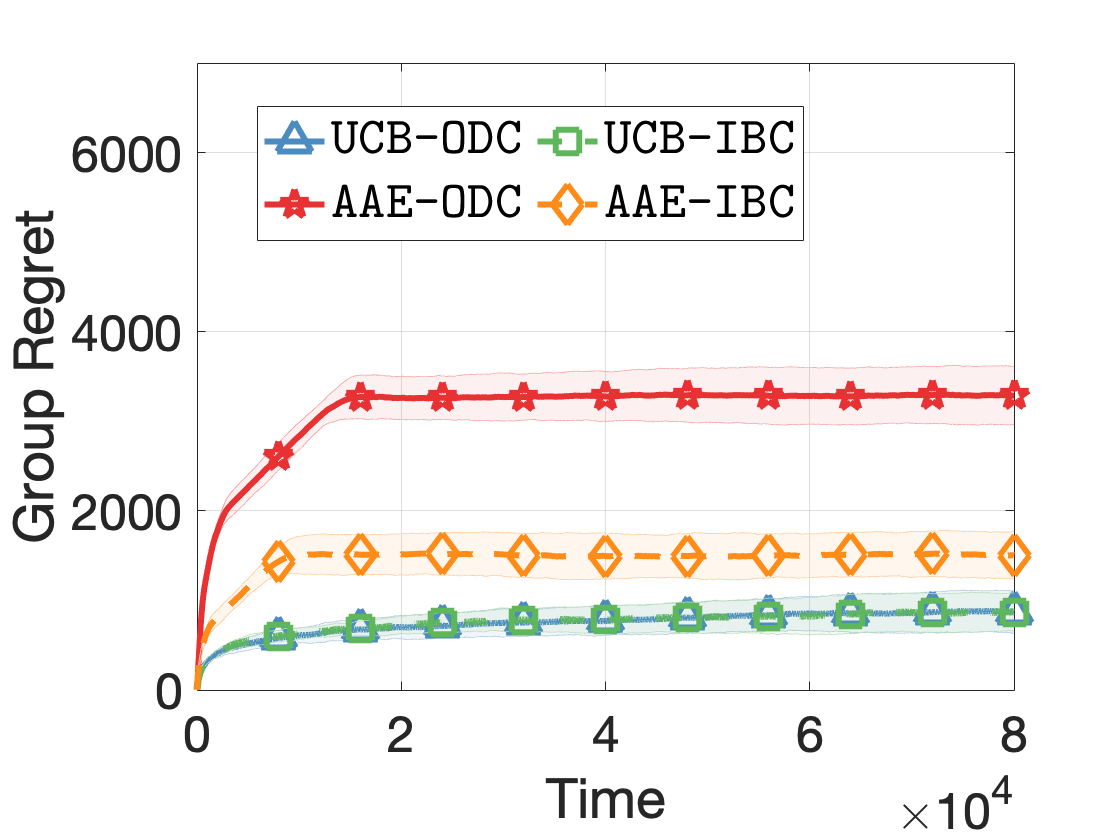}}
	\caption{\AAEODC outperforms \UCBODC in terms of communication complexity when the time horizon $T$ is comparatively large for the arm reward instance.
	} \label{fig:easy}
\end{figure}

\vfill

\end{document}